\documentclass[smallextended]{svjour3}

\paperwidth=\dimexpr
    1in + \oddsidemargin
    + \textwidth
    + 1in + \oddsidemargin
  \relax
  \paperheight=\dimexpr
    1in + \topmargin
    + \headheight + \headsep
    + \textheight
    + 1in + \topmargin
  \relax
  \usepackage[pass]{geometry}\relax

\usepackage{hyperref}
\usepackage{pdfsync}
\usepackage{dsfont}
\usepackage{color}
\usepackage{enumerate} 
\usepackage{amsfonts}
\usepackage{amssymb}%
\usepackage{mathtools}
\usepackage{braket}
\usepackage{bm}
\usepackage{amsmath}
\usepackage{graphicx}
\usepackage{caption}
\usepackage{pb-diagram}
\numberwithin{equation}{section}
\newtheorem{assumption}[theorem]{Assumption}

\newenvironment{lista}
  {\begin{list}
  {}
  {\setlength{\labelwidth}{60pt}
  \setlength{\leftmargin}{60pt}
  \addtolength{\itemindent}{0pt}
  \setlength{\topsep}{0pt}
  \setlength{\itemsep}{1pt}}}
 {\vspace{5pt} \end{list} }

\newcommand{\rc}{R} 
\newcommand{\1}{\mathds{1}}
\newcommand{\nt}{{N_c}}

\newcommand{\R}{\mathds{R}}
\newcommand{\N}{\mathds{N}}
\renewcommand{\P}{\mathds{P}}
\newcommand{\M}{\mathcal{M}}
\newcommand{\x}{\mathbf{x}}

\newcommand{\veps}{\varepsilon}
\newcommand{\te}{\textrm}
\DeclareMathOperator{\Dom}{Dom}
\DeclareMathOperator{\card}{card}
\newcommand{\vol}{Vol}
\DeclareMathOperator{\length}{Length}

\DeclareMathOperator{\dist}{dist}

\DeclareMathOperator{\diam}{diam}

\DeclareMathOperator{\divergence}{div}

\DeclareMathOperator{\esssup}{esssup}
\DeclareMathOperator{\Lip}{Lip}

\newcommand{\red}{\color{black}}  

\definecolor{mygreen}{rgb}{0.1,0.75,0.2}

\newcommand{\nc}{\normalcolor}

\newcommand{\mh}[1]{{\color{blue}#1}}

\newcommand{\applied}[2]{\langle #1,#2\rangle}
\DeclarePairedDelimiter\norm{\lVert}{\rVert}
\DeclarePairedDelimiter\abs{\lvert}{\rvert}

\begin{document}

\title{Error estimates for spectral convergence of the graph Laplacian on random geometric graphs towards the Laplace--Beltrami operator}
\titlerunning{Convergence of the graph Laplacian  towards the Laplace--Beltrami operator}
%
%
%
%
%

\author{Nicol\'as Garc\'ia Trillos         \and Moritz Gerlach \and Matthias Hein  \and  Dejan Slep\v{c}ev}
\institute{N.\ Garc\'ia Trillos \at
               Division of Applied Mathematics \\
              Brown University\\
              \email{nicolas\_garcia\_trillos@brown.edu}           
           \and
           M.\ Gerlach \at
           Department of Mathematics and Computer Science\\
          Saarland University\\
          \email{gerlach@cs.uni-saarland.de}
         \and
         M.\ Hein \at
         Department of Mathematics and Computer Science\\  
         Saarland University\\
         \email{hein@cs.uni-saarland.de}
          \and
         D.\ Slep\v{c}ev\at
         Department of Mathematical Sciences\\
         Carnegie Mellon University\\
         \email{slepcev@math.cmu.edu}
}

\date{\today}

\maketitle

\begin{abstract}
        We study the convergence of the graph Laplacian of a random geometric graph generated by an i.i.d. sample from a $m$-dimensional submanifold
        $\M$ in $\R^d$ as the sample size $n$ increases and the neighborhood size $h$ tends to zero. We show that eigenvalues and eigenvectors
       of the graph Laplacian converge with a rate of $O\Big(\big(\frac{\log n}{n}\big)^\frac{1}{2m}\Big)$ to the eigenvalues and eigenfunctions of the weighted Laplace-Beltrami operator of $\M$. 
       No information on the submanifold $\M$ is needed in the construction of the graph
       or the ``out-of-sample extension'' of the eigenvectors. Of independent interest is a generalization of the rate of convergence of empirical measures on 
       submanifolds in $\R^d$ in infinity transportation distance.
\end{abstract}


%
%

\subsection{Notation}
\begin{lista}
\item[$\M$] compact manifold without boundary embedded in $\R^d$. Riemann metric on $\M$ is the one inherited from $\R^d$. 
\item[$m$] the dimension of $\M$.
\item[$\vol(A)$] the volume of $A \subset \M$ according to Riemann volume form. 
\item[$d(x,y)$] the geodesic distance between points $x,y \in \M$.
\item[$  B_\M(x,r)$] ball in $\M$ with respect to geodesic distance on $\M$.
\item[$  B(r)$] ball in $\R^d$ of radius $r$, centered at the origin.
\item[$\mu$] probability measure supported on $\M$ that describes the data distribution.
\item[$p$] density of $\mu$ with respect to volume form on $\M$.
\item[$\rho$] density of the weight measure (which allows us to consider the normalized graph Laplacian) with respect to $\mu$.
\item[$\alpha$] constant describing the  bounds on the densities $p$ and $\rho$, see \eqref{def:alphap} and \eqref{eqn:DensityBounds}.
\item[$X$] point cloud $X = \{x_1, \dots, x_n\} \subset \M$ drawn from distribution $\mu$. Also considered as the set of vertices of the associated graph.
\item[$\mu_n$] empirical measure of the sample $X$.
\item[$\vec{m}$] The vector giving the values of the discrete weights used in various forms of graph Laplacian, see Sections \ref{sec:unGL} and \ref{sec:rwGL}. 
\item[$w_{i,j}$] edge weight between vertices $x_i$ and $x_j$.
\item[$\delta u$]  differential of function $u : X \to \R$. It maps edges to $\R$ and is defined by $\delta u_{i,j} = u(x_j) - u(x_i)$. 
\item[$i_0$] injectivity radius of $\M$. \red The injectivity radius at a point $p \in \M$ is the largest radius of a ball for which the exponential map at $p$ is a diffeomorphism. The injectivity radius $i_0$ is the infimum of the injectivity radii at all points of $\M$.\nc
\item[$K$] maximum of the  absolute value of sectional curvature of $\M$
\item[$\rc$] reach of $\M$, defined in \eqref{defreach}. 
\item[$\eta$] nonnegative function setting the edge weights as a function of the distance between the vertices, see \eqref{eqn:weights}.
\item[$h$] length scale such that weight between vertices is large if their distance is comparable to or less than $h$.
\item[$\sigma_\eta$] is the kernel dependent scaling factor relating  the graph Laplacian and the continuum Laplacian; defined in \eqref{def:sigma}.
\item[$\omega_m$] the volume of unit ball in $\R^m$.
\item[$d_\infty(\mu, \nu)$] infinity transportation distance between measures $\mu$, $\nu$. 
\item[$\veps$] upper bounds on the transportation distance between $\mu$ and $\mu_n$.
\item[$L$] Lipschitz constant of various functions: $p$, $\rho$ and $\eta$. 
\item[$P$] discretization operator defined in \eqref{eqn:InterpolatingOp}.
\item[$P^*$] is the adjoint of $P$ if $\rho \equiv 1$ and an approximate adjoint otherwise.
\item[$I$] Interpolation operator defined in \eqref{eqn:InterpolatingOp}.
\end{lista}


\newpage

\section{Introduction}

Given an i.i.d. sample $X=\left\{ x_1,\dots, x_n \right\}$ from the data generating measure $\mu$ in Euclidean space $\R^d$, the goal of most tasks in machine learning and statistics is to infer properties of $\mu$. A particularly interesting case is if $\mu$ has support on a $m$-dimensional compact submanifold $\M$ in $\R^d$ e.g. due to strong dependencies between the individual features. In this case one can construct a neighborhood
graph on the sample by connecting all vertices of Euclidean distance less than a certain  length-scale  $h$, and in this way produce a discrete approximation of the unknown manifold $\M$. Laplacian Eigenmaps \cite{Belkin02laplacianeigenmaps} and Diffusion
Maps \cite{Coifman1} have been proposed as tools to extract intrinsic structure of the manifold by considering the eigenvectors of the resulting unnormalized resp. normalized \emph{graph Laplacian};  in particular, Laplacian eigenmaps are used in the first step of spectral clustering \cite{vonLux_tutorial}, one of the most popular graph-based clustering methods . In general, it is well known that the spectrum of the graph
Laplacian resp. Laplace-Beltrami operator captures important structural resp. geometric properties of the graph \cite{Moh1991} resp. manifold \cite{Cha1984}. 

 In this paper we examine this question: under what conditions, and at what rate, does the spectrum of the graph Laplacian built from i.i.d. samples on a submanifold  converge to the spectrum of the  (weighted) Laplace--Beltrami operator of the submanifold as the sample size $n \rightarrow \infty$ and the neighborhood radius $h\rightarrow 0$?  

 Graph-based approximations to the Laplace-Beltrami operator have been studied by several authors and in a variety of settings.  The pointwise convergence of the graph Laplacian towards the Laplace-Beltrami operator has been proven in \cite{HeAuvL07,bel_niy_LB,GK,Hei2006,singer06,THJ}.  The spectral convergence of the graph Laplacian for fixed
neighborhood size $h$ for Euclidean domains has been established in \cite{vLBeBo08,RosBelVit2010}. 
The spectral convergence of the graph Laplacian
towards the Laplace--Beltrami operator for the uniform distribution has been discussed in \cite{belkin2007convergence} for the case of Gaussian weights and in \cite{SinWu13} for the connection Laplacian, without precise information on allowable scaling of neighborhood radius, $h$ and without convergence rates. In \cite{GTSspectral} the authors establish the conditions on graph connectivity for the spectral convergence on domains in $\R^m$. In particular they prove convergence when  $h \to 0$ as $n \to \infty$ and 
\[ h \gg \frac{(\log n)^{p_m}}{n^{\frac{1}{m}}}, \]
where
\begin{equation} \label{defpm}
p_m = \begin{cases}
		\frac{3}{4} \quad & \te{if } m = 2 \medskip \\
		\frac{1}{m} & \te{if } m\geq 3.
	 \end{cases}
\end{equation}
However no error estimates were established. 
\red 
The preprint \cite{Shi2015} establishes (in Theorem 1.1)  the spectral convergence of graph Laplacians constructed from data sampled from a submanifold in $\R^d$ with a convergence rate of $O\Big( \big(\frac{\log n}{n}\big)^{\frac{1}{4m+14}}\Big)$, where $m$ is the intrinsic dimension of the submanifold. 
\nc

In this paper we propose a general framework to analyze the rates of spectral convergence for a large family of graph Laplacians. This framework in particular allows us to improve the results in \cite{Shi2015} 
 and establish a convergence rate of $O\Big(\big(\frac{\log n}{n}\big)^\frac{1}{2m}\Big)$ which is a significant improvement, in particular for small dimensions $m$. These convergence rates hold for different reweighing schemes of the graph Laplacian found in the literature including the unnormalized  Laplacian, normalized Laplacian, and the random walk Laplacian. When the intrinsic dimension of the submanifold $\M$ is small, our results show, to some extent, why Laplacian eigenmaps can effectively extract geometric information from the data set, even though the number of features $d$ may be high. 
Moreover, similar to \cite{GTSspectral}, we show that
the conditions in \eqref{defpm} are sufficient for spectral convergence. This is essentially the same condition required to ensure that the constructed graph is almost surely connected \cite{PenroseBook} and thus is close to optimal. It is interesting to note that for pointwise consistency  of the graph Laplacian \cite{HeAuvL07,GK} the required stronger condition is $\frac{nh^{m+2}}{\log n} \rightarrow \infty$. 

Our framework is completely different from that in \cite{belkin2007convergence,Shi2015} and builds on two main ideas. First, it builds on an extension of the recent result of Burago, Ivanov und Kurylev \cite{BIK}, see also \cite{Fuj1995}, which shows in a non-probabilistic setting how one can approximate 
eigenvalues and eigenfunctions of the Laplace-Beltrami operator using the eigenvalues/eigenvectors of the graph Laplacian associated to an $\epsilon$-net of the submanifold.  As in our setting the manifold $\M$ is unknown, we generalize the result of \cite{BIK} by using a graph construction which requires no knowledge about the submanifold $\M$ but which achieves the same approximation guarantees for the eigenvalues. In addition, we introduce a new out-of-sample extension of the eigenvectors for the approximation of the eigenfunctions which requires no information about the submanifold without significant loss in the convergence rate compared to the
corresponding construction used in \cite{BIK}.  Our second main result generalizes the recent work of
Garc\'{i}a Trillos and Slep{\v{c}}ev \cite{GTS15a}  to the setting of empirical measures on submanifolds $\M \subset \R^d$ and establishes their rate of convergence in $\infty$-optimal transportation (OT) distance; the $\infty$-OT distance between the empirical measure associated to a point cloud and the volume form of the submanifold can be seen to be closely related to the notion of $\epsilon$-net used in \cite{BIK}. These estimates encompass all the probabilistic computations that we need to obtain our main results, and in particular, when combined with our deterministic computations, provide all the probabilistic estimates that quantify the rate of convergence of the spectrum of graph Laplacians constructed from randomly generated data towards the spectrum of a (weighted) Laplace-Beltrami operator on $\M$. We believe that both the generalization of \cite{BIK}, as well as the generalization of \cite{GTS15a} are of independent interest. The combination of these two ideas and a number of careful estimates
lead to our main results.

In what follows we make the setting that we consider in the sequel precise,  as well as  define precisely the different graph Laplacians and their continuous counterparts. 
\nc

\subsection{\red Graph construction \nc}
\label{subsec:Setting}
Let $\M$ be a compact connected $m$-dimensional Riemannian manifold without boundary, embedded in $\R^d$, with $m\geq 2$. We assume that the absolute value of sectional curvature is bounded by $K$, the injectivity radius is  $i_0$ and
with reach $\rc$. 
We write $d(x,y)$ for the distance between $x$ and $y$ on the manifold and $\abs{x-y}$ for the
Euclidean distance in $\R^d$.

Let $\mu$ be a probability measure on $\M$ that has a non-vanishing Lipschitz continuous density $p$ with respect to the Riemannian volume on $\M$ {with
Lipschitz constant $L_p$}.
Compactness of $\M$ and continuity of $p$ guarantee the existence of a constant $\alpha \geq 1$ such that
\begin{equation} \label{def:alphap}
  \frac{1}{\alpha} \leq p(x) \leq \alpha \quad \text{for all } x \in \M.
\end{equation}

We let $x_1, x_2, \dots, x_n, \dots$ be a sequence of i.i.d.\ samples from $\mu$.  In order to leverage the geometry of $\M$ from the data, we build a graph with vertex set $X \coloneqq \{x_1,\dots,x_n\}$.
In the simplest setting, for each $n\in\N$ we choose a neighborhood parameter $h=h_n$ and we put an edge from $x_i$ to $x_j$ and from $x_j$ to $x_i$ (and write $x_i \sim x_j$)  provided that $\abs{x_i-x_j} \leq h$; 
we let $E=\{(i,j) \in \{1,\dots,n\}^2 : x_i\sim x_j  \}$ be the set of such edges. More generally, we consider weighted graphs, with weights that depend  on the distance between the vertices connected by them. For that purpose, 
let us consider a decreasing function $\eta\colon [0,\infty) \rightarrow [0,\infty)$ with support on the interval $[0,1]$ such that the restriction of $\eta$ to $[0,1]$ is Lipschitz continuous. 
Normalizing $\eta$ if needed allows us to assume from here on that 
\begin{equation}\label{eq:normalizedkernel}
 \int_{\R^m} \eta(\abs{x}) dx =1.
\end{equation}
For convenience we assume that $\eta(1/2)>0$.
 We denote by
 \begin{equation} \label{def:sigma}
\sigma_\eta \coloneqq \int_{\R^m} \abs{y_1}^2 \eta(\abs{y}) dy,
\end{equation}
the \emph{surface tension} of $\eta$, where $y_1$ represents the first coordinate of the vector $y\in \R^m$. 
To every  given edge $(i,j) \in E$ we assign the weight $w_{i,j}$ where
\begin{align}
\label{eqn:weights}
  {w_{i,j} = \frac{1}{nh^m}\eta\left(\frac{\abs{x_i-x_j}}{h}\right)} 
\end{align}
and we consider the weighted graph $(X,w)$ with $w_{i,j} $ as in \eqref{eqn:weights} for every $(i,j)$. 
In fact, note that if the points $x_i$, $x_j$ are not connected by an edge in $E$ then $w_{i,j}=0$.

\begin{remark}
\label{rem:sigmaetanum}
The function $\eta$ can be chosen as $c \mathds{1}_{[0,1]}$
as well as a smooth  function like
\[ \eta(t) \coloneqq c \begin{cases} \exp\left(\frac{1}{t-1}\right) & 0 \leq t < 1 \\ 0 & t\geq 1,\end{cases}\]
(where $c$ is the appropriate constant ensuring normalization)
or simply a truncated version of a Gaussian. 
Also, we note that for $\eta = \frac{1}{\omega_m} \mathds{1}_{[0,1]}$ it follows from \cite[(2.7)]{BIK} that 
$\sigma_\eta = \frac{1}{m+2}$, where $\omega_m$ is the volume of the unit ball in $\R^m$.
While the definition of the weights is up to the constant $\sigma_\eta$ and a slightly different rescaling in terms of $h$
is similar to \cite{BIK}, the main difference is that we use the Euclidean metric of the ambient space $\R^d$ 
in \eqref{eqn:weights}, whereas in \cite{BIK} neighborhoods are throughout defined in terms of the geodesic distance. Here we are forced to use the metric from the ambient space as the manifold $\M$ is in general assumed to be unknown.  
\end{remark}

\begin{remark}
We have assumed that $\eta\colon [0,1] \rightarrow \R$ is decreasing and that $\eta(1/2)>0$, which would imply that $\eta(0)>0$. Nevertheless, we remark that none of the results 
presented in this paper change if we modify the value of $\eta(0)$. In particular we allow for $\eta(0)=0$ if desired and we can simply assume that $\eta$ is decreasing and Lipschitz in $(0,1)$ (then the condition $\eta(0)>0$ changes to $\eta(0+) >0$). 
This observation is relevant in order to allow for graphs where vertices have no edges with themselves. 
\end{remark}

\red
\begin{remark}
The requirement that $\eta$ is compactly supported is purely a technical one. It is in principle possible to 
carry out the arguments of this work for noncompact kernels, like the Gaussian one. However that would require obtaining error bounds on extra terms and would make the already involved estimates even more complicated. 
\end{remark}
\nc

\subsection{Dirichlet forms and laplacians}
\label{sec:DirichletForms}
In this section we introduce the Laplacians in both discrete and continuous settings.\nc

We use the graph structure defined in the previous section to define a Dirichlet form in the discrete setting. First, the weights $w_{i,j}$ serve as a measure on the set $E$ and thus induce a scalar product of functions $F,G \colon E \to \R$ given  by
\[ 
\applied{F}{G} \coloneqq \applied{F}{G}_{L^2(E,w)} 
\coloneqq {\frac{1}{n \sigma_{\eta}}}\sum_{(i,j) \in E} w_{i,j} F(i,j)G(i,j).
\]
Second, for functions $u,v \colon X\to\R$ on the vertices, we define the \emph{discrete differential}
\begin{equation}
 (\delta u)(i,j) \coloneqq {\frac{1}{h}}\big(u(x_j) - u(x_i)\big) \quad \text{for } (i,j) \in E .
 \label{eqn:DiscreteDiff}
 \end{equation}
We can then define the \emph{discrete Dirichlet form} between $u,v \colon X\to\R$ as
\begin{align}
\label{eqn:discreteform}
b(u,v) \coloneqq \applied{\delta u}{\delta v}_{L^2(E,w)}.
\end{align}

\medskip

In the continuous setting, on the domain $V\coloneqq H^1(\M,\mu)$ (the Sobolev space of functions in $L^2(\M , \mu)$ with distributional derivative in $L^2(\M, \mu)$) 
we define the Dirichlet form $D\colon V\times V \to \R$ as
\begin{align}
\label{eqn:contform}
D(f,g) \coloneqq \int_\M \applied{\nabla f}{ \nabla g}_x p^2(x) dVol(x),
\end{align}
where $Vol$ stands for the Riemannian volume form of $\M$, $ \nabla f$ and $ \nabla g$ are the gradients of $f$ and $g$ and $\langle \cdot, \cdot \rangle$ represents the Riemannian metric induced on $\M$. Since $p$ is bounded from above, this symmetric bilinear form is continuous, i.e.\
$\abs{D(f,g)} \leq C' \norm{f}_V \norm{g}_V$ for a suitable constant $C'>0$ and all $f,g \in V$. For the remainder we use  $b(u)$ and $D(f)$ as shorthand for $b(u,u)$ and $D(f,f)$, respectively.

Next, we choose measures on $X$ and on the manifold $\M$ and define corresponding operators associated with the forms $b$ and $D$ on $L^2(X)$ and $L^2(\M)$, respectively.
The idea is that by modifying the inner product in $L^2(X)$ and in $L^2(\M)$ we obtain different realizations of Laplacian operators.  The so-called unnormalized and random walk
graph Laplacian (see definitions below), as well as their continuous counterparts, are instances of the general framework that we consider. Let $\mu_n$ be the empirical measure of the random sample, i.e.
\[ \mu_n = \frac{1}{n} \sum_{i=1}^n \delta_{x_i}. \]
On $X$ we consider the measure $\mu_n$ endowed with a density $\vec{m}=(m_1,\dots,m_n)$, denoted by $\vec{m}\mu_n$. 
On the other hand, on $\M$, we consider the measure
$\rho\mu$, where $\rho$ is a Lipschitz continuous density {with Lipschitz constant $L_\rho$ with respect to $\mu$} satisfying
\begin{equation}
\frac{1}{\alpha} \leq \rho(x) \leq \alpha \quad \text{for all } x \in \M.
\label{eqn:DensityBounds}
\end{equation}

On the graph $\Gamma = \Gamma(X, \vec{m}\mu_n, E ,w)$, we define the associated \emph{weighted graph Laplacian} $\Delta_\Gamma$
as $\delta^* \delta$, i.e.\ as the unique operator satisfying
\[ \applied{ \Delta_\Gamma u}{v}_{L^2(X,\vec{m}\mu_n)} =  \applied{\delta u}{\delta v}_{L^2(E,w)}\]
for all $u,v \in L^2(X)$.

At the continuum level, we define a weighted Laplacian associated with the form $D$ and the measure $\rho \mu$ as follows. On the domain
\[ \Dom(\Delta) \coloneqq \Big\{ f \in V: \, \exists\, h \in L^2(\M,\rho\mu)\text{ s. t. }D(f,g)=
 \applied{h}{g}_{L^2(\M, \rho \mu)}
 \; \forall \, g\in V\Big\} \]
we set $\Delta f \coloneqq h$. The operator $\Delta$ is formally defined as 
\[  \Delta f  =- \frac{1}{\rho p} \divergence(  p^2 \nabla f),\]
where $\divergence$ stands for the divergence operator on $\M$.

One of the main results of this paper is that the spectrum of $\Delta_\Gamma$ approximates well that of $\Delta$. Intuitively, one of the elements needed for this to be true is that the measure $\vec{m} \mu_n$ approximates  $\rho\mu$ as $n \rightarrow \infty$.  We use
\begin{equation} \label{eq:mrho}
 \lVert \vec{m} - \rho \rVert_\infty := \max_{i=1, \dots, n} \lvert m_i - \rho(x_i) \rvert 
\end{equation}
to quantify this approximation. 

We now describe particular forms of the graph Laplacian which frequently used in the machine learning literature.

\subsubsection{Unnormalized graph Laplacian}   \label{sec:unGL}
To obtain the unnormalized graph Laplacian, we choose the density vector $\vec{m}$
as $(1,1,\dots)$. Then $\Delta_\Gamma$ is explicitly given by
\[ (\Delta_\Gamma u)(x_i) = {\frac{2}{\sigma_\eta \,h^2 }} \sum_{j : i\sim j} w_{i,j} (u(x_i) - u({x_j})) \]
for all $x_i \in X$, {which is, up to the factor $\frac{2}{\sigma_\eta \,h^2}$, known as the unnormalized graph Laplacian.}
In this case  $\rho \equiv1$, since $\rho$ is the limit of $\vec{m}$ as $n \to \infty$.
This results in a realization of the Laplacian on $L^2(\M,\rho\mu)$ that satisfies
\[ \int_\M \Delta f g p(x) dx = \int_\M  \langle \nabla f,   \nabla g \rangle_x p^2(x) dx = D(f,g)\]
for all $f,g \in \Dom(\Delta)$. In case $p\in C^1(\M)$, this operator $\Delta$ coincides with
\[ \Delta f =  -p\cdot \Delta_2 f  = - \frac{1}{p} \divergence(p^2 \nabla f )\] 
from Definition 8 of \cite{HeAuvL07}, where it was identified as the 
pointwise limit of the unnormalized graph Laplacian.


\subsubsection{Random walk graph Laplacian}   \label{sec:rwGL}
 In order to obtain the random walk graph Laplacian,
we choose the density vector $\vec{m}$ as the vertex degrees, i.e.\
\begin{align}
\label{eqn:minormlaplace} 
m_i \coloneqq {\sum_{j=1}^n w_{ij} = }\frac{1}{nh^m} \sum_{j=1}^n \eta\left(\frac{\abs{x_i-x_j}}{h}\right) \quad \text{for }i\in\{1,\dots,n\}
\end{align}
and $\rho(x) =  p(x)$ for all $x \in \M$. Then $\Delta_\Gamma$ is given by
\[ (\Delta_\Gamma u)(x_i) = {\frac{2}{\sigma_\eta \,h^2}} \sum_{j : i\sim j} \frac{w_{i,j}}{m_i} \big(u(x_i) - u({x_j})\big) \]
for all $x_i \in X$ and $\Delta$ satisfies 
\[ 
\int_\M \Delta f\cdot g \cdot p^2 d\vol =  D(f,g)\]
for all $f,g \in \Dom(\Delta)$.
In case that $p\in C^1(\M)$, 
$
 \Delta$ is nothing but 
\[ \Delta f = -\Delta_2  f = - \frac{1}{p^2} \divergence(p^2 \nabla  f)\]
from \cite[Definition 8]{HeAuvL07}. In the remainder we use $\Delta^{rw}_\Gamma$ to denote 
the random walk graph Laplacian and $\Delta^{rw}$ for its continuous counterpart. 
\red Showing the closeness of $\vec{m}$ and $\rho$, \eqref{eq:mrho}, reduces to showing a kernel density estimate on a manifold.  In the Appendix \ref{sec:kdeA} we show that provided $h$ satisfies Assumption \ref{stdassumptions}, we have
\begin{equation}
 \max_{i=1, \dots, n}  \lvert m_i - p(\x_i) \rvert  \leq  CL_p h  +  C \alpha \eta(0) m \omega_m \frac{\veps}{h} + C \alpha m \left(K + \frac{1}{R^2} \right)h^2 ,  
 \label{LinfWeights}
 \end{equation}
where $C>0$ is a universal constant and $\veps$ is the $\infty$-OT distance between $\mu_n$ and $\mu$ (see \eqref{epsilon} and Section \ref{sec:transport}). These estimates are proved using a simple and general approach using the transportation maps introduced in Section \ref{sec:transport};
in contrast to  usual kernel density estimation approaches. 
The estimates are not optimal, but they are on the same order of error as the approximation error of the Dirichlet form $D$ by the discrete Dirichlet form $b$ that we present in Lemma  \ref{lem:discretization} and Lemma \ref{lem:interpolation}; the bottom line is that the rates of convergence for the spectrum of the random walk graph Laplacian are unaffected by the non-optimal estimate \eqref{LinfWeights}. On the other hand our proof of \eqref{LinfWeights} has the advantage of reducing all probabilistic estimates in our problem to estimating the $\infty$-OT distance between $\mu_n$ and $\mu$; which is done in Section \ref{sec:transport}.  \nc


\subsubsection{Normalized graph Laplacian}   \label{sec:nGL}
So far we have described how one can obtain the unnormalized and random walk Laplacians as examples of the general framework introduced in this section.  Let us recall another popular version of normalized Laplacian usually referred to as \textit{symmetric} normalized graph Laplacian.  For given $u\colon X \rightarrow \R$, the symmetric normalized Laplacian of $u$ is given by
\[ (\Delta_\Gamma^S u)(x_i) \coloneqq  {\frac{2}{\sigma_{\eta} \,h^2}} \sum_{j : i\sim j} \frac{w_{i,j}}{\sqrt{m_i}} \left( \frac{u(x_i)}{\sqrt{m_i}} - \frac{u({x_j})}{\sqrt{m_j}} \right) \]
with $m_i$ defined by \eqref{eqn:minormlaplace}.
We remark that $\Delta^S_{\Gamma}$ can not be obtained by appropriately choosing the measure $\vec{m}\mu$ as described in this section (in order to recover it we would have to modify the definition of discrete 
differential in \eqref{eqn:DiscreteDiff}).  Nevertheless, we can indirectly analyze the rate convergence of its spectrum towards that of a continuous counterpart 
noting that $\Delta^S$ and $\Delta^{rw}$ are similar matrices.
Indeed, we recall that
$\Delta_\Gamma^S u = \lambda u$ if and only if $\Delta^{rw}_\Gamma v = \lambda v$ where $v(x_i)\coloneqq m_i^{-1/2}u(x_i)$. Thus, $\Delta_\Gamma^{rw}$ and $\Delta_\Gamma^S$ share the same spectrum.


\subsection{Main results}
\label{sec:mainresults}

\subsubsection{Convergence of eigenvalues \red and transportation estimates \nc}
Our first main result is the following.
\begin{theorem}\label{thm:rate}
	Let $x_1, \dots, x_n$ be i.i.d. samples from a distribution $\mu$ supported on $\M$, with density $p$ satisfying \eqref{def:alphap}. Consider $\vec{m}$ and $\rho$ as in Section \ref{sec:unGL} or Section \ref{sec:rwGL}. For $k \geq 2$ let $\lambda_k(\Gamma)$ be the $k$-th eigenvalue of the graph Laplacian $\Delta_\Gamma$ defined in  Section \ref{sec:DirichletForms} with
		\[ h := \sqrt{\frac{\log(n)^{p_m}}{n^{1/m}}}, \]
		where $p_m =3/4$ if $m=2$ and $p_m=1/m$ if $m\geq 3$.
		 Let $\lambda_k(\M)$ be the $k$-th eigenvalue of the Laplacian $\Delta$ defined in Section \ref{sec:DirichletForms}. Then,
		\[ \frac{|\lambda_k(\Gamma) - \lambda_k(\M)|}{\lambda_k(\M)} = O\Bigg(  \sqrt{\frac{\log(n)^{p_m}}{n^{1/m}}}\Bigg), \quad \textrm{ almost surely}.\]
	\end{theorem}




\red
 The actual choice of $h$ in the previous theorem is explained by the more general and detailed result stated in Theorem  \ref{thm:conveigenvalues}, together with the estimates for the $\infty$-OT distance between $\mu$ and $\mu_n$ in Theorem \ref{thm:transport}. Indeed, we have taken $h$ to scale like $\sqrt{\veps}$ where $\veps$ is the $\infty$-OT distance between $\mu$ and $\mu_n$. More precisely, 
\begin{equation}
 \veps = d_\infty(\mu, \mu_n) := \min_{T : T_{\sharp } \mu = \mu_n} \esssup_{x \in \M} d(x,T(x)).
\label{epsilon}
\end{equation}  
where $T_{\sharp } \mu = \mu_n$ means that $\mu(T^{-1}(U)) = \mu_n(U)$ for every Borel subset $U$ of $\M$. Such mappings $T$ are called transport maps from $\mu$ to $\mu_n$.
One of the key ingredients needed to establish Theorem \ref{thm:rate} is the probabilistic estimate on $\infty$-OT distance contained in our next theorem. 
\nc

\begin{theorem} \label{thm:transport}
	Let $\mathcal{M}$ be a smooth, connected, compact manifold with dimension $m$. Let $p\colon\M \rightarrow \R$ be a probability density satisfying \eqref{eqn:DensityBounds} and consider the measure {$d \mu= p\, d\vol$}. Let $x_1,\dots, x_n$ be an i.i.d sample of $\mu$. 
	Then, for any $\beta>1$ and every $n\in\N$
	there exists a transportation map $T_n\colon \M \to X$ and a constant $A$ such that
	\begin{equation} \label{winfest}
	\sup_{x \in \M} d(x, T_n(x)) \leq  \ell := A 
	\begin{cases}
	\frac{\log (n)^{3/4}}{n^{1/2}},  & \text{if } m =2, \medskip \\
	\frac{(\log n)^{1/m}}{n^{1/m}},  & \text{if } m \geq 3,
	\end{cases}
	\end{equation}
	holds with probability at least $1-C_{K,\vol(\M),m,i_0} \cdot n^{-\beta}$,
	where $A$ depends only on $K$, $i_0$, $m$, $\vol(\M)$, $\alpha$ and $\beta$. \nc
\end{theorem}
The exact {dependency} of $A$ in \eqref{winfest} on the geometry of $\M$ is given in Lemma \ref{lemma:W8L8}. \red We remark that the scaling on $n$ on the right-hand side is optimal, even in the Euclidean case \cite{GTS15a}. \nc

With the estimates in Theorem \ref{thm:transport} at hand, Theorem \ref{thm:rate} follows from the more general Theorem \ref{thm:conveigenvalues} below (more precisely from its corollaries). Indeed, convergence rates for the spectrum of graph Laplacians can be written in terms of $h$ and $\veps$ as long as $0 < \veps \ll h \ll 1$.  Throughout this paper we assume that $h,\veps, \frac{\veps}{h}$ and $\lVert \vec{m}-\rho \rVert_\infty$ are sufficiently small. In particular we make the following assumptions.
\begin{assumption}
	\label{stdassumptions}
	Assume that 
	\[ h < \min \left\{1,  \frac{i_0}{10}, \frac{1}{\sqrt{mK}}, \frac{R}{\sqrt{27m}}\right \} \quad \te{and} \quad 
	{(m+5)} \veps < h, \]
	where $i_0$ is the injectivity radius of the manifold $\M$, $K$ is a global upper bound on the absolute value of sectional curvatures of $\M$, $m$ is the dimension of $\M$, and $R$ is the reach of $\M$ (seen as a submanifold embedded in $\R^d$).
\end{assumption}
\nc

%

\begin{theorem}
	\label{thm:conveigenvalues}
	For $k \in \N$ let $\lambda_k(\Gamma)$ be the $k$-th eigenvalue of the graph Laplacian $\Delta_\Gamma$ defined in  Section \ref{sec:DirichletForms} using the weights $\vec{m}$, and let $\lambda_k(\M)$ be the $k$-th eigenvalue of the Laplacian $\Delta$ defined in Section \ref{sec:DirichletForms} using the weight function $\rho$. Finally let $\veps$ be the $\infty$-OT distance between $\mu_n$ and $\mu$ and assume that $h>0$ satisfies Assumptions \ref{stdassumptions}. Then,
\begin{enumerate}
		\item(Upper bound)  If $\veps$ and $\lVert \vec{m}-\rho \rVert_\infty$  are such that
		\begin{equation}
		\sqrt{\lambda_k(\M)}\, \veps + \lVert \vec{m}- \rho \rVert_\infty   < c,
		\label{Prop1Condition}
		\end{equation}
		 for a positive constant $c$ that depends only on $m, \alpha, L_\rho, L_p$ and $ \eta$,
 		then, 
		\begin{equation} \label{evupb}
		 \frac{\lambda_k(\Gamma) -  \lambda_k(\M)}{\lambda_k(\M)}  \leq \tilde{C} \left( L_p h +  \frac{\veps}{h} +  \sqrt{\lambda_k (\M)} \veps + K h^2+ \frac{h^2}{R^2} + \lVert \vec{m} - \rho \rVert_\infty  \right) 
		\end{equation}
		where $\tilde{C}$ only depends on $m, \alpha, L_\rho, L_p$, and $\eta$.
		\item (Lower bound) If $h$ and $\lVert \vec{m}- \rho \rVert_\infty$ are  such that
		\begin{equation}
		\sqrt{\lambda_k(\M)}h + \lVert \vec{m}- \rho \rVert_\infty   < c,
		\label{Prop2Condition}
		\end{equation}
		for a positive constant $c$ that depends only on $m, \alpha, L_\rho, L_p$, and $ \eta$, then, 
		\begin{equation} \label{evlowb}
	\frac{\lambda_k(\Gamma) -  \lambda_k(\M)}{\lambda_k(\M)} \geq - \tilde{C}\left( L_p h    +\frac{\veps}{h} +  \sqrt{\lambda_k(\M)} h + K h^2 + \norm{\vec{m}- \rho}_\infty  \right)
		\end{equation}
	where $\tilde{C}$ only depends on $m, \alpha, L_\rho, L_p$, and $ \eta$.
	\end{enumerate}
\end{theorem}

\begin{remark}
Note that the lower bound does not depend on the reach $R$. This is due to the one sided-inequality 
\[ |x-y| \leq d(x,y) , \quad \forall x,y \in \M .\]
In contrast, for the upper bound one must use a reverse inequality with an additional higher order correction term that depends on $R$. See Proposition \ref{prop:metricestimates}. 

It is also worth pointing out that the presence of the term $\sqrt{\lambda_k(\M)} \veps $ in the upper bound ultimately comes from the estimate on how far is the map $P$ in \eqref{def:P} from being an isometry when restricted to the first $k$ eigenspaces of $\Delta$; the relevant length-scale for this estimate is the size of transport cells, i.e., $\veps$. On the other hand, the term $\sqrt{\lambda_k(\M)} h $ in the lower bound comes from the estimate on how far is the map $I$  in \eqref{eqn:InterpolatingOp} from being an isometry when restricted to the first $k$ eigenspaces of $\Delta_\Gamma$; the relevant length-scale for this estimate is $h$, which is of the same order as the bandwidth for the kernel used to define the map $I$. This can be seen from Lemmas \ref{lem:discretization} and \ref{lem:interpolation} respectively.
\end{remark}

\begin{remark}
	From the estimates \eqref{evupb} and \eqref{evlowb} we see that  curvature of $\M$ only introduces a second order correction to the rate of convergence of $\lambda_k(\Gamma)$ towards $\lambda_k(\M)$.
\end{remark}

The estimates on $\veps$ from Theorem \ref{thm:transport} combined with Theorem \ref{thm:conveigenvalues} imply that $\lambda_k(\Gamma) $ converges towards $\lambda_k(\M)$ 
with probability one whenever $\lVert \vec{m} - \rho \rVert_\infty \rightarrow 0$, $h \rightarrow 0$ ,  $\frac{\veps}{h} \rightarrow 0$.  We can specialize Theorem \ref{thm:conveigenvalues} to the examples from Section \ref{sec:DirichletForms}, where in particular we provide estimates on  $\lVert \vec{m} - \rho\rVert_\infty$ in terms of $n$.

\nc

\begin{corollary}[Convergence of eigenvalues unnormalized graph Laplacian]
\label{cor:cor1}
In the context of Theorem \ref{thm:conveigenvalues} suppose that the weights are taken to be $\vec{m} \equiv 1$ and $\rho \equiv 1$.   If $h$ is small enough for
\[ ( \sqrt{\lambda_k(\M)}+1 )h \leq c , \]
to hold for a positive constant $c$ that depends only on $m, \alpha,  L_p$, and $ \eta$,
then
\begin{equation}
\label{RelativeError1}
 \frac{|\lambda_k(\Gamma) - \lambda_k(\M)|}{\lambda_k(\M)} \leq \tilde C \left(  \frac{\veps}{h} +  (1+\sqrt{\lambda_k (\M)}) h + \left( K+ \frac{1}{R^2}\right) h^2 \right)  , 
 \end{equation}
 where $\tilde{C}$ only depends on $m, \alpha, L_p$, and $\eta$.
\end{corollary}
\begin{proof}
The result follows directly from Theorem \ref{thm:conveigenvalues} after noticing that in this case $\lVert \vec{m} - \rho \rVert_\infty =0$ and $L_\rho=0$.
\end{proof}
\begin{corollary}[Convergence of eigenvalues random walk graph Laplacian]
\label{cor:cor2}
In the context of Theorem \ref{thm:conveigenvalues} suppose that the weights $\vec{m}$ are as in \eqref{eqn:minormlaplace} and $\rho \equiv p$.  If $h$ and $\veps/h$ are such that
\[ ( \sqrt{\lambda_k(\M)}+1 )h + \frac{\veps}{h} \leq c , \]
for a positive constant $c$ that depends only on $m, \alpha,  L_p$, and $\eta$,
then,
\begin{equation}
\label{RelativeError2}
\frac{|\lambda_k(\Gamma) - \lambda_k(\M)|}{\lambda_k(\M)} \leq \tilde C \left(  \frac{\veps}{h} +  (1+\sqrt{\lambda_k (\M)}) h + \left( K+ \frac{1}{R^2} \right) h^2 \right) ,  
\end{equation}
where $\tilde{C}$ only depends on $m, \alpha,  L_p$, and $\eta$.
\end{corollary}
\begin{proof}
The result follows directly from Theorem \ref{thm:conveigenvalues} after using \eqref{LinfWeights}. Indeed, notice that the term $\lVert \vec{m} - \rho \rVert_\infty$ can be absorbed in the $h$, $\frac{\veps}{h}$ and $h^2$ terms by enlarging constants if necessary.
\end{proof}

\begin{remark}
Notice that the estimates in the previous results provide a lower bound on the mode at which the spectrum of the graph Laplacian stops being informative about the spectrum of the Laplace-Beltrami operator. Namely, notice that the right hand sides of \eqref{RelativeError1} and \eqref{RelativeError2} are small when $h \sqrt{\lambda_k(\M)}$ is small. Using Weyl's law for the growth of eigenvalues of the Laplace-Beltrami operator
	 we know that
	\[ \sqrt{\lambda_k(\M)} \sim k^{1/m}, \]
 and thus, the relative error of approximating $\lambda_k(\M)$ with $\lambda_k(\Gamma)$ is small when $k \lesssim \frac{1}{h^{m}}$ and $\varepsilon \ll h$. In particular, if $h$ is taken to scale like $h=\sqrt{\veps}$
 (as is the case in Theorem \ref{thm:rate})
 then $\lambda_k(\M)$ is approximated by $\lambda_k(\Gamma)$ if $k \lesssim \sqrt{\frac{n}{\log(n)}}$ for $m \geq 3$ and $k \lesssim \sqrt{\frac{n}{\log(n)^{3/2}}}$ for $m=2$.
\end{remark}

\begin{remark}
	\label{rem:probestimates}
 We would like to remark that one of the main advantages of writing all our estimates in Theorem \ref{thm:conveigenvalues} in terms of the quantity $\veps$ (which is the only one where randomness is involved) is that we can transfer probabilistic estimates for $\veps$ into probabilistic estimates for the error of approximation of $\lambda_k(\Gamma)$. In particular, when combined with Theorem \ref{thm:transport}, Corollary 1 and Corollary 2 can be read as follows: Suppose that $\frac{\log(n) ^{p_m}}{n^{1/m}} \ll h \ll 1$.  Let $k:= k_n$ be such that $k_n \ll \frac{1}{h^m}$. Let $\beta>1$. Then, with probability at least $1-C_{K,\vol(\M),m,i_0} n^{-\beta}$, 
\[ \frac{|\lambda_j(\Gamma) - \lambda_j(\M)|}{\lambda_j(\M)} \leq \tilde C_\beta \left(  \frac{\log(n)^{p_m}}{hn^{1/m}} +  (1+\sqrt{\lambda_j (\M)}) h + \left( K+ \frac{1}{R^2} \right) h^2 \right) \]
for all $j=1, \dots, k_n$. 
\end{remark}

\begin{remark}
Moreover,  writing all our estimates in Theorem \ref{thm:conveigenvalues} in terms of the quantity $\veps$ is also convenient because the conclusion of the theorem holds even when the points $x_1, \dots, x_n$ are not i.i.d. samples from the measure $\mu$. That is one only needs to ensure that the assumption \eqref{Prop1Condition} is satisfied. In other words whenever one has an
 an estimate on the $\infty$-OT distance between the point cloud and the measure $\mu$ and a kernel density estimate to ensure that \eqref{Prop1Condition} holds, the theorem provides an error estimate on the eigenvalues. We note that the kernel density estimate in terms of the $\infty$-OT distance we provide in Lemma \ref{lem:kde} implies that one in fact only needs  an estimate on the $\infty$-OT distance between the point cloud and the measure $\mu$. 
\end{remark}



\subsubsection{Convergence of eigenfunctions}

We prove that eigenvectors of $\Delta_\Gamma$ converge towards eigenfunctions of $\Delta$ and provide quantitative error estimates. To make the statements precise, we need to make sense of how to compare functions defined on the {graph/sample} $X$ with functions 
defined on the manifold $\M$. 
In this paper we consider two different ways of doing this.

The first approach involves an interpolation step by composing with the 
optimal transportation map $T : \M \to X$ from \eqref{epsilon} followed by a smoothening step.
Both of these steps require the knowledge of $\M$.
 The map $T$ induces a partition $U_1, \dots, U_n$ of $\M$ where 
\begin{equation} \label{def:Ui}
U_i := T^{-1}(\{ x_i \}).
\end{equation}
\red We note that $\mu(U_i)=\frac{1}{n}$ for all $i=1, \dots,n$. \nc
We define the \emph{contractive discretization} map $P\colon L^2(\M, \rho \mu) \to L^2(X, \vec{m}\mu_n)$ by 
\begin{equation}
(Pf)(x_i) \coloneqq n \cdot \int_{U_i} f(x) d\mu(x), \quad f \in  L^2(\M, \rho \mu),
\label{def:P}
\end{equation}
and the \emph{extension} map $P^*\colon L^2(X, \vec{m}\mu_n) \to L^2(\M, \rho \mu)$ by
\begin{equation}
(P^*u)(x) \coloneqq \sum_{i=1}^n u(x_i) \mathds{1}_{U_i} (x), \quad u \in L^2(X, \vec{m}\mu_n).
\label{def:P*}
\end{equation}
We note that $P^*u$ can be written as $P^*u = u \circ T$. We then consider the interpolation operator $I\colon L^2(X, \vec{m}\mu_n) \to \Lip(\M)$
\begin{equation}
Iu \coloneqq \Lambda_{h - 2\veps} P^*u
\label{eqn:InterpolatingOp}
\end{equation}
where $\Lambda_{h - 2\veps}$ is defined in \eqref{smoothingLam} and is simply a convolution operator using some particularly chosen kernel; see Section \ref{sec:discuss} for a discussion on why we need to consider a specific kernel.

%
%

\begin{theorem}
\label{thm:Iuapproxf}
Let $\Delta_\Gamma$ be the graph Laplacian defined in Section \ref{sec:DirichletForms} using the weights $\vec{m}$, and let  $\Delta$ be the Laplacian defined in Section \ref{sec:DirichletForms} using the weight function $\rho$. Let $\veps$ be the $\infty$-OT distance between $\mu_n$ and $\mu$ and assume that $h>0$ satisfies Assumptions \ref{stdassumptions}.  Finally, assume that $h$ and $\lVert \vec{m}-\rho \rVert_\infty$ are small enough so that 
\[ (1+ \sqrt{\lambda_k(\M)} )h + \lVert  \vec{m}-\rho\rVert_\infty \leq c,  \]
for a constant $c$ that depends only on $m, \alpha, L_p , L_\rho, \eta$. 

Then, for every $u\in  L^2(X, \vec{m} \mu_n)$ normalized eigenfunction of $\Delta_\Gamma$ corresponding to the eigenvalue $\lambda_k(\Gamma)$, there exists a normalized eigenfunction $f\in L^2(\M,\rho\mu)$ of $\Delta$ corresponding to the $k$-th eigenvalue $\lambda_k(\M)$
such that
\[ \norm{I u - f}_{L^2(\M,\rho\mu)} \leq \frac{\tilde C}{ \red \gamma_{k,\rho\mu} \nc}\left(\frac{\veps}{h} + (1+\sqrt{\lambda_k(\M)}) h  \red + Kh^2 + \frac{h^2}{R^2}  + \lVert \vec{m} - \rho \rVert_\infty  \right),\]
 where $\tilde C$ is a constant that only depends on $m,\eta, \alpha, L_\rho, L_p$ and where $\gamma_{k, \rho\mu}$ is the difference between the smallest eigenvalue of $\Delta$ that is strictly larger than $\lambda_k(\M)$ and the largest eigenvalue of $\Delta$ that is strictly smaller than $\lambda_k(\M)$ (i.e a spectral gap). 
 
 In particular, if we take
 \[ h := \sqrt{\frac{\log(n)^{p_m}}{n^{1/m}}}, \]
 where $p_m =3/4$ for $m=2$ and $p_m=1/m$ for $m\geq 3$, then, 
 \[ \norm{I u - f}_{L^2(\M,\rho\mu)} = O\Bigg(  \sqrt{\frac{\log(n)^{p_m}}{n^{1/m}}}\Bigg), \quad \textrm{ almost surely}.\]
 
\end{theorem}


\begin{remark}
As in Remark \ref{rem:probestimates}, we would like to emphasize that the probabilistic estimates for $\veps$ translate directly into probabilistic estimates for the convergence of eigenfunctions in Theorem \ref{thm:Iuapproxf}. Likewise, we would like to point out that Theorem \ref{thm:Iuapproxf} can be made concrete in the context of Sections \ref{sec:unGL} and \ref{sec:rwGL} using the corresponding estimates for $\lVert \vec{m}- \rho\rVert_\infty$ in terms of $\veps$ and $h$.
\end{remark}

The second approach to compare eigenvectors of $\Delta_\Gamma$ with eigenfunctions of $\Delta$ is to extrapolate the values of discrete eigenvectors to the Euclidean Voronoi cells induced by the points $\{x_1, \dots, x_n \}$. That is, for an arbitrary function $u: X \rightarrow \R$ we assign to each point $x\in \M$ the value $u(x_i)$ where $x_i$ is the nearest neighbor of $x$ in $X$ with respect to the Euclidean
distance. More formally, for $i\in \{1,\dots,n\}$ we consider the Voronoi cells
\begin{align}
\label{eqn:voronoidef}
 V_i \coloneqq \{ x\in \M : \abs{x-x_i}  = \min_{j=1,\dots,n} \abs{x - x_j} \},
\end{align}
and define the function $\bar u \in L^2(\M, \rho \mu)$ by
\begin{align}
\label{eqn:barudef}
\bar u(x)\coloneqq \sum_{i=1}^{n}u(x_i) \1_{V_i}(x) \quad \text{for } x \in \M.
\end{align}
We notice that the Voronoi cells $V_1, \dots, V_n$ form a partition of $\M$, up to a set of ambiguity of $\mu$-measure zero. Besides being a computationally simple interpolation, the Voronoi extension can be constructed exclusively from the data and no information on $\M$ is needed. We show that the interpolation $\overline{u}$ of a discrete eigenvector $u$ approximates the corresponding eigenfunction $f$ on $\M$ with
almost the same rate as in Theorem \ref{thm:Iuapproxf}.  In order to obtain convergence of the Voronoi extensions $\bar u$, we require $h=h_n$ to satisfy 
\begin{align}
\label{eqn:hncondition}
 \lim_{n\to\infty} \log^{mp_m}(n) \cdot \left(h+ \frac{\veps}{h}\right) = 0
 \end{align}
This condition holds, for instance, when $h$ is chosen as $\sqrt{\veps}$, which minimizes the error
in the following result.

\begin{theorem}
\label{thm:voronoiapprox}

Fix $\beta>1$. Let $\Delta_\Gamma$ be the graph Laplacian defined in Section \ref{sec:DirichletForms} using the weights $\vec{m}$, and let  $\Delta$ be the Laplacian defined in Section \ref{sec:DirichletForms} using the weight function $\rho$. Let $\veps$ be the $\infty$-OT distance between $\mu_n$ and $\mu$ and assume that $h>0$ satisfies Assumptions \ref{stdassumptions}.  Finally, assume that $h$ and $\lVert \vec{m}-\rho \rVert_\infty$ are small enough so that in particular
\[ (1+ \sqrt{\lambda_k(\M)} )h + \lVert  \vec{m}-\rho\rVert_\infty \leq c,  \]
for a constant $c$ that depends only on $m, \alpha, L_p , L_\rho, \eta$. 

Then, with probability at least $1-C_{m,K,\vol(\M),i_0}\cdot n^{-\beta}$, for every $u\in  L^2(X, \vec{m} \mu_n)$ normalized eigenfunction of $\Delta_\Gamma$ corresponding to the eigenvalue $\lambda_k(\M)$,  it holds
\begin{align}
\label{est:THM6}
\begin{split}
 \norm{\bar{u} - f}_{L^2(\M,\rho\mu)}  \leq & \frac{\tilde C_\beta \sqrt{\log(n)^{m p_m}}}{ \red \gamma_{k,\rho\mu} \nc}\Bigg(\frac{\veps}{h} +  (1+\sqrt{\lambda_k(\M)}) h  \\& + \red Kh^2 + \frac{h^2}{R^2} 
  + \lVert \vec{m} - \rho \rVert_\infty  \Bigg)      +  C_\M \lambda_k(\M)^{\frac{m+1}{4}} \veps,
\end{split}
\end{align}
where $f$ and $\gamma_{k, \rho \mu}$ are as in Theorem \ref{thm:Iuapproxf}, $\bar u$ is as in \eqref{eqn:barudef}, and $C_\M$ is a constant that depends on the manifold $\M$.
%
%
\end{theorem}

\begin{remark}
We remark that the first term in \eqref{est:THM6} is worse than the estimate in Theorem \ref{thm:Iuapproxf} by a logarithmic factor of $n$. This is due to our uniform estimates on the size of Voronoi cells based on transportation (see Lemma \ref{lem:UivsVi}). On the other hand, the extra factor in \eqref{est:THM6} is an estimate for the difference of the averages of $f$ over transport cells and Voronoi cells; here we use the regularity of an eigenfunction $f$ and in particular we use a bound for $\lVert \nabla f \rVert_\infty$ found in \cite{ShiXu10}.  
\end{remark}

\nc
\subsection{Outline of the approach and discussion} 
\label{sec:discuss}
To prove our main results we exploit well-known variational characterizations for the spectra of $\Delta_\Gamma$ and $\Delta$. Our results are then deduced from a careful comparison between the objective functionals of the variational problems.  

From the definition of $\Delta_\Gamma$ in Section \ref{sec:DirichletForms} it clear that $\Delta_{\Gamma}$ is positive-semidefinite with respect to the inner product of $L^2(X,\vec{m}\mu_n)$.  We denote by
\[ 0 = \lambda_1(\Gamma)\leq \lambda_2(\Gamma) \leq \lambda_3(\Gamma) \leq \dots\]
the eigenvalues of $\Delta_\Gamma$, repeated according to their multiplicities. By the minmax principle we have
\begin{equation}
\lambda_k(\Gamma) = \min_{L_k} \max_{u\in {L_k} \setminus\{0\}} \frac{b(u)}{\norm{u}_{L^2(X,\vec{m}\mu_n)}^2}  
\label{eqn:Courant0}
\end{equation}
where the minimum is over all $k$-dimensional subspaces {$L_k$} of $L^2(X, \vec{m}\mu_n)$. At the continuum level, and given that $\rho$ and $p$ are bounded from below, one can show that $\Delta$ is a closed and densely defined symmetric operator with compact resolvent  \cite[Lemma 2.7]{AreEls12}. Therefore, its spectrum consists of positive eigenvalues only, which we denote by
\[ 0 = \lambda_1(\M) \leq \lambda_2(\M) \leq \lambda_3(\M) \leq \dots, \]
where eigenvalues are repeated according to their multiplicities. Moreover, by Courant's minmax principle we have
\begin{align}
\label{eqn:Courant}
\lambda_k(\M) = \min_{{L_k}} \max_{f\in {L_k}\setminus\{0\}} \frac{D(f)}{\norm{f}^2_{L^2(\M,\rho\mu)}}
\end{align}
where the {minimum} is over all $k$-dimensional subspaces $L_k$ of $L^2(\M,\rho\mu)$, see \cite[Lemma 2.9]{MugNit12}. 
%
%
\nc

The proof of our results may be split into two main parts. The first part contains all the probabilistic estimates needed in the rest of the paper 
and is devoted to the proof of Theorem \ref{thm:transport}. The study of the estimates for $d_\infty(\mu,\mu_n)$ goes back to \cite{ShorYukich,LeightonShor,TalagrandGenericChain} where the problem was considered in a simpler setting: $\mu$ is the Lebesgue measure on the unit cube $(0,1)^m$ and the points $x_1,\dots,x_n$ are i.i.d.\
uniformly distributed on $(0,1)^d$. In that context, with very high probability,
\[  d_\infty(\mu, \mu_n)  \approx \frac{(\log(n))^{p_m}}{n^{1/m}}, \]
where $p_m$ is defined in \eqref{defpm}. In \cite{GTS15a} the estimates are extended to measures defined on more general domains 
(not just $(0,1)^d$) and with more general densities (not just uniform). In this paper we extend the results in \cite{GTS15a} to the manifold case. In order to prove Theorem \ref{thm:transport}, we use a similar proof scheme to the one used in \cite{GTS15a}. Indeed, we first establish Lemma \ref{lemma:W8L8} below which is analogous to \cite[Theorem 1.2]{GTS15a} and is of interest on its own. \red The result includes 
explicit estimates on how the distance depends on the geometry of $\M$. \nc
\begin{lemma} \label{lemma:W8L8}
Let $\rho_1$, $\rho_2$ be two probability densities defined on $\M$ with 
\[  \frac{1}{\alpha} \leq \rho_i(x) \leq \alpha \quad \text{for all } x \in \M \text{ and } i\in\{1,2\} \]
for some $\alpha \geq 1$. Then {it holds }for the corresponding measures $\nu_1$, $\nu_2$, defined as $d\nu_1 = \rho_1 dx$ and $d\nu_2 = \rho_2 dx$,
\begin{equation}
 d_\infty(\nu_1, \nu_2) \leq A \lVert \rho_1 - \rho_2\rVert_{L^\infty(\M)}, 
 \label{eqn:W8L8}
 \end{equation}
where
$ A= \frac{C_{m,\alpha} \vol(\M)^3}{r^{3m-1}} \max \left\{  \tilde \nt  , \frac{\diam (\M)}{r } \right\}$, 
$\tilde \nt  = (1+ C m K r^2) \frac{2^m\vol(\M)}{\omega_m r^m}$ and $r = \frac15 \min\{1, i_0, \frac{1}{\sqrt K}  \} $.
\end{lemma}
With Lemma \ref{lemma:W8L8} at hand, the next step is to construct a careful partition of the manifold $\M$ into patches in which we can use directly the results from \cite{GTS15a}. 
The construction requires some geometric estimates which are obtained in Section \ref{sec:NicePartition}. Using properties of the constructed partition of $\M$ and Lemma \ref{lemma:W8L8}, 
we can establish Theorem \ref{thm:transport}. 
%

The second part of the proof of our main results consists of a set of precise deterministic computations used to relate the discrete and continuum Dirichlet energies appearing in the variational characterization of the spectra of the graph and continuum Laplacians; these computations are based on ideas from \cite{BIK}. Roughly speaking, the proof of our main results relies on the following upper and lower bounds. We first show the \textit{upper bound}
\[ b(P f) \leq (1+error) E_r ( f ) \leq (1+ error) D(f), \quad f \in L^2(\M, \rho\mu), \]
where $E_r$ is the non-local kernel approximation of the continuum Dirichlet energy defined in \eqref{eqn:ErfV} and $r$ is a length-scale which up to leading order is equal to $h$; the term $error$ can be explicitly written in terms of $h$, $\veps$ and geometric quantities associated to the manifold $\M$. It is possible to interpret the first inequality as a ``variance'' estimate as it relates an energy constructed exclusively from the graph with an ``average'' energy. The second inequality on the other hand can be thought as a ``bias'' estimate. We would like to point out that the second inequality is a manifestation of the intuitive fact that local energies bound non-local ones. Our \textit{lower bound} takes the form
\[ D(I(u))\leq (1+error)E_r(P^* u) \leq (1+error)b(u), \quad u \in L^2(X). \]
We remark that it is not too hard to obtain a relation of the form $D(I(u))\leq C E_r(P^* u)$ for some constant $C$. Nevertheless, since our goal is to find \textit{error estimates}, the constant $C$ must be sharp (up to some small error). We obtain this sharp constant using the specific form of the convolution operator $\Lambda$ in the definition of $I$ (see \eqref{eqn:SmootheningOp}). Our analysis of convergence of the spectra is completed by showing that the maps $P$, $P^*$ and $I$ are almost isometries when restricted to eigenspaces (discrete or continuum).


 We want to highlight the fact that in contrast with the construction in \cite{BIK}, our graphs and our ``out-of-sample extensions'' of  eigenvectors are defined exclusively from the ambient space Euclidean distance. Theorem \ref{thm:voronoiapprox} is obtained a posteriori from Theorem \ref{thm:transport}  and uses Theorem \ref{thm:transport} to bound the measure of Voronoi cells.  We also use uniform estimates for the gradient of eigenfunctions of the Laplace--Beltrami operator from \cite{ShiXu10}.

\subsection{Outline} The rest of the paper is organized as follows. In Section \ref{intro:geometry} we present some estimates from differential geometry that we need in the sequel. Section \ref{sec:transport} is devoted to the estimation of the $\infty$-transportation distance between $\mu_n$ and $\mu$ and in particular contains the proof of Theorem \ref{thm:transport}. Section \ref{sec:Kernelbased} contains results on the kernel-based approximation of the Laplacian operator; in more precise terms, we relate the (weighted) Dirichlet energy $D$ with the non-local Dirichlet energy \eqref{eqn:ErfV}. Section \ref{sec:Eigenvalues} addresses the convergence of eigenvalues and in particular contains the proof of Theorem \ref{thm:conveigenvalues}. Finally, in Section \ref{sec:Eigenfunctions} we establish the convergence of eigenvectors of graph Laplacians, first in the sense of the interpolation map $I$ from \ref{eqn:InterpolatingOp} (Theorem \ref{thm:Iuapproxf})  and then in the sense of Voronoi extensions (Theorem \ref{thm:voronoiapprox}). \red The Appendix \ref{sec:kdeA} contains the optimal-transportation-based  proof of kernel-density estimates on manifolds.
\nc

\subsection{Some estimates from differential geometry}
\label{intro:geometry}
We conclude the introduction by recalling some notation and stating a few results from differential geometry.

For a point $x \in \M$, we denote by $T_x\M$ the tangent space of $\M$ at $x$. Fix $0 < r \leq \min\{ i_0,1/\sqrt{K}\}$ and let us denote by $\exp_x\colon B(r) \subseteq T_x\M \rightarrow \M$ the Riemannian exponential map. 
Since $r < i_0$ the map is a diffeomorphism between the ball $B(r)$ and the geodesic ball $B_\M(x,r)$.  In particular $\exp_x^{-1}$ defines a local chart at $x$. Let $g$ be the pull back of the metric of $\M$ by the exponential map. That is for  an orthonormal basis $e_1,\dots, e_m$ of $T_x\M$ and for given $v \in B(r)$ let $g_{i,j}|_v \coloneqq   \langle (d\exp_x)_v(e_i) , (d\exp_x)_v(e_j)  \rangle$, where we have identified the tangent space of $T_x\M$ at $v$ with $T_x \M $ itself.
Then
\begin{equation} \label{metdist}
  \delta_{i,j} - C  K\abs{v}^2 \leq  g_{i,j} \leq \delta_{i,j} + C  K \abs{v}^2 ,   
\end{equation}
where $|v|$ is the Euclidean length of $v$, $\delta_{i,j}$ is $1$ if $i=j$ and $0$ otherwise and where $C$ is a universal constant. \nc Such estimates are bounds on the metric distorsion by the exponential map and follow from Rauch comparison theorem (\cite[Chapter 10]{doCa92} and \cite[Section 2.2]{BIK}).
Similarly, since $r < 1/\sqrt{K}$, one can show that for any $v \in B(r)$ and any $w \in T_x\M \cong T_v(T_x\M)$,
\begin{equation}  \label{expderest}
 \frac12 \lvert w\rvert_x \leq \lvert (d\exp_x)_v(w) \rvert_{\exp_x(v)} \leq 2 \lvert w \rvert_x. 
\end{equation}

\begin{proposition} \label{prop:expbilip}
Assume  $0 < r \leq \min\{ i_0,1/\sqrt{K}\}$.
Let $p\in\M$ and consider any smooth curve $\gamma\colon[0,1] \to B(r) \subset T_p \M$. Then 
\[ \frac12 \length(\gamma) \leq \length (\exp_p \circ \gamma) \leq 2 \length(\gamma). \]
Furthermore, on $B_\M\left(p,{\frac{r}{2}}\right)$ the exponential mapping $\exp_p \colon B\left(0,{\frac{r}{2}}\right) \subseteq T_p \M \to B_\M \left(p,{\frac{r}{2}}\right)$ is a bi-Lipschitz bijection with bi-Lipschitz constant 2.
\end{proposition}
\begin{proof}
The first claim follows immediately from \eqref{expderest}. To deduce the second part let $q_1, q_2 \in B_\M(p,\frac{r}{2})$. Consider a smooth curve $\tilde{\gamma}\colon[0,1] \rightarrow \M$ connecting $q_1$ and $q_2$, 
i.e., $\tilde{\gamma}(0)=q_1$ and $\tilde{\gamma}(i)=q_2$. We observe that if $\tilde{\gamma}$ is not contained in $B_\M(p,r)$, then 
\[  d(q_1,q_2) \leq  d(q_1,p) + d(q_2,p) < r  \leq \length(\tilde{\gamma}). \] 
In fact, to deduce that $r \leq  \length(\tilde{\gamma})$ let $s \in (0,1) $ be such that $\tilde{\gamma}(s) \not \in B_\M(p,r)$. It is straightforward to see that the length of the restriction of $\tilde{\gamma}$ to the 
interval $[0,s]$ is larger than the distance between $\tilde{\gamma}(s)$ and $\partial B_\M(p,\frac{r}{2})$, which in turn is larger than $\frac{r}{2}$. 
Similarly the length of the restriction of $\tilde{\gamma}$ to the interval $[s,1]$ is larger than $\frac{r}{2}$. Hence $r  \leq \length(\tilde{\gamma})$ as desired.

Now, let $\tilde{\gamma}$ be a smooth curve realizing the distance between $q_1 $ and $q_2$ (which after appropriate normalization has to be a geodesic). 
From the previous observation we see that $\tilde{\gamma}$ is contained in $B_\M(p, r)$.  Consider $\gamma\coloneqq  \exp_p^{-1} \circ \tilde{\gamma}$, where we note that $\exp^{-1}_p$ is well defined along $\tilde{\gamma}$ given that $r \leq i_0$. From the first part of the proposition, 
we deduce that 
\[  \frac{1}{2} d(\exp_p^{-1}(q_1), \exp_p^{-1}(q_2)) \leq \frac{1}{2} \length (\gamma) \leq \length(\tilde{\gamma}) = d(q_1,q_2). \]
Finally, for an arbitrary smooth curve $\gamma \colon [0,1] \rightarrow B(r) \subseteq T_p\M$ with $\gamma(0)=\exp^{-1}_p(q_1)$ and $\gamma(i)= \exp^{-1}_p(q_2)$ we have
\[   d(q_1,q_2) \leq   \length (\exp_p \circ \gamma) \leq  2 \length(\gamma). \]
Taking the infimum on the right hand side over all such curves $\gamma$ we deduce that $d(q_1,q_2) \leq 2 d(\exp^{-1}_p(q_1),\exp^{-1}_p(q_2))$. This completes the proof.
\end{proof}

The bounds on metric distortion \eqref{metdist} imply that the Jacobian of the exponential map (i.e. the volume element) $J_x(v)\coloneqq  \sqrt{\det(g)}$ satisfies
\begin{equation}
(1+C m K \abs{v}^2)^{-1} \leq J_x(v) \leq (1+C m K \abs{v}^2).
\label{eqn:EstimateJacobian}
\end{equation}
A direct consequence of \eqref{eqn:EstimateJacobian} is that
\begin{align}
\label{eqn:EstimateVolumeBall2}
\frac{\omega_m r^m}{1+CmKr^2} \leq \vol(B_\M(x,r)) & \leq (1+CmKr^2)  \omega_m r^m,
 \intertext{which implies that}
\label{eqn:EstimateVolumeBall}
 \abs{ \vol(B_\M(x,r)) - \omega_m r^m }& \leq C m K r^{m+2}
\end{align}
where $\omega_m$ is the volume of the unit ball in $\R^m$.

Now we want to state a relation between the intrinsic distance on the manifold and the Euclidean distance on the ambient space. For that purpose we recall that $\rc$, the reach of the manifold $\M$, is defined as 
\begin{align} \label{defreach}
\begin{split}
 \rc\coloneqq  \sup \big\{ t>0 \: : \: & \forall x \in \R^d \text{, } \dist(x, \M) \leq t \text{, } \\
 & \exists ! y \in \M\, \text{ s.t.\ } \dist(x,\M) = \abs{x-y}  \big\}. 
 \end{split}
\end{align}
We note that $\rc$ is an extrinsic quantity, meaning it depends on the specific embedding of $\M$ into $\R^d$. In addition, we note that the quantity $\frac{1}{\rc}$ is related to extrinsic curvature, as it uniformly controls the principal curvatures of $\M$ (see \cite{NiSmWe08}).  We now show that the distances $\M$ are locally a second order perturbation of the Euclidean distance in $\R^d$ and provide explicit error bounds in terms of the reach of $\M$.
%

\begin{proposition}
\label{prop:metricestimates}
Let $\rc$ be the reach of the manifold $\M \subseteq \R^d$. Let $x,y \in \M$ and suppose that $\abs{x-y} \leq \frac{\rc}{2}$. Then,
\[ \abs{x-y} \leq d(x,y) \leq \abs{x-y}+ \frac{8}{\rc^2}\abs{x-y}^3.  \] 
\nc
\end{proposition}
\begin{proof}
The inequality $\abs{x-y} \leq d(x,y)$ is trivial. 
To show the other inequality we note that since $\abs{x-y} \leq \frac{\rc}{2}$, it follows from \cite[Prop 6.3]{NiSmWe08} that  
\[  d(x,y) \leq \rc - \rc \sqrt{1- \frac{2\abs{x-y}}{\rc}}.  \]
Using the fact that for every $t \in[0,1]$, $\sqrt{1-t}\geq 1- \frac12 t - \frac12 t^2 $ 
\begin{equation}  \label{temp:dest}
  d(x,y) \leq \rc - \rc \left( 1- \frac{\abs{x-y}}{\rc} - \frac{2}{\rc^2} \abs{x-y}^2 \right)  = \abs{x-y}+ \frac{2}{\rc}\abs{x-y}^2 \leq 2 |x-y|.
\end{equation}
To improve the error estimate let $L=d(x,y)$ and let $\gamma:[0,L] \to \M$ be an arc-length-parameterized length-minimizing geodesic between $x$ and $y$. Heuristically, $\gamma$ is a ``straight" line in $\M$ and thus its curvature in $\R^d$ is bounded by the maximal principal curvature of $\M$ in $\R^d$, which is bounded by $\frac{1}{\rc}$. More precisely we claim that 
\begin{equation} \label{curvgam}
 \abs{\ddot \gamma(t)} \leq \frac{1}{\rc^2} \qquad \textrm{for all } t \in [0, L].
\end{equation}
This statement follows from \cite[Prop 6.1]{NiSmWe08} 
(and is used in the proof of Proposition 6.3 of \cite{NiSmWe08}).
Using translation we can assume that $x=0$. Furthermore note that
that $\dot \gamma(t) \cdot \ddot \gamma(t)=0$ for all $t$. Thus
\begin{align} \label{eq:124b}
\begin{split}
\abs{x-y} = \abs{\gamma(L)} & \geq \gamma(L) \cdot \dot \gamma(L)  = \int_0^L \dot \gamma(s) \cdot \dot \gamma(L) ds \\
& = \int_0^L \left( \dot \gamma(L) - \int_s^L \ddot \gamma(r) dr \right) \cdot \dot \gamma(L)  \,ds \\
& = L - \int_0^L \int_s^L \int_r^L \ddot \gamma(r) \cdot \ddot \gamma(z) dz dr ds 
 \geq L - \frac{L^3}{\rc^2}
\end{split}
\end{align}
Combining with \eqref{temp:dest} implies $L \leq |x-y| + \frac{8}{\rc^2} |x-y|^3$.
\end{proof}

\section{The $\infty$-transportation distance}
\label{sec:transport}

The main goal of this section is to prove Theorem \ref{thm:transport}. For that purpose, we use a similar proof scheme to the one used in \cite{GTS15a}. 
We first establish Lemma \ref{lemma:W8L8} and then we construct a ``nice'' partition of the manifold $\M$ by using a Voronoi tessellation using some (fixed) appropriately chosen points; 
what makes the partition nice is that each of its cells is bi-Lipschitz homeomorphic (with universal bi-Lipschitz constant) to a fixed ball in $\R^m$ where we can apply the results from \cite{GTS15a}. 
In Section \ref{sec:NicePartition} we present the construction of such partition and prove Theorem \ref{thm:transport}.

Throughout this section, we make use of the following construction and estimates. 
Let $r = \frac15 \min\{1, i_0, \frac{1}{\sqrt K}  \}.$
Let $Y = \{y_i : i \in I\}$ be a maximal subset of $\M$ such that $d(y_i, y_j) \geq r$ for all $i \neq j$. 
Note that the balls $\left\{ B_\M(y_i,r/2) \right\}_{i \in I}$ do not overlap. From \eqref{eqn:EstimateVolumeBall2}, we conclude that $\nt \coloneqq \card Y$ satisfies 
\[ \nt (1+ C mK r^2 )^{-1}  \frac{\omega_m r^m}{2^m} \leq \sum_{i\in I} \vol\left(B_\M(y_i,r/2)\right) \leq \vol(\M)  \]
and hence
\begin{equation}
\nt  \leq (1+ C m K r^2) \frac{2^m\vol(\M)}{\omega_m r^m}.
 \label{ineq:numberballs}
\end{equation}
From now on we list the elements of $Y$ as $y_1, \dots, y_\nt $.
It follows from the maximality of $Y$ that the collection of balls $\left\{ B_\M(y_i,r) \right\}_{i=1,\dots,\nt }$ covers $\M$. We also claim that if $\dist(y_i, y_j) \leq 2r$, 
then the balls $B_\M(y_i,2r)$ and $B_\M(y_j,2r)$ have a ``big'' overlap in the sense that
\begin{equation}
 (1+ C mK r^2)^{-1} \omega_m r^m  \leq \vol( B_\M(y_i,2r) \cap B_\M(y_j,2r)).   
\label{ineq:sizeoverlap}
\end{equation}
In fact, let $y_{ij}$ be the point that is halfway from $y_i$ to $y_j$ on the geodesic connecting $y_i$ and $y_j$. Let $y \in B_\M(y_{ij},r)$. 
Then $\dist(y,y_i ) \leq \dist(y,y_{ij})+ \dist(y_{ij}, y_i) < r + r \leq 2r$. This shows that $B_\M(y_{ij}, r) \subseteq B_\M(y_i,2r)$. Similarly, we have $B_\M(y_{ij}, r) \subseteq B_\M(y_j,2r)$. Inequality \eqref{ineq:sizeoverlap} now follows from the fact that
\[   (1+ C mK r^2)^{-1}  \omega_m r^m      \leq \vol(B_\M(y_{ij}, r))\]

We now claim that for arbitrary $y_i,y_j$, there is a way to start from $y_i$ and move from ball to ball until reaching $y_j$ in such a way that any two consecutive balls visited have big overlap, i.e.\ that \eqref{ineq:sizeoverlap} holds. 
To make this idea precise, let us consider a graph $(Y, \leftrightarrow)$ where 
\begin{equation}
y_j \leftrightarrow y_i \,\text{ iff }\, y_j \not = y_i \text{ and } \dist(y_j, y_i ) \leq 2r.
\label{def:GraphNicePartition}
\end{equation}

 We claim that $(Y, \leftrightarrow)$ is a connected graph; this is a consequence of the connectedness of $\M$. In fact, suppose for the sake of contradiction that  $(Y, \leftrightarrow)$ is not connected. 
 Then, we can find a partition of $Y$ into two nonempty sets $S_1,S_2$ such that for all $y_i \in S_1$ and all $y_j \in S_2$, $y_i \not \leftrightarrow y_j$ (i.e. $d(y_i, y_j) > 2r$). Because of this, we can find $\veps>0$ such that 
\[  \bigcup_{y \in S_1} B_\M\left(y, r+ \veps\right) \cap  \bigcup_{y\in S_2} B_\M\left(y,r +\veps\right) = \emptyset, \]
but since
\[   \M = \bigcup_{i=1,\dots,\nt } B_\M\left(y_i,r\right)\subseteq  \bigcup_{y \in S_1} B_\M\left(y, r+ \veps\right) \cup  \bigcup_{y\in S_2} B_\M\left(y,r +\veps\right),  \]
this implies that $\M$ is disconnected, which is not true. Hence, we conclude that the graph $(Y, \leftrightarrow)$ is connected. We are now ready to prove Lemma \ref{lemma:W8L8}.

\begin{proof}[Lemma \ref{lemma:W8L8}]
In order to estimate $d_\infty(\rho_1,\rho_2)$, the idea is to construct intermediate densities and estimate the distances between them using \cite[Theorem 1.2]{GTS15a}. But to use \cite[Theorem 1.2]{GTS15a} we need to map the intermediate densities
to the the Euclidean space. Motivated by this, we consider the balls $B_\M(y_1,2r),\dots, B_\M(y_\nt ,2r)$ constructed before. By relabelling if necessary, the connectedness of the graph $(Y,\leftrightarrow)$ implies that we can assume that for every $k=1,\dots,\nt $, the graph 
 $(\left\{y_1,\dots,y_k \right\},\sim)$ is connected. For $k=1,\dots,\nt $, we define the sets 
\[  I_k:= B_\M(y_k,2r) \setminus \bigcup_{j=1}^{k-1}B_\M(y_j,2r), \quad  O_k:= B_\M(y_k,2r) \cap \bigcup_{j=1}^{k-1}B_\M(y_j,2r) .\]
Note that $I_1 = B(y_1,2r)$ and $O_1= \emptyset$. We define the functions $\gamma_k^+, \gamma_k^-, \tilde{\rho}_k   $ iteratively as follows. Let us start with $k=\nt $. If $\int_{I_\nt } \rho_1 dx \geq \int_{I_\nt } \rho_2 dx$
we set $\gamma_\nt ^+=\rho_1$ and $\gamma_\nt ^-= \rho_2$; if not, we reverse the roles of $\rho_1$ and $\rho_2$. We let $\tilde{\rho}_\nt $ be 

\[  \tilde\rho_\nt (x) = \begin{cases} 
      \gamma_{\nt }^-(x) & \text{if }x\in I_\nt , \\
      \gamma_{\nt }^+(x) + \beta_\nt  & \text{if }x \in  O_\nt , \\
      \gamma_\nt ^+(x) & \text{otherwise}, 
   \end{cases}  \]
where
\[  \beta_\nt := \frac{ \int_{I_\nt }(\gamma_\nt ^+ - \gamma_\nt ^-  ) dx }{\vol(O_\nt )} . \]
Having defined the functions $\gamma^+, \gamma^-, \tilde{\rho}$ for 
the iterations $\nt ,\nt -1,\dots,k+1$, we define the functions $\gamma_k^{+},\gamma_k^-, \tilde{\rho}_k$ as follows. 
If $\int_{I_k} \gamma_{k+1}^- dx \geq \int_{I_k} \tilde{\rho}_{k+1} dx$
we set $\gamma_k^+=\gamma_{k+1}^-$ and $\gamma_k^-= \tilde{\rho}_{k+1}$; if not, we reverse the roles of $\gamma_{k+1}^-$ and $\tilde{\rho}_{k+1}$. The function $\tilde{\rho}_k$ is defined as 

\[  \tilde\rho_k(x) = \begin{cases} 
      \gamma_{k}^-(x) & \text{if }x\in I_k, \\
      \gamma_{k}^+(x) + \beta_k & \text{if }x \in  O_k, \\
      \gamma_k^+(x) & \text{otherwise}, 
   \end{cases}  \]
where
\[  \beta_k:= \frac{ \int_{I_k}(\gamma_k^+ - \gamma_k^-  ) dx }{\vol(O_k)}.  \]
We note that $\tilde{\rho}_1= \gamma_1^-$ and set $\beta_1:=0$. Also, observe that for every $k$, $\beta_k \geq 0$ and
\[ \int_{\M} \gamma_k^- dx= \int_{\M} \gamma_k^+ dx= \int_{\M} \tilde{\rho}_k dx, \]
where the second equality follows from the definition of $\beta_k$ and where the first equality follows iteratively from the definitions above.

Using  the triangle inequality and the above definitions we obtain
\begin{align*}
d_\infty(\rho_1  , \rho_2 ) &= d_\infty(\gamma_\nt ^+ , \gamma_\nt ^- )
\\ & \leq d_\infty( \gamma_\nt ^+ , \tilde{\rho}_\nt  ) +  d_\infty( \tilde{\rho}_\nt  ,  \gamma_\nt ^-  ) 
\\&  = d_\infty( \gamma_\nt ^+ , \tilde{\rho}_\nt  ) + d_\infty( \gamma_{\nt -1}^+, \gamma_{\nt -1}^-  )
\\& \leq d_\infty( \gamma_\nt ^+ , \tilde{\rho}_\nt  ) + d_\infty( \gamma_{\nt -1}^+, \tilde{\rho}_{\nt -1}  ) + d_\infty(\tilde{\rho}_{\nt -1} , \gamma_{\nt -1}^- ).
\end{align*}
Continuing the chain of inequalities provides, by induction,
\[ d_\infty(\rho_1, \rho_2) \leq \sum_{k=1}^{\nt } d_\infty( \gamma_k^+ , \tilde{\rho}_k ).\] 
Our goal is to estimate each of the terms $d_\infty(\gamma_k^+ , \tilde{\rho}_k)$. From the definitions above, it is straightforward to see that $\gamma_k^+$ and $\tilde{\rho}_k$ coincide in $\M \setminus B_\M(y_k,2r)$ and thus
\begin{equation}
d_\infty(\rho_1,\rho_2) \leq \sum_{k=1}^{\nt } d_\infty( \gamma_k^+ , \tilde{\rho}_k ) \leq \sum_{k=1}^{\nt } d_\infty( \gamma_k^+ \vert_{B_\M(y_k,2r)}, \tilde{\rho}_k \vert_{B_\M(y_k,2r)}   ).  
\label{estimated8Partition}
\end{equation}
The last inequality is a consequence of the following observation: if two measures $\nu_1,\nu_2$ give the same total mass and we can write $\nu_1= \nu + \tilde{\nu}_1$ and $\nu_2= \nu + \tilde \nu_2$, 
then one possible way to transport mass from $\nu_1$ into $\nu_2$ is to leave the mass distributed as $\nu$ where it is and simply focus on transporting the mass distributed as $\tilde{\nu}_1$ to have it distributed as $\tilde{\nu}_2$. This observation leads to the desired inequality.  

In order to obtain an estimate on $ d_\infty( \gamma_k^+ \vert_{B_\M(y_k,2r)}  , \tilde{\rho}_k \vert_{B_\M(y_k,2r)} ) $, we first estimate $\lVert \gamma_k^+ - \tilde{\rho}_k  \rVert_{L^\infty(B_\M(y_k,2r))}$. From the definitions above we have
\begin{equation}
 \lVert \gamma_k^+ - \tilde{\rho}_k  \rVert_{L^\infty(B(y_k,2r)} \leq \max \left\{  \lVert \gamma_k^+ - \gamma_k^-  \rVert_{L^\infty(I_k)}    ,\beta_k \right\}. 
\label{Linftyestimate2}
\end{equation}
Hence, we focus on obtaining estimates for $\lVert \gamma_k^+ - \gamma_k^-  \rVert_{L^\infty(I_k)} $ and $\beta_k$.

First, we claim that for every $k$, the function $(\gamma_k^+ - \gamma_k^-) \mathds{1}_{I_k}$ has the form 
\begin{equation}
 (\gamma_k^+ - \gamma_k^-) \mathds{1}_{I_k} = \pm (\rho_1 - \rho_2) \mathds{1}_{I_k} + \sum_{j=k+1}^{\nt } \pm \beta_j \mathds{1}_{I_k \cap O_j}. 
\label{gamma+gamma-}
\end{equation}
To see this, note that in case $k=\nt $ the result is trivial. In general, from the definitions above it follows that
\begin{align*}
\begin{split}
(\gamma_k^+ - \gamma_k^-) \mathds{1}_{I_k} &= \pm ( \gamma_{k+1}^- - \tilde{\rho}_{k+1} ) \mathds{1}_{I_k} 
\\ &= \pm ( (\gamma_{k+1}^-  -\gamma_{k+1}^+ ) \mathds{1}_{I_k }  -  \beta_{k+1} \mathds{1}_{I_k \cap O_{k+1}} )    
\\&= \pm (\gamma_{k+1}^-  -\gamma_{k+1}^+ ) \mathds{1}_{I_k } + \pm  \beta_{k+1} \mathds{1}_{I_k \cap O_{k+1}}
\\&= \pm (\gamma_{k+2}^-  -\tilde{\rho}_{k+2} ) \mathds{1}_{I_k  }  + \pm \beta_{k+1} \mathds{1}_{I_k \cap O_{k+1}}
\\&= \pm ( \gamma_{k+2}^-  -  \gamma_{k+2}^+ )\mathds{1}_{I_k } + \pm \beta_{k+2}  \mathds{1}_{I_k \cap O_{k+2}}  + \beta_{k+1}\mathds{1}_{I_k \cap O_{k+1}}.
\end{split}
\end{align*}
Continuing the chain of inequalities proves the claim in $\nt -k$ iterations. An immediate consequence of the previous fact is that for $k=2 , \dots, \nt $
\[  \beta_k = \pm \frac{\int_{I_k} (\rho_1-\rho_2)dx }{\vol(O_k)} + \sum_{j=k+1}^{\nt } \pm \beta_j \frac{\vol(I_k \cap O_j)}{\vol(O_k)},  \]
and in particular
\begin{equation}
 \beta_k \leq  \lVert \rho_1 - \rho_2 \rVert_{L^\infty(\M)} \frac{\vol(I_k)}{\vol(O_k)} + \sum_{j:k < j \leq \nt } \beta_j \frac{\vol(I_k \cap O_j )}{\vol(O_k)}, \quad \forall k =2,\dots,\nt . 
\label{Boundsbetak} 
\end{equation}
For every $k=2,\dots,\nt $ we claim that
\begin{equation}
\beta_k \leq \lVert \rho_1-\rho_2 \rVert_{\infty} \left(   \sum a_{j_1\dots j_s}  \right), 
\label{relationajs}
\end{equation}
where the sum is taken over all $s \leq N_c-k$ and all $s$-tuples $\nt \geq j_1 > j_2 > \dots> j_{s-1}> j_s=k $, and where 
\[  a_{j_1\dots j_s} := \frac{\vol(I_{j_1})}{\vol(O_{j_1})} \cdot \frac{\vol(I_{j_2} \cap O_{j_1})}{\vol(O_{j_2})} \dots  \frac{\vol(I_{j_{s-1} \cap O_{j_{s-2}}})}{\vol(O_{j_{s-1}})}\cdot \frac{\vol(I_{j_s} \cap O_{j_{s-1}})}{\vol(O_{j_s})}.  \]
In fact, relation \eqref{relationajs} is obtained inductively by using recursion \eqref{Boundsbetak} and the fact that $\beta_\nt  \leq \lVert \rho_1 - \rho_2 \rVert_{L^\infty(\M)} \frac{\vol(I_\nt )}{\vol(O_\nt )} $. Let us now fix $s$ with $0 \leq s \leq \nt -k $ and $k'$ with $k +s \leq k' \leq \nt $; set $j_1=k' $ and $j_s=k$. Let us write $a_{j_1\dots j_s}$ in the more convenient way:
\[  a_{j_1\dots j_s} = \frac{\vol(I_{j_1})}{\vol(O_{j_s})}  \cdot \frac{\vol(I_{j_2} \cap O_{j_1})}{\vol(O_{j_1})} \dots  \frac{\vol(I_{j_{s-1} \cap O_{j_{s-2}}})}{\vol(O_{j_{s-2}})}\cdot \frac{\vol(I_{j_s} \cap O_{j_{s-1}})}{\vol(O_{j_{s-1}})}.  \]
Note that
\[  a_{j_1\dots j_s} \leq \frac{\vol(I_{j_1})}{\vol(O_{j_s})}  \cdot \frac{\vol(I_{j_2} \cap O_{j_1})}{\vol(O_{j_1})} \dots  \frac{\vol(I_{j_{s-1} \cap O_{j_{s-2}}})}{\vol(O_{j_{s-2}})},  \]
and therefore summing over $j_{s-1}$ we obtain
\begin{align*}
 \sum_{ j_{s-1}} a_{j_1\dots j_s} \leq  & \frac{\vol(I_{j_1})}{\vol(O_{j_s})}  \frac{\vol(I_{j_2} \cap O_{j_1})}{\vol(O_{j_1})} \dots \frac{\vol(I_{j_{s-2}} \cap O_{j_{s-3}})}{\vol(O_{j_{s-3}})} \\
 &  \sum_{j_{s-1}} \frac{\vol(I_{j_{s-1}} \cap O_{j_{s-2}})}{\vol(O_{j_{s-2}})}. 
 \end{align*}
Observe that the sum on the right hand side of the above expression is less than one because the sets $I_{j_{s-1}}$ are disjoint. Proceeding in this fashion adding over $j_{s-2}, \dots, j_{2}$ we conclude that
\[ \sum_{ j_2  \dots j_{s-1} }a_{j_1\dots j_s}  \leq \frac{\vol(I_{k'})}{\vol(O_k)}. \]

Finally, first summing over all such $s$ and then over all such $k'$, it follows from \eqref{relationajs} that 
\begin{align}
\label{estimatebetak}
\begin{split}
 \beta_k \leq  & \lVert \rho_1-\rho_2 \rVert_{\infty} \sum a_{j_1,\dots,j_s} \leq \lVert \rho_1-\rho_2 \rVert_{\infty} \sum_{k<k' \leq \nt  } \frac{\nt \vol(I_{k'})}{\vol(O_k)} \\ 
 \leq & \lVert \rho_1-\rho_2 \rVert_{\infty}  \frac{\nt \vol(\M)}{\vol(O_k)}
\end{split}
\end{align}
where in the last inequality we have used the fact that the sets $I_{k'}$ are disjoint.

Going back to \eqref{Linftyestimate2}, we note that from \eqref{gamma+gamma-} and \eqref{estimatebetak} it follows that for every $k=1,\dots,\nt $
\begin{align}
\begin{split}
 \lVert \gamma_k^+ - \tilde{\rho}_k & \rVert_{L^\infty(B_\M(y_k,2r))}  \leq \lVert \rho_1 - \rho_2  \rVert_{L^\infty(\M)} + \sum_{j=k}^{\nt }\beta_j 
\\& \leq \lVert \rho_1 - \rho_2\rVert_{L^\infty(\M)} \left( 1 +  \nt ^2 \vol(M) \max_{j=2,\dots, \nt } \frac{1}{\vol(O_j)} \right) 
\\& \leq  \lVert \rho_1 - \rho_2\rVert_{L^\infty(\M)} \left( 1 +   \frac{C_m \nt ^2 \vol(M)}{r^m} \right), 
\end{split}
\label{estimateL8gamma+}
\end{align}
where the last inequality follows from the lower bound on the size of the overlaps \eqref{ineq:sizeoverlap}. 

Now we notice that from the standing assumption  $\rho_1(x) , \rho_2(x) \geq \frac{1}{\alpha}$ for every $x \in \M$, it follows that for every $k=1,\dots,\nt $ and every $x \in \M$
\[  \gamma_k^+(x) , \gamma_k^-(x), \tilde{\rho}_k(x) \geq \frac{1}{\alpha}  \quad \text{for all } x \in \M . \]
Likewise, from the standing assumption $\rho_1(x) , \rho_2(x) \leq \alpha$ for all $x \in \M$, it follows that for every $k=1,\dots,\nt $ and every $x \in \M$
\[  \gamma_k^+(x) , \gamma_k^{-}(x) , \tilde{\rho}_k(x) \leq \alpha + \sum_{j=1}^{\nt }\beta_j  \leq \alpha  + \lVert \rho_1 - \rho_2 \rVert_{L^\infty(\M)} \frac{C_m \nt ^2 \vol(\M)}{r^m}      \]

Assume for a moment that $\lVert \rho_1 -\rho_2\rVert_{L^\infty(\M)}$  is small enough so that in particular $\lVert \rho_1 - \rho_2 \rVert_{L^\infty(\M)} \frac{C_m \nt ^2 \vol(\M)}{r^m}  \leq \alpha$.
In that case, for every $k=1,\dots, \nt $ we would have  
\begin{equation}
\frac{1}{\alpha} \leq \gamma_k^+, \tilde{\rho}_k \leq 2 \alpha.
\label{boundsalpha}
\end{equation}

Consider the exponential map $\exp_{y_k} \colon B( 2r) \subseteq T_{y_k}\M \rightarrow B_\M(y_k, 2r) \subseteq \M$ and the functions $g_1, g_2\colon B(2r) \rightarrow (0,\infty)$ 
defined as
\[ g_1(v) \coloneqq \gamma_k^+(\exp_{y_k}(v))J_{y_k}(v) \]
and
\[ g_2(v)\coloneqq  \tilde{\rho}_k(\exp_{y_k}(v))  J_{y_k}(v),\]
where $J_{y_k}$ denotes the Jacobian of the exponential map.  From \eqref{boundsalpha}, \eqref{estimateL8gamma+} and \eqref{eqn:EstimateJacobian} we conclude that
\begin{equation}
\frac{1}{\alpha Cm(1+ Kr^2)}  \leq g_i(v) \leq \alpha Cm(1+Kr^2) \quad \te{for } i=1,2 \te{ and all } v \in B(2r)
\label{ineq:auxbounds}
\end{equation}
and that for all $v \in B(2r) $
\begin{align}
\begin{split}
\abs{g_1(v) - g_2(v)} & \leq   (1+Cm K r^2) \abs{\gamma_k^+(\exp_{y_\nt }(v)) - {\tilde{\rho}_k}(\exp_{y_\nt }(v))}
\\  & \leq \frac{C_m \nt ^2 \vol(\M)}{r^m} \lVert \rho_1-\rho_2 \rVert_{L^\infty(\M)}
\end{split}
\label{ineq:auxdistance}
\end{align}
We recall that our choice of $r$ in particular gurantees that $r^2K \leq 1$. Applying \cite[Theorem 1.2]{GTS15a} to the densities $g_1$ and $g_2$ with the bounds given by \eqref{ineq:auxbounds} we conclude that
\[ d_\infty(g_1 , g_2 ) \leq C_{m,\alpha} r \lVert g_1- g_2 \rVert_{L^\infty(B(2r))} \leq   \frac{C_{m,\alpha} \nt ^2 \vol(\M)}{r^{m-1}}  \lVert \rho_1-\rho_2 \rVert_{L^\infty(\M)}, \footnote{Note that as stated, our theorems give $C_{m,\alpha,r }$ , but in this case $C_{m,\alpha,r }= C_{m,\alpha } r$ because we can always rescale to the unit ball.  }  \]
where the last inequality follows from \eqref{ineq:auxdistance}. From the second part of Proposition \ref{prop:expbilip}, it follows that
\[ d_\infty(  \gamma_k^+  , \tilde{\rho}_k ) \leq 2  d_\infty(g_1, g_2   ) \leq   \frac{C_{m,\alpha} \nt ^2 \vol(\M)}{r^{m-1}}  \lVert\rho_1-\rho_2\rVert_{L^\infty(\M)}. \]
Therefore, using \eqref{estimated8Partition} it follows that if
\[ \lVert \rho_1 - \rho_2 \rVert_{L^\infty(\M)} \frac{C_m \nt ^2 \vol(\M)}{r^m}  \leq \alpha, \] 
then
\[  d_\infty(\rho_1, \rho_2) \leq    \frac{C_{m,\alpha} \nt ^3 \vol(\M)}{r^{m-1}}  \lVert\rho_1-\rho_2\rVert_{L^\infty(\M)}. \]
In case $\lVert \rho_1 - \rho_2 \rVert_{L^\infty(\M)} \frac{C_m \nt ^2 \vol(\M)}{r^m} >  \alpha \geq 1$, we have
\begin{align*}
 d_\infty(\rho_1,\rho_2) &\leq \diam(\M) 
\\& \leq   \frac{C_m \nt ^2\vol(\M) \diam(\M)}{r^m} \lVert \rho_1 - \rho_2 \rVert_{L^\infty(\M) } ,
\end{align*}
where we note that the first inequality in the above expression is always true, as the maximum distance any point can travel in $\M$ is $\diam(\M)$. Therefore, in any case we have 
\begin{equation*}
d_\infty(\rho_1, \rho_2) \leq \tilde{C}\lVert  \rho_1 - \rho_2 \rVert_{L^\infty(\M)},
\end{equation*}  
where $\tilde{C}$ can be written as
\begin{equation}
 \tilde{C}= \frac{C_{m,\alpha} \nt ^2 \vol(\M)}{r^{m-1}} \max \left\{  \nt  , \frac{\diam (\M)}{r } \right\}.
 \label{tildeC}
\end{equation}
\end{proof}
\nc

\subsection{Proof of Theorem \ref{thm:transport}}
\label{sec:NicePartition}

In the following, we consider the Voronoi tessellation induced by the set $Y=\{y_1,\dots,y_\nt \}$ constructed in the beginning of Section \ref{sec:transport}, i.e.\ for each $i\in\{1,\dots,\nt \}$ we define
\[ V_\M(y_i) \coloneqq \{ x\in \M : d(x,y_i) \leq d(x,y_j) \text{ for all } j\in\{1,\dots,\nt \}\} .\]
These measurable sets form a partition of $\M$ up to a negligible set of ambiguity of measure zero.
We make use of the following.
\begin{proposition}
For each $i\in\{1,\dots,\nt \}$ 
there exists a bi-Lipschitz bijection $\Psi_i \colon  V_\M(y_i) \to \overline{B\left(0, \frac{r}{2} \right)} \subseteq \R^m$ with bi-Lipschitz constant at most $18$. 
\label{cor:VoronoiBiLipschitz}
\end{proposition}
To prove Proposition \ref{cor:VoronoiBiLipschitz} we use the sequence of lemmas that follow.

\begin{lemma} \label{lemma:nrangle}
For all $i\in \{1,\dots, \nt\}$
\[ B_\M (y_i,r/2)  \subset  V_\M(y_i) \subset B_\M(y_i, r). \]
 Let $V(y_i) = \exp_{y_i}^{-1}(V_\M(y_i) ).$
 Then $B\left(0, \frac{r}{2} \right) \subset  V(y_i) \subset B( r)$ and for almost every $z_0 \in \partial V(y_i)$
 \[ \frac{z_0}{\abs{z_0}} \cdot n_0 \geq \frac{1}{8}, \]
 where $n_0$ is the outward unit normal vector to  $\partial V(y_i)$ at $z_0$. 
\end{lemma}
\begin{proof}
Let $y_i \in Y$. Since for every $x \in  B_\M \left(y_i, \frac{r}{2} \right)$ and every $y_j \in Y$ with $j \neq i$ it holds that $d(x,y_j) \geq d(y_j, y_i) - d(y_i,x) > \frac{r}{2}$ we conclude that $ B_\M \left(y_i, \frac{r}{2} \right)  \subset V_\M(y_i)$. 
On the other hand, since $Y$ is a maximal set with the property that $d(y_j, y_k) \geq r$ for all $j \neq k$, we conclude that for all $x \in \M$ there exists $y_j \in X$ such that $d(x,y_j) < r$. 
Therefore $V_\M(y_i) \subset B_\M(y_i, r)$. Since $\exp_{y_i}$ maps $B(s)$ bijectively to  $B_\M(y_i,s)$ for $s=\frac{r}{2}$ and for $s=r$, it follows that $ B \left(0, \frac{r}{2} \right) \subset  V(y_i) \subset B( r)$. This establishes the first part of the statement.

Now let us consider the second part of the statement. For almost every $z_0 \in \partial V(y_i)$ there exists a unique $y_j \neq y_i$ such that $z_0 \in \partial \exp^{-1}_{y_i}(V_\M(y_j))$; let us fix one such $z_0$. Note that $ 2r \geq d(y_i,y_j) \geq r$ and that $d(y_i,z)=\abs{z_0}_{y_i} < r$. 
We let $z:= \exp_{y_i}(z_0)$. We consider the level set $\Gamma \coloneqq \{ x\in \M : d(x,y_i) = d(x,y_j) \}$, which is a $C^1$-hypersurface around $z$ by the implicit function theorem; moreover a unit normal vector
to $\Gamma$ at the point $z$ is given by 
\[ n:= \frac{\tilde{u}_i - \tilde{u}_j  }{\abs{\tilde{u}_i- \tilde{u}_j}_{z}}  =\frac{u_i - u_j}{\abs{u_i-u_j}_z } ,  \]
where $\tilde u_i :=- \frac{\exp^{-1}_{z}({y_i})}{ d(y_i,z) }  $ , $u_i :=- \exp^{-1}_{z}({y_i})$ and $u_j$, $\tilde{u}_j$ are defined analogously.

Let us consider the set $\Gamma_0 := \exp_{y_i}^{-1} \left( \Gamma \cap B_\M (y_i,2r) \right) $; note that around the point $z_0$, $\Gamma_0$ coincides with $\partial V(y_i)$, and in particular
given that $\Gamma$ is a $C^1$-hypersurface around $z$, $\partial V(y_i)$ is a $C^1$-hypersurface around $z_0$. Let us denote by $n_0$ the outward unit normal to $\partial V(y_i)$ at $z_0$. We write $\frac{z_0}{\abs{z_0}_{y_i}}$ as  
\[  \frac{z_0}{\abs{z_0}_{y_i}} = w_0 + c n_0 ,\]
where $\langle w_0 ,  \frac{z_0}{\abs{z_0}_{y_i}}   \rangle_{y_i} =0 $ and $\langle n_0 ,  \frac{z_0}{\abs{z_0}_{y_i}}   \rangle_{y_i}=c $. Clearly $c\geq 0$. Now, by definition of the exponential map,
$   \tilde{u}_i=(d\exp_{y_i})_{z_0} \left( \frac{z_0}{\abs{z_0}_{y_i}} \right) , $ and so
\[ \tilde{u}_i = w + c \tilde{n},   \]
where $w:= (d\exp_{y_i})_{z_0}(w_0)$ and $\tilde{n}:= (d\exp_{y_i})_{z_0}(n_0)$. Then,
\[ \applied{\tilde u_i}{n}_z = \applied{w+c\tilde{n}}{n}_z = c\applied{\tilde{n}}{n}_z \leq c\abs{\tilde{n}}_z \leq 2c\abs{n_0}_{y_i} =2c,\]
where the second inequality follows from the fact that $w$ is tangent to $\Gamma$ (which in turn follows from the fact that $w_0$ is tangent to $\Gamma_0$) and where the last inequality follows from \eqref{expderest}. 
It thus remains to show that $\applied{\tilde u_i}{n}_z \geq 1/4$. To see this, simply note that the fact that $\applied{\tilde u_i + \tilde u_j}{\tilde u_i - \tilde u_j}_z = 0$ implies
	\[ \applied{\tilde u_i}{n}_z = \biggl\langle \frac{\tilde u_i - \tilde u_j}{2}, \frac{\tilde u_i - \tilde u_j}{\abs{\tilde u_i - \tilde u_j}_z}\biggr\rangle_z = \frac{\abs{\tilde u_i-\tilde u_j}_z}{2} =  \frac{\abs{ u_i- u_j}_z }{2 d(z,y_i)}  \geq \frac{d(y_i,y_j)}{4 d(z,y_i)} \geq \frac{1}{4},\]
where the second to last inequality follows from Proposition \ref{prop:expbilip}.  
\end{proof}
\nc

So far we have been able to construct a partition of $\M$ into cells (the Voronoi cells $V_\M(y_i)$) with the property that when each of the cells $V_\M(y_i)$ is mapped by the inverse of the exponential map, 
the resulting set $V_i$ (which is contained in $\R^m$) is a star shaped domain with center the origin. In the next lemma we show that when the unit normal to the boundary of a star shaped domain does not deviate too much from the radial direction emanating from its center, 
the domain is bi-Lipschitz homeomorphic to a ball and the bi-Lipschitz constant can be controlled. 
This establishes Proposition \ref{cor:VoronoiBiLipschitz}.

\begin{lemma} \label{lemma:starbilip}
Let $V$ be a star-shaped subset of $\R^m$ with center at $0$ and such that  $B(R) \subset V \subset B(2R)$. Assume $V$ has Lipschitz boundary and let $n$ be the unit 
outside normal vector to $\partial V$. Assume there exists $\beta \in (0,1)$ such that for a.e.\ $x \in \partial V$
\[ n \cdot \frac{x}{\abs{x}} \geq \beta. \]
Let $r \colon S^{m-1} \to [R,2R]$ be the function describing $\partial V$ in radial coordinates. That is let 
$r(z) = \sup\{ s \in \R \::\: sz \in V\}$. 
Consider the function $\Phi \colon V \to \overline{B(R)}$ given by  
\[ \Phi(x) = \frac{R}{r\big( \frac{x}{\abs{x}}\big)} \, x  \quad  \text{for }x \neq 0  \]
and $\Phi(0) = 0$. Then $\Phi$ is a bi-Lipschitz bijection with bi-Lipschitz constant at most $\frac{1}{\beta} + 1$. 
\end{lemma}
\begin{proof}
Extend $r$ to $\R^m \backslash \{0\}$ by $\tilde r(x)\coloneqq  r\big( \frac{x}{\abs{x}} \big)$. 
For $x \neq 0$
\begin{equation} \label{dpsi}
 D \Phi(x) = - \frac{R}{\tilde r^2(x)} \, x  (\nabla \tilde r(x))^T + \frac{R}{\tilde r(x)} I. 
\end{equation}
Consider the  function $G\colon \R^m \backslash \{0\} \to \partial V$ given by $x \mapsto  \tilde r(x)  \frac{x}{\abs{x}}$. 
Note that at $z \in S^{m-1}$
\[ DG(z) = z (\nabla \tilde r(z))^T + r(z) \left( I - z z^T \right).  \]
Since $n$ is orthogonal to the image of $G$, we conclude that  $(DG(z))^T n   = 0$, which implies
\[ (n \cdot z) \nabla \tilde r(z) + r(z)( n - (n\cdot z)  z)  = 0. \]
Since $n \cdot z \geq \beta$ we obtain 
\[  \beta \abs{\nabla \tilde r(z)} \leq \tilde{r}(z) \quad \text{for all } z \in S^{m-1}. \]
Combining this with \eqref{dpsi}, we deduce that $\Phi$ is $(\frac{1}{\beta} +1)$-Lipschitz. 
Analogous computations show that $\Phi^{-1}$, which is given by 
$\Phi^{-1}(y) = r \big(\frac{y}{\abs{y}} \big) y$, is also $(\frac{1}{\beta}+1)$-Lipschitz.
\end{proof}

\begin{proof}[Proposition \ref{cor:VoronoiBiLipschitz}]  
By Proposition \ref{prop:expbilip}  the exponential map $\exp_{y_i} \colon B(r) \to B_\M(y_i,r)$ is a bi-Lipschitz bijection with bi-Lipschitz constant at most $2$. 
By Lemmas \ref{lemma:nrangle} and \ref{lemma:starbilip}, with $R = \frac{r}{2}$ and $\beta = \frac{1}{8}$, there exists a mapping 
\[ \Psi_i\colon \exp_{y_1}^{-1}(V_\M(y_i))  \to \overline{B \left(0, \frac{r}{2} \right)}\]
which is a bi-Lipschitz bijection with bi-Lipschitz constant at most 9.  
The composition $ \Psi_i \circ \exp_{y_i}^{-1}  $ provides the desired mapping.
\end{proof}

\begin{proof}[Theorem \ref{thm:transport}] 
We consider the maps $\Psi_i \colon V_\M (y_i) \rightarrow \overline{B(r/2)} \subseteq \R^m$ from Proposition \ref{cor:VoronoiBiLipschitz}. 
Given the sample $x_1,\dots,x_n$ from the density $p$, we define a density $p_n\colon \M \rightarrow \R$ by setting
\begin{equation}
  p_n(x)\coloneqq  p(x) + \frac{\mu_n(V_\M(y_i)) - \mu(V_\M(y_i)) }{\vol(V_\M(y_i))} \quad \text{for } x \in V_\M(y_i).
\end{equation}

Let us recall that Hoeffding's inequality states that for every $t>0$,
\[  \mathbb{P} \left( \lvert \mu_n(V_\M(y_i)) - \mu(V_\M(y_i)) \rvert  > t \right) \leq {2 e^{-2 nt^2}}.\]
Using the previous concentration inequality we conclude that for every $i=1,\dots, \nt$ 
\[ \lVert p- p_n   \rVert_{L^\infty(V_\M(y_i))} \leq \frac{1}{2\alpha} \]
with probability at least {$1- 2\exp\left(- n \frac{\vol(V_\M(y_i))^2}{2 \alpha^2}   \right)$}
In particular, using a union bound, we conclude that with probability at least $1 - 2\nt\exp\left(- n \frac{C_m r^{2m}}{ \alpha^2}   \right) $
\begin{equation}
\frac{1}{2\alpha} \leq p_n(x) \leq 2 \alpha , \quad x \in \M.
\label{boundspn}
\end{equation}
Similarly, with probability at least $1-  2\nt\exp\left(- n \frac{C_m r^{2m}}{ \alpha^2}   \right)$
\begin{equation}
  \frac{1}{2} \mu(V_\M(y_i)) \leq \mu_n(V_\M(y_i)) \leq \frac{3}{2}\mu(V_\M(y_i) )
\label{boundsni}
\end{equation}
Hoeffding's inequality together with an union bound also shows that with probability at least $1- 2\nt n^{-\beta}$,
\begin{equation}
\lVert p - p_n \rVert_{L^\infty(\M)} \leq \frac{C_m}{r^m} \sqrt{\frac{\beta \log(n)}{n}}.
\label{boundppn}
\end{equation}
We let $A_n$ be the event where \eqref{boundspn}, \eqref{boundsni} and \eqref{boundppn} hold. From the above we know that $A_n$ occurs with probability at least $1- Cn^{-\beta}$. 
Where the constant $C$ depends on $r,\alpha, \beta,m, \vol(\M)$. We denote by $\tilde{\mu}_n$ the measure $d \tilde{\mu}_n = p_n dx$. Conditioned on the event $A_n$, we see from Lemma \ref{lemma:W8L8} and from \eqref{boundppn} that
\[ d_\infty(\tilde{\mu}_n , \mu) \leq \tilde{C} \lVert p - p_n \rVert_{L^\infty(\M)}  \leq \tilde{C} \frac{C_m}{r^m} \sqrt{\frac{\beta \log(n)}{n}}, \]
where $\tilde{C}$ is the constant in \eqref{tildeC}.

Now we estimate $d_\infty(\tilde{\mu}_n, \mu_n)$ in the event $A_n$. Observe that 
\[  \tilde{\mu}_n(V_\M(y_i)) = \mu_n(V_\M (y_i)) \quad \text{for all } i=1,\dots,\nt  \]
and hence 
\[   d_\infty(\mu_n, \tilde{\mu}_n) \leq \max_{i=1,\dots,\nt} d_\infty( \mu_n \llcorner_{V_\M(y_i)} ,  \tilde{\mu}_n \llcorner_{V_\M(y_i)} ),  \]
where we denote by $\llcorner_{V_\M(y_i)}$ the restriction of a measure to $V_\M(y_i)$. The goal is now to estimate $d_\infty( \mu_n \llcorner_{V_\M(y_i)}, \tilde{\mu}_n \llcorner_{V_\M(y_i)} )$ for every $i$.

Let $x_{j_1}, \dots, x_{j_{n_i}}$ be the points in $X$ that fall in $V_\M(y_i)$. 
We consider the transformed points $\Psi_i(x_{j_1}), \dots, \Psi_i(x_{j_{n_i}})$ and the measure  ${\Psi_i} _{\sharp}( \tilde{\mu}_n \llcorner_{V_\M(y_i)} )$, which is supported on $\overline{B(r/2)}$.
The fact that $\Psi_i$ is bi-Lipschitz with constant $18$ implies that the measure ${\Psi_i} _{\sharp}( \tilde{\mu}_n \llcorner_{V_\M(y_i)} )$ has a density with respect to the Lebesgue measure and this density 
is lower and upper bounded by constant multiples of the lower and upper bounds of the density $p$.
Hence, the transformed points are almost surely samples from ${\Psi_i} _{\sharp}( \tilde{\mu}_n \llcorner_{V_\M(y_i)} )$ restricted to the open ball $B(r/2)$. 
Therefore, it follows from \cite[Theorem 1.1]{GTS15a} that conditioned on the event $A_n$,
\[  d_\infty \left(  {\Psi_i} _{\sharp}( \tilde{\mu}_n \llcorner_{V_\M(y_i)} ),  {\Psi_i} _{\sharp}( \mu_n \llcorner_{V_\M(y_i)} )  \right) \leq C_{m, \alpha, \beta} \, r \, \frac{\log(n_i)^{p_m}}{n_i^{1/m}}\]
holds\mh{ \footnote{Note that as stated, Theorem 1.1 in \cite{GTS15a} gives $C_{m,\alpha,\beta,r }$, but in this case $C_{m,\alpha,\beta,r }= C_{m,\alpha,\beta }\, r$ as one can simply rescale to the unit ball.}} for all $i\in \{1,\dots,\nt\}$ with probability at least $1-C\nt n^{-\beta}$, where $C$ is a constant that depends on $\beta, r ,\alpha,m$. Note that we have used the fact that in the event $A_n$, the second inequality in \eqref{boundsni} 
is satisfied and so we can give the probability bounds in terms of $n$ and not in terms of $n_i$. Moreover, from the first inequality in \eqref{boundsni} it follows that 
\[  \frac{\log(n_i)^{p_m}}{n_i^{1/m}} \leq C_m \frac{\alpha^{1/m}}{r}\frac{(\log(n))^{p_m}}{n^{1/m}}. \]

Finally, from the fact that $\Psi_i^{-1}$ is Lipschitz with Lipschitz constant no larger than $18$, it follows that
\[  d_\infty(\tilde{\mu}_n \llcorner_{V_\M(y_i)},\mu_n \llcorner_{V_\M(y_i)})  \leq 18 d_\infty\left( {\Psi_i} _{\sharp}(\tilde{ \mu}_n \llcorner_{V_\M(y_i)} ) ,  {\Psi_i} _{\sharp}( \mu_n \llcorner_{V_\M(y_i)} ) \right). \]
From the previous discussion, we deduce that with probability at least $1-C\nt n^{-\beta}=1-C_{m,\beta,\alpha,r,\vol(\M)}\cdot n^{-\beta}$,
\begin{align*}
 d_\infty(\mu, \mu_n) & \leq d_\infty(\mu, \tilde{\mu}_n) + d_\infty(\tilde{\mu}_n, \mu_n) \\
 & \leq  C' \biggl(  \sqrt{\frac{\log(n)}{n}} +  \frac{(\log(n))^{p_m}}{n^{1/m}}  \biggr) \leq  C' \frac{(\log(n))^{p_m}}{n^{1/m}} 
\end{align*}
for a constant $ C'$ that can be written as $C' = \frac{C_{\alpha,\beta,m}}{r^m} \tilde{C} $, where $\tilde{C}$ is as in \eqref{tildeC}.
\nc
\end{proof}
\nc

\section{Kernel-based approximation of the Laplacian}
\label{sec:Kernelbased}
Here we focus on a kernel-based approximation of the continuous Dirichlet form defined in \eqref{eqn:contform}. 
This part does not depend on the graph obtained from the sample set $X$ and can be seen as the bias part of the desired error estimates.

The results in this section correspond to those of Section 3 and 5 in \cite{BIK} but cannot be directly infered from them.
Instead, we need to adjust most of the proofs to our setting. 

\medskip

For $f\in L^2(\M)$, $0<r<2h$ and a Borel set $V\subseteq \M$ let
\begin{align}
\label{eqn:ErfV}
	 E_r(f,V) \coloneqq \int_V \int_{\M} \eta\left(\frac{d(x,y)}{r} \right) \abs{f(y)-f(x)}^2 d \mu(y) d \mu(x).
\end{align}
 We write $E_r(f)$ shorthand for $E_r(f,\M)$. The main results of this section, Lemma \ref{lem:Ervsdf} and \ref{lem:normdLambdaf},
 demonstrate how this functional approximates the form $\Delta$.

\begin{remark}
\label{rem:Erwithkernels}
	Let $\tilde E_r(f,V)$ denote the functional in \eqref{eqn:ErfV} when $\eta$ is taken to be the kernel $\mathds{1}_{[0,1]}$.
	Then $\tilde E_r(f,V)$ is nothing but $E_r(f,V)$ as defined in \cite[Def.\ 3.1]{BIK}.
       Note that, for general $\eta$ satisfying the assumptions from Section \ref{subsec:Setting} 
\begin{equation}
\tilde E_r(f,V) \leq \frac{1}{\eta(1/2)} E_{2r} (f,V).
\label{eqn:tildeEvsE}  
\end{equation}
for every $f\in L^2(\M)$ and any Borel set $V\subseteq \M$.
\end{remark}

\begin{lemma}
\label{lemma:fourthradius}
 Suppose $h$ satisfies Assumptions \ref{stdassumptions}\nc. Then there exists a universal constant $C>0$ such that for every $0<r<2h$  and every $f \in L^2(\M, \mu)$
\[   E_{r}(f) \leq C 2^m (1+  \alpha L_p) E_{r/2}(f).\]
\end{lemma}
\begin{proof} 
Let $0<r<2h$. Then $r\leq \min\{i_0, 1/\sqrt{K}\}$ by Assumptions \ref{stdassumptions}.
Note that it suffices to consider $f$ to be smooth because smooth functions are dense in $L^2(\M,\mu)$ and both sides of the inequality are continuous with respect to $L^2$-convergence; notice that for smooth functions we can talk about pointwise values.  For $x,y \in \M$ with $d(x,y) \leq r$ let $z_{xy}$ be the point in $\M$ which lies halfway along the geodesic connecting $x$ and $y$, i.e.\ $z_{xy} = \exp_x(\frac{1}{2}\exp_x^{-1}(y))$.
In particular $d(x,z_{xy}) = d(y,z_{x,y}) = \frac{1}{2}d(x,y)$. Since $\abs{ f(x)- f(y)}^2 \leq 2\abs{f(x) - f(z_{x,y})}^2 + 2\abs{f(y) - f(z_{x,y})}^2 $, by symmetry we obtain
\begin{align*}
 E_{r}(f) & \leq 4 \int_{\M}\int_{\M} \eta\left(\frac{d(x,y)}{r}\right)\abs{f(x)- f(z_{x,y}) }^2 d \mu(y) d\mu(x)\\
& = 4 \int_{\M}\int_{B(r)} \eta\left(\frac{\abs{v}}{r}\right)\abs[\Big]{f(x)- f\left(\exp_x\left(\frac{v}{2}\right)\right) }^2 J_x(v)p(\exp_x(v)) dv d\mu(x)\\
& \leq C2^m(1 + \alpha L_p ) \int_{\M}\int_{B(\frac{r}{2})} \eta\left(\frac{2\abs{w}}{r}\right)\abs{f(x)- f(\exp_x\left(w)\right) }^2 J_x(w) \\
& \phantom{ \leq C2^m(1 + \alpha L_p ) \int_{\M}\int_{B(\frac{r}{2})} }\;\,
p(\exp_x(w)) dw d\mu(x)\\
&= C2^m(1 + \alpha L_p ) E_{r/2}(f),
\end{align*}
where $C$ is a universal constant. In the above, we used the change of variables $w= \frac{v}{2}$ (which explains the term $2^m$) and we also used the inequalities:
\[  J_x(v) \leq (1+CmKr^2)^2 J_x \left( \frac{v}{2} \right) \leq C J_x \left( \frac{v}{2} \right),  \] 
(combined with Assumptions \ref{stdassumptions}) and 
\[ p(\exp_x(v)) \leq (1+ \alpha L_p) p(\exp_x(v/2)). \] 
\end{proof}

\begin{lemma}[{cf.\ \cite[Lemma 3.3]{BIK}}]
	\label{lem:Ervsdf} 
	 Suppose $h$ satisfies Assumptions \ref{stdassumptions} \nc.	Then, there exists a universal constant $C>0$ such that
	
	\[ E_r(f) \coloneqq E_r(f,\M) \leq (1+L_p\alpha r )\cdot(1+ Cm Kr^2) \sigma_\eta r^{m+2} D(f), \]
	\nc
for every $f\in H^1(\M)$ and $0 < r < 2h$.
\end{lemma}
\begin{proof}
Let us first consider the case in which $\eta$ takes the form $\eta = \mathds{1}_{[0,1]}$; as we will see the general case follows easily from this special case. As in \cite[Lemma 3.3]{BIK}, we may assume that $f$ is smooth and we write
\[  \int_{B_\M(x,r) } \abs{f(y)- f(x)}^2 d\mu(y) = \int_{B(r)} \abs{f(\exp_x(v))- f(x)}^2 p(\exp_x(v))J_x(v) dv  \]
where $J_x$ denotes the determinant of the Jacobian of the exponential map.
We recall from \eqref{eqn:EstimateJacobian} that there exists a constant $C>0$ such that
$J_x(v)$ is bounded from above by  $1+ CmKr^2 $ for all $v \in B(r)$.
From the fundamental theorem of calculus it follows that
\[\abs{f(\exp_x(v)) - f(x) }^2  \leq \int_{0}^{1}  \abs*{\frac{d}{dt} f(\exp_x(tv))}^2 dt = \int_{0}^{1}  \abs{ d f(\Phi_t(x,v))}^2 dt,\]
 In the above $\Phi_t$ denotes the time $t$ geodesic flow,  $\Phi_t(x,v) =( \gamma_{x,v}(t), \gamma'_{x,v}(t) )$, where $\gamma_{x,v}(t):= \exp_x(tv)$. The expression $d f(\Phi_t(x,v)$ has to be interpreted as: the form $df$ at $\gamma_{x,v}(t)$ acting on the tangent vector $\gamma'_{x,v}(t)$\nc. Therefore,
\begin{align*}
	A &\coloneqq \int_\M \int_{B(r)} \abs{f(\exp_x(v))-f(x)}^2 p(\exp_x(v)) d v\, p(x) d \vol(x) \\
	&\leq \int_0^1 \int_\M \int_{B(r)} \abs{df(\Phi_t(x,v))}^2 p(\Phi_1(x,v)_1) p(\Phi_0(x,v)_1) dv d\vol(x) dt 
\end{align*}
where $\xi \mapsto \xi_1$ denotes the projection of $\xi \in T \mathcal{M}$ on $\mathcal{M}$.
From the Lipschitz continuity of $p$, it follows that
$p(x) \leq (1+L_p\alpha r)p(y)$ for all $x,y \in\M$ where $d(x,y)\leq r$.
Using the fact that $\Phi_t$ preserves the canonical volume $\vol_{T\M}$ on $T\M$ and that
\[ \mathcal{B}_r  \coloneqq \{ \xi = (x,v)\in T\M : \abs{v} \leq r\}\]
is invariant under $\Phi_t$, see \cite[1.125]{Bess78}, we obtain after a change of variables
\begin{align*}
A&\leq (1+L_p\alpha r)^2 \int_0^1 \int_{\mathcal{B}_r } \abs{df(\Phi_t(\xi))} p^2(\Phi_t(\xi)_1) d \vol_{TM}(\xi) d t \\
	&= (1+L_p\alpha r)^2 \int_{\mathcal{B}_r} \abs{df(\xi)}^2p^2(\xi_1) d \vol_{TM}(\xi)\\
	&= (1+L_p\alpha r)^2 \int_\M \frac{\omega_m}{m+2} r^{m+2} \abs{\nabla f}^2 p^2(x) d\vol(x).
\end{align*} 
Using the previous computations, \eqref{eqn:EstimateJacobian} and Remark \ref{rem:sigmaetanum}, we deduce that
\begin{align}
\begin{split}
\label{eqn:assertionforindicator}
E_r(f) \leq (1+ CmKr^2) \cdot A  & \leq (1+CmKr^2)\cdot (1+ L_p \alpha r  )\frac{\omega_m}{m+2} r^{m+2} D(f) 
\\ & =  (1+CmKr^2)\cdot (1+ L_p \alpha r  ) \sigma_\eta r^{m+2} D(f)
\end{split}
\end{align}
for a universal constant $C$, which proves the claim for $\eta = \mathds{1}_{[0,1]}$. Now, notice that one easily obtains from the previous computations that \eqref{eqn:assertionforindicator} is still valid for $\eta$ of the form $\eta = \mathds{1}_{[0,t]}$ for some $0<t<1$. Finally, since $E_r(f)$ and $\sigma_\eta$ are linear in $\eta$, the statement holds if $\eta\colon [0,1]\to [0,\infty)$ is a decreasing step function (and hence can be written as linear combination of functions of the form $\mathds{1}_{[0,t]}$). By monotone convergence applied on both sides of the inequality, the assertion follows for any decreasing (and thus measurable) function $\eta$.\nc
\end{proof}

\begin{remark}
Note that in comparison to the case of constant $p$ treated in \cite[Lemma 3.3]{BIK}, the above estimates 
have the additional  term $(1+\alpha L_p r)$. 
\end{remark}
\nc

\begin{lemma}[{cf.\ \cite[Lemma 3.4]{BIK}}]
\label{lem:meanvariation} Suppose $h$ satisfies Assumptions \ref{stdassumptions}. Let $\veps<r<2h$, $f\in L^2(\M)$ and $V\subseteq \M$ a Borel set such that $\mu(V)>0$ and $\diam (V) \leq 2\veps$. Then 
 \[ \int_V\abs*{f(x) - \frac{1}{\mu(V)} \int_V fd\mu }^2 d \mu(x) \leq \frac{2(1+ Cm Kr^2)}{\eta(1/2)\omega_m(r-\veps)^m} E_{2r}(f,V).\]
\end{lemma}
\begin{proof}
The proof is almost identical to  the proof of \cite[Lemma 3.4]{BIK}, replacing the volume with the measure $\mu$ and taking Remark \ref{rem:Erwithkernels} into account.
\end{proof}

Next we define a smoothening operator $\Lambda\colon  L^2(\M, \rho\mu) \rightarrow \Lip(\M)$ similar to the one introduced in \cite[Section 5]{BIK} but adapted to the kernel $\eta$. To this end, we first define a mapping $\psi \colon [0,\infty) \to [0,\infty)$ by
\[  \psi(t) \coloneqq \frac{1}{\sigma_\eta} \int_{t}^{\infty} \eta(s)s  ds.\]
Note that, as $\eta$ is supported on $[0,1]$, $\psi(t)=0$ for all $t\geq 1$.
\begin{remark}
We remark that for $\eta(t) =\mathds{1}_{[0,1]}(t)$ the above $\psi$ coincides with the kernel function used in  \cite[Section 5]{BIK}. 
\end{remark}
 For every $r>0$, we define the operator $\Lambda_r^0 \colon L^2(\M,\vol) \to \Lip(\M)$ by
\begin{equation} \label{smoothingLam}
 (\Lambda_r^0 f )(x) \coloneqq \int_\M f(y) k_r(x,y) d\vol(y)
\end{equation}
where
\[ k_r(x,y) \coloneqq \frac{1}{r^{m}} \psi \left(\frac{d(x,y)}{r} \right).\]
As in \cite[Definition 5.2]{BIK}, we define the smoothing operator $ \Lambda_r \colon L^2(\M,\rho \mu) \to \Lip(\M)$ by
\begin{equation}
\Lambda_rf(x) \coloneqq (\theta(x))^{-1} \Lambda_r^0f(x),
\label{eqn:SmootheningOp}
\end{equation}
where $\theta\coloneqq \Lambda_r^0 \mathds{1}$.
Note that the term $\theta$ is introduced so that $\Lambda_r$  preserves constant functions. 

Let us deduce some useful properties of the functions just introduced.
Since $\psi'(s)= - \frac{1}{\sigma_\eta} \eta(s) s$ for all $s\geq 0$, we obtain from the mean value theorem that
for any $0 \leq t \leq r$ there exists $\frac{t}{r} \leq s\leq 1$ such that
\[ \frac{1}{r^m} \psi\left(\frac{t}{r}\right) = \frac{1}{\sigma_\eta r^m} \eta(s) s \biggl(1-\frac{t}{r}\biggr). \]
Hence, by the monotonicity of $\eta$, we have
\begin{align}
\label{eqn:krvseta}
k_r(x,y) \leq \frac{1}{\sigma_\eta r^m} \eta\left(\frac{d(x,y)}{r}\right) 
\end{align}
for every $x,y \in \M$. If $d(x,y)\leq r$, then the gradient of the kernel $k_r$ can be written as
\begin{align}
\begin{split}
 \label{eqn:d_xK}
 \nabla k_r(\cdot,y)(x) &= \frac{1}{r^{m+1}}\psi' \left(\frac{d(x,y)}{r} \right) \frac{-\exp_x^{-1}(y) }{d(x,y)} \\
 &= \frac{1}{\sigma_\eta r^{m+2}} \eta \left( \frac{d(x,y)}{r} \right) \exp_x^{-1}(y)
 \end{split}
 \end{align}
where we refer to \cite[(2.6)]{BIK} for the gradient of the distance function.
Moreover, we have 
\begin{align}
\label{eqn:normingpsi}
 \int_{\R^m}\psi(\abs{x}) dx =1.
\end{align}
To see this, first note that using polar coordinates we obtain
\[ m \sigma_\eta = \sum_{i=1}^{m} \int_{\R^m} \eta(\abs{x}) x_i^2 dx = \int_{\R^m} \eta(\abs{x}) \abs{x}^2 dx =  m \omega_m  \int_{0}^\infty \eta(r) r^{m+1}dr,\]
where $\omega_m$ is the volume of the Euclidean unit ball in $\R^m$. Thus, using integration by parts and polar coordinates, it follows that
\begin{align*}
\int_{\R^m} \psi(\abs{x}) dx &= m \omega_m \int_{0}^{\infty} \psi(r) r^{m-1}dr \\
&= - \omega_m   \int_{0}^{\infty} \psi'(r)r^m dr \\
&= \frac{\omega_m}{\sigma_\eta} \int_{0}^{\infty} \eta(r) r^{m+1}dr =1. 
\end{align*}

For $\theta(x) \coloneqq \Lambda_r^0(\mathds{1})$ we now obtain the following bounds.

\begin{lemma}[cf.\ {\cite[Lemma 5.1]{BIK}}]
\label{lem:thetabound}
	There exists an absolute constant $C>0$ such that
	\[ (1+CmKr^2)^{-1} \leq \theta(x) \leq 1+CmKr^2 \]
	and $\abs{\nabla \theta(x)} \leq CmKr/\sigma_\eta$ for all $x\in \M$.
\end{lemma}
\begin{proof}
	We have
	\[ \theta(x) = \frac{1}{r^m} \int_{B_\M(x,r)} \psi\left(\frac{d(x,y)}{r}\right) d \vol(y) = \frac{1}{r^m} \int_{B(r)} \psi\left(\frac{\abs{v}}{r}\right) J_x(v) dv.\] 
	Thus, the first assertion now follows from \eqref{eqn:EstimateJacobian} and \eqref{eqn:normingpsi}. 
    Since \eqref{eqn:EstimateJacobian} implies $\abs{J_x(v)-1}\leq CmK\abs{v}^2$ 
	and since
	\[ \int_{B(r)} \psi\left(\frac{\abs{v}}{r}\right) v dv = 0\]
	for symmetry reasons, the bound on the gradient of $\theta$ can be obtained from \eqref{eqn:d_xK} as
	\begin{align*}
	\abs{\nabla \theta(x)} &= \frac{1}{\sigma_\eta r^{m+2}} \abs*{\int_{B_\M(x,r)} \psi\left( \frac{d(x,y)}{r} \right) \exp_x^{-1}(y) d\vol(y) } \\
	&= \frac{1}{\sigma_\eta r^{m+2}} \abs*{ \int_{B(r)} \psi\left(\frac{\abs{v}}{r} \right) v  J_x(v)  dv - \int_{B(r)} \psi\left(\frac{\abs{v}}{r}\right)v dv} \\
	&= \frac{1}{\sigma_\eta r^{m+2}} \abs*{ \int_{B(r)} \psi\left(\frac{\abs{v}}{r} \right) v  (J_x(v)-1)  dv} \\
	& \leq \frac{CmKr^3}{\sigma_\eta r^{m+2}} \int_{B(r)} \psi\left( \frac{\abs{v}}{r} \right) dv = \frac{CmKr}{\sigma_\eta}. 
	\end{align*}
\end{proof}

In order to establish the following properties of $\Lambda_r$
we make use of the fact that the densities $p$ and $\rho$ are Lipschitz continuous and are bounded from
below. Thus 
\[p(x) \leq (1+L_p\alpha r)p(y) \quad \text{and}\quad \rho(x)\leq (1+L_\rho\alpha r)\rho(y)\]
 whenever $d(x,y) \leq r$.

%

\begin{lemma}[cf.\ {\cite[Lemma 5.4]{BIK}}]
	\label{lem:Lambdaf-f}
 Suppose that $h$ satisfies Assumptions \ref{stdassumptions} \nc. Then, there exists a universal constant $C>0$ such that 
  \[ \norm{\Lambda_r f}_{L^2(\M, \rho\mu)}^2 \leq  (1+ \alpha L_p r)(1+\alpha L_\rho r)(1+ Cm Kr^2) \lVert f \rVert^2_{L^2(\M, \rho \mu)}  \]
  \nc
and
	\[ \norm{\Lambda_r f-f}_{L^2(\M,\rho\mu)}^2 \leq \frac{C\alpha^2}{\sigma_\eta r^m} E_r(f)\]
	for all $f\in L^2(\M)$ and all $r<2h$.
\end{lemma}
\begin{proof}
 The first assertion follows from Jensen's inequality,
\begin{align*}
\int_{\M}(\Lambda_r f(x))^2 \rho(x) d \mu(x) & \leq \int_\M \int_\M \frac{K_r(x,y)}{\theta(x)}\rho(x) (f(y))^2 d \vol(y) d\mu(x)  
\\& \leq (1+ \alpha L_p r)(1+\alpha L_\rho r)(1+ CmKr^2) \lVert f \rVert_{L^2(\M, \rho \mu)}^2,
\end{align*}
where the last inequality follows from the Lipschitz continuity of $p$ and $\rho$ together with the estimates from Lemma \ref{lem:thetabound}.
\nc

For the second assertion notice that as in the proof of \cite[Lemma 5.4]{BIK} we can conclude that for a.e. $x$
\[ \abs{\Lambda_rf(x) - f(x)}^2 \leq \frac{1}{\theta(x)}\int_{B_\M(x,r)} k_r(x,y) \abs{f(y)-f(x)}^2 d\vol(y).\]
Integrating this inequality with respect to $\rho\mu$ and using \eqref{eqn:krvseta} we obtain that
	 \begin{align*}
	 \norm{\Lambda_r & f- f}_{L^2(\M,\rho\mu)}^2 \\ & \leq \frac{1+CmKr^2}{\sigma_\eta r^m}  \int_\M \int_\M \eta\left(\frac{d(x,y)}{r}\right) \abs{f(x)-f(y)}^2 d\vol(y) \rho(x) d\mu(x)  \\
&	\leq \frac{C}{\sigma_\eta r^m} \alpha^2 E_r(f).
	\end{align*}
\end{proof}

\begin{lemma}[[cf.\ {\cite[Lemma 5.5]{BIK}}]
	\label{lem:normdLambdaf} 
	 Suppose that $h$ satisfies Assumptions \ref{stdassumptions}. 
	Then, there exists a universal constant $C>0$ such that 
		\[ D(\Lambda_rf) \leq (1+\alpha L_pr)\cdot (1+ C(1+ 1/\sigma_\eta)mKr^2  ) \frac{1}{\sigma_\eta r^{m+2}} E_r(f) \]
		\nc
	for every $f\in L^2(\M)$ and every $0< r < 2h$.
\end{lemma}
\begin{proof}
	We can write
	\[ \nabla (\Lambda_r f) = \frac{1}{\theta(x)} A_1(x) + A_2(x) \]
	where 
	\[ A_1(x) \coloneqq \int_{B_\M(x,r)} \nabla k_r(\cdot,y)(x) (f(y)-f(x)) d\vol(y)\]
	and 
	\[ A_2(x) \coloneqq \nabla (\theta^{-1})(x)\int_{B_\M(x,r)} k_r(x,y) (f(y)-f(x)) d\vol(y).\]
Regarding $A_1$ we have $\abs{A_1(x)}= \applied{A_1(x)}{w}$ for some unit vector $w\in T_x\M$. 
Therefore, using \eqref{eqn:d_xK},
\begin{align*}
 \abs{A_1(x)}&=\applied{ A_1(x)}{ w} \\
 &= \frac{1}{\sigma_\eta r^{m+2}} \int_{B_\M(x,r)}\eta\left( \frac{d(x,y)}{r} \right) (f(y) - f(x)) \applied{\exp_x^{-1}(y)}{w} d\vol(y) \\ 
 &= \frac{1}{\sigma_\eta r^{m+2}} \int_{B(r)} \eta\left(\frac{\abs{v}}{r}\right) \varphi(v) \applied{v}{w} J_x(v) dv .
\end{align*}
where $\varphi(v) \coloneqq f(\exp_x(v))-f(x)$. 
By the Cauchy-Schwartz inequality,
\begin{align*}
 \abs{A_1(x)}^2 &\leq  \frac{1}{\sigma_\eta^2 r^{2(m+2)}} \int_{B(r)} \abs{\varphi(v)}^2 J_x(v)^2 \eta \left(\frac{\abs{v}}{r}\right) dv \,  \int_{B(r)} \langle v,w \rangle ^2 \eta \left(\frac{\abs{v}}{r}\right) dv \\
 & = \frac{1}{\sigma_\eta r^{m+2}}\int_{B(r)} \abs{\varphi(v)}^2 J_x(v)^2 \eta \left(\frac{\abs{v}}{r} \right) dv  
 \end{align*}
where, in the last step, we used radial symmetry to conclude that
 \[ \int_{B(r)} \applied{ v}{w}^2 \eta \left(\frac{\abs{v}}{r}\right) dv  = r^{m+2} \int_{B(1)} u_1^2 \eta(\abs{u}) du = r^{m+2} \sigma_\eta.\]
Now we obtain from \eqref{eqn:EstimateJacobian} that 
\begin{align*}
 \abs{A_1(x)}^2 &\leq \frac{1+ CmKr^2}{\sigma_\eta r^{m+2}} \int_{B(r)} \abs{\varphi(v)}^2\eta \left(\frac{\abs{v}}{r}\right) J_x(v)  dv \\
 &= \frac{1+CmK r^2}{\sigma_\eta r^{m+2}} \int_\M \eta\left(\frac{d(x,y)}{r}\right) (f(y)-f(x))^2 d\vol(y)
\end{align*}
\nc
 Integrating this inequality with respect to the density $p^2$ 
	and using the Lipschitz continuity of $p$, we obtain 
	\begin{align*}
&	\norm{A_1}_{L^2(\M,p^2\vol)}^2 \\
	&\leq \frac{1+ CmKr^2}{\sigma_\eta r^{m+2}} \! \int_\M \!\int_{B_\M(x,r)} \eta\left(\frac{d(x,y)}{r}\right) \abs{f(y)-f(x)}^2 d \vol(y) p^2(x) d\vol(x) \\
	&\leq \frac{(1+\alpha L_p r)(1+CmKr^2)}{\sigma_\eta r^{m+2}} \!\! \int_\M \! \int_{B_\M(x,r)} \!\!\eta\left(\frac{d(x,y)}{r}\right) \abs{f(y)-f(x)}^2 d \mu(y) d\mu(x)\\
	&\leq \frac{(1+\alpha L_p r)(1+CmKr^2)}{\sigma_\eta r^{m+2}} E_r(f).
	\end{align*}
Regarding $A_2$, first note that $\abs{\nabla (\theta^{-1})}\leq CmKr/\sigma_\eta$ and $\theta \leq C$ by Lemma \ref{lem:thetabound}. Therefore,
by the Cauchy-Schwartz inequality and \eqref{eqn:krvseta}, we obtain
\begin{align*}
	\abs{A_2(x)}^2 &\leq \abs{ \nabla (\theta^{-1})}^2 \int_\M k_r(x,y)dy \int_\M \abs{f(y)-f(x)}^2k_r(x,y) d\vol(y) \\
	&= \abs{\nabla (\theta^{-1})}^2 \theta(x) \int_\M \abs{f(y)-f(x)}^2k_r(x,y) d\vol(y) \\
	&\leq  \frac{C m^2 K^2 r^2}{\sigma_\eta^{3} r^m} \int_\M \eta\left(\frac{d(x,y)}{r}\right) \abs{f(y)-f(x)}^2 d \vol(y)
\end{align*}
Integrating this inequality with respect to the density $p^2$ while using the Lipschitz continuity of $p$ shows that
\[  \norm{A_2}_{L^2(\M,p^2\vol)}  \leq  \frac{C(1+ \alpha L_p r) mK r^2}{\sigma_\eta} \sqrt{\frac{1}{\sigma_\eta r^{m+2}} E_r(f) } \]
for some universal constant $C$. By combining these estimates and the lower bound for $\theta$ from Lemma \eqref{lem:thetabound} we obtain that
\[ D(\Lambda_r f)^\frac{1}{2} \leq (1+\alpha L_p r)\cdot ( 1+ C(1+ 1/\sigma_\eta) mKr^2 )  \sqrt{\frac{1}{\sigma_\eta r^{m+2}} E_r(f) }. \]
Hence the claim follows.
\end{proof}

\section{Convergence of eigenvalues}
\label{sec:Eigenvalues}

In order to prove Theorem \ref{thm:conveigenvalues} we estimate the discrete Dirichlet form \eqref{eqn:discreteform} in terms of the continuous one \eqref{eqn:contform} while we interpolate and discretize between the graph and the manifold 
in an almost isometric manner using the mappings $P$, $P^*$ from \eqref{def:P}, \eqref{def:P*} and $\Lambda_r$ from \eqref{eqn:SmootheningOp}.
We start this section with some preliminary lemmas.
\nc

\medskip

\begin{lemma}
\label{lem:etaestimates}
Let us assume that the support of $\eta$ is contained in $[0,1]$ and that $\eta$ is Lipschitz in $[0,1]$. Then, for all $r,s>0$ and $t\geq 0$ we have
	\begin{enumerate}[(i)]
	\item $\eta\left(\frac{t}{r+s}\right) \leq \eta\left(\frac{(t-s)_+}{r}\right) \leq \eta\left(\frac{t}{r+s}\right) + L_\eta\frac{s}{r}\mathds{1}_{\{t\leq r+s\}}$
	\item $\eta\left( \frac{t+s}{r} \right) \geq \eta\left( \frac{t}{r-s} \right) - L_\eta \frac{s}{r} \mathds{1}_{\left\{ t \leq r-s \right\}}$ provided that $s < r$.\nc
	\end{enumerate}
where $L_\eta > 0$ denotes the Lipschitz constant of $\eta$ restricted to $[0,1]$.
\end{lemma}
\begin{proof}
Regarding assertion (i) first note that every term vanishes for $t> r+s$.
In order to prove the first inequality in the remaining case, we need to verify that
$(t-s)/r \leq t/(r+s)$ provided that $t \leq r+s$.
This follows from 
\begin{align*}
	\label{eqn:etaarguments}
	\frac{t}{r+s} - \frac{t-s}{r} = \frac{rt - (r+s)(t-s)}{r(r+s)} = \frac{s}{r}\frac{r+s-t}{r+s}= \frac{s}{r} \left(1-\frac{t}{r+s}\right) > 0.
\end{align*}
Combining this estimate with the Lipschitz continuity of $\eta$ shows that
\begin{align*}
0& \leq \eta\left(\frac{(t-s)_+}{r}\right) - \eta\left(\frac{t}{r+s}\right)  \leq  L_\eta \left( \frac{t}{r+s} - \frac{(t-s)_+}{r}  \right) \\
& \leq L_\eta \left( \frac{t}{r+s} - \frac{t-s}{r}  \right)  = L_\eta   \frac{s}{r} \left(1-\frac{t}{r+s} \right)  \leq L_\eta\frac{s}{r} 
\end{align*}
\nc
which implies the second inequality of assertion (i). The proof of assertion (ii) is completely analogous.
\end{proof}

The next results relate the operators $P$ and $P^*$ defined in \eqref{def:P} and \eqref{def:P*}. In particular, we show that $P$ and $P^*$ are almost adjoint to each other and that $P^*$ is almost an isometry. In case $\vec{m}\mu_n=\mu_n$ and $\rho \mu= \mu$, 
(i.e. in case $m = (1,\dots,1)$ and $\rho\equiv 1$) then $P$ and $P^*$ are truly adjoint to each other and $P^*$ is truly an isometry.
\begin{lemma}
\label{lem:P*almostisometric}
For all  $u\in L^2(X)$ and $f\in L^2(\M)$ 
\[  \left \lvert \applied{P^* u}{f}_{L^2(\M,\rho \mu)} \!- \!\applied{u}{Pf}_{L^2(X,\vec{m}\mu_n)}  \right \rvert  \leq \alpha (\lVert \vec{m}- \rho \rVert_{\infty} + \veps L_\rho )\applied{P^*\abs{u}}{\abs{f}}_{L^2(\M,\rho \mu)} \] 
and
\[ 	\abs[\big]{\norm{P^*u}^2_{L^2(\M,\rho \mu)} - \norm{u}^2_{L^2(X,\vec{m}\mu_n)} }  \leq  \alpha(\lVert \vec{m}- \rho \rVert_{\infty} + \veps L_\rho   ) \norm{P^*u}^2_{L^2(\M,\rho \mu)}. \]
\nc
Moreover, if we assume that  $\alpha \lVert \vec{m} - \rho\rVert_\infty \leq \frac{1}{2}$ then $\forall u \in L^2(X)$,
\[ \lVert P^* u \rVert^2_{L^2(\M, \rho \mu)} \leq 2(1+ \alpha L_\rho \veps) \norm{u}^2_{L^2(X, \vec{m} \mu_n)}  \]
for some universal constant $C>0$.
\end{lemma}
\nc
\begin{proof}
We infer from \eqref{LinfWeights} that
\begin{align*} 
\abs{\applied{u}{Pf&}_{{L^2(X,\vec{m}\mu_n)}}  -  \applied{P^*u}{f}_{L^2(\M,\rho \mu)}}  \\
	& =  \abs*{\sum_{i=1}^n \frac{m_i}{n} u(x_i) \cdot n \int_{U_i} f d \mu - \int_\M \sum_{i=1}^n u(x_i)\mathds{1}_{U_i} f \rho d \mu} \\
	& \leq  \int_\M \sum_{i=1}^n \abs{ u(x_i)} \mathds{1}_{U_i} \abs{f(x)} \cdot \abs[\big]{m_i - \rho(x_i) + \rho(x_i)-\rho(x)}  d \mu\\
	& \leq  \alpha (\lVert \vec{m}- \rho \rVert_{\infty} + \veps L_\rho ) \applied{P^*\abs{u}}{\abs{f}}_{L^2(\M,\rho \mu)}.
\end{align*}
%
%
and
\begin{align*}
	\abs[\big]{\norm{P^*u}^2_{L^2(\M,\rho \mu)} - & \norm{u}^2_{L^2(X,\vec{m}\mu_n)} } \\
	&= \abs*{ \sum_{i=1}^n \left( \int_{U_i} u^2(x_i) \rho d\mu - \int_{U_i}m_iu^2(x_i) d\mu \right)}\\
	&\leq \sum_{i=1}^n \int_{U_i} u^2(x_i) \abs[\big]{\rho(x)-\rho(x_i)+\rho(x_i)-m_i} d\mu \\
	&\leq \alpha(\lVert \vec{m}- \rho \rVert_{\infty} + \veps L_\rho   ) \norm{P^*u}^2_{L^2(\M,\rho \mu)}.
\end{align*}
To prove the last part of the lemma we notice that 
\begin{align*}
\begin{split}
\norm{P^* u }^2_{L^2(\M, \rho \mu)} = \sum_{i=1}^n u(x_i)^2 \int_{U_i}\rho(y) d \mu(y) & \leq \frac{2(1+ \alpha L_\rho \veps)}{n}\sum_{i=1}^n u(x_i)^2 m_i \\
& = 2(1+ \alpha L_\rho \veps) \lVert u \rVert_{L^2(X, \vec{m} \mu_n)}^2. 
\end{split}
\end{align*}
\end{proof}

The next lemma is a straightforward generalization of \cite[Lemma 4.2]{BIK}.
\begin{lemma}[cf.\ {\cite[Lemma 4.2]{BIK}}]
	\label{lem:P*Palmostinverse}
	For every $f \in L^2(\M)$ we have
	\[  \norm{P^*P f}_{L^2(\M, \rho \mu)}^2 \leq (1+ 2\alpha L_\rho \veps)\norm{f}_{L^2(\M , \rho\mu)}^2   \]
In addition, there exists a universal constant $C> 0$ such that
	\[ \norm{f-P^*Pf}_{L^2(\M,\mu)} \leq \frac{C(1+m\alpha L_p \veps )m 2^{m/2} \sigma_\eta^{1/2}}{\sqrt{\eta(1/2)\omega_m}}   
\veps D(f)^{\frac{1}{2}} \]
	for all $f\in H^1(\M)$.  \nc
	
\end{lemma}
\begin{proof}
The first assertion follows from Jensen's inequality and the Lipschitz continuity of $\rho$:
\begin{align*}
\int_{\M} (P^* P f (x))^2 \rho(x) d \mu(x) & \leq \sum_{i=1}^n \int_{U_i}  \int_{U_i}  n f(y)^2 \rho(x) d \mu(y)  d \mu(x)  
\\ & \leq (1+ 2\alpha L_\rho \veps) \sum_{i=1}^n \int_{U_i}  \int_{U_i}  n f(y)^2 \rho(y) d \mu(y)  d \mu(x)
\\&= (1+2 \alpha L_\rho \veps)\int_{\M} f(y)^2 \rho(y) d \mu(y).
\end{align*}

For the second assertion we can use Lemma \ref{lem:Ervsdf}, Lemma \ref{lem:meanvariation} and Assumptions \ref{stdassumptions} on $h$, to obtain 
\begin{align*}  \norm{f-P^*Pf}^2_{L^2(\M,\mu)} & \leq \frac{2(1+Cm Kr^2)}{\eta(1/2)\omega_m(r-\veps)^m}E_{2r}(f) \\
&\leq \frac{C (1+2\alpha L_pr)2^m \sigma_\eta}{\eta(1/2) \omega_m} \frac{r^m}{(r-\veps)^m} r^2 D(f) 
\end{align*}
for any $r\in (\veps,2h )$. When choosing $r=(m+1)\veps$ the quotient $\frac{r^m}{(r-\veps)^m}$ is bounded by $3$ and the assertion follows.
\end{proof}
\nc

The next lemma is a generalization of \cite[Lemma 4.3]{BIK}.
\begin{lemma}[cf.\ {\cite[Lemma 4.3]{BIK}}]
	\label{lem:discretization}
	The following assertions hold:	
\begin{enumerate}[(i)]
\item For every $f\in H^1(\M)$, 
\begin{align*}
\abs[\big]{\norm{Pf}_{L^2(X,\vec{m}\mu_n)}^2 - \norm{f}_{L^2(\M,\rho\mu)}^2 }  \leq & \alpha( \lVert \vec{m} - \rho \rVert_{\infty} + \veps L_\rho  )  \norm{f}_{L^2(\M,\rho\mu)}^2\\
&  +  \tilde C' \veps \lVert f  \rVert_{L^2(\M,\rho\mu)} D(f)^{\frac{1}{2}},
\end{align*}
where $\tilde{C}'$ has the form
\begin{equation}
\tilde C' = \frac{C \alpha (1+ \alpha L_\rho)  (1+m\alpha L_p )m 2^{m/2} \sigma_\eta^{1/2}}{\sqrt{\eta(1/2)\omega_m}},
\label{defC'}
\end{equation}
  for some universal constant $C>0$.
	
	\item For every $f\in H^1(\M)$, 
	\[ b(Pf)\leq (1+ C_1'  h+ C_2' \frac{\veps}{h} + C_3' h^2 )D(f),\]
	\end{enumerate}
	 where the constants $C_1'$, $C_2'$, $C_3'$ can be written in terms of geometric quantities as
	\[ C_1'= C \alpha L_p  , \quad    C_2'= C \left(m + \frac{2^{m+1}L_\eta(1+  \alpha L_p)}{\eta(1/2)}\right)  ,\quad C_3'=Cm \left(K + \frac{1}{R^2} \right), \]
	where $C$ is a universal constant. 
	\nc
\end{lemma}
\begin{proof}
Since $P^*$ is almost an isometry by Lemma \ref{lem:P*almostisometric}, we have
\begin{align*} 
\abs[\big]{\norm{Pf}_{L^2(X,\vec{m}\mu_n)}^2 & -  \norm{f}_{L^2(\M,\rho\mu)}^2} \\
\leq & \, \abs[\big]{ \norm{Pf}^2_{L^2(X,\vec{m}\mu_n)} - \norm{P^*Pf}_{L^2(\M,\rho\mu)}^2 } 
\\&+ \abs[\big]{\norm{P^*Pf}_{L^2(\M,\rho\mu)}^2-\norm{f}_{L^2(\M,\rho\mu)}^2}\\
\leq & \, \alpha( \lVert \vec{m}- \rho \rVert_\infty + \veps L_\rho )  \norm{P^*Pf}_{L^2(\M,\rho\mu)}^2 
\\& +  ( \norm{P^*Pf }_{L^2(\M,\rho\mu)} + \norm{f}_{L^2(\M,\rho\mu)} ) \norm{P^*Pf-f}_{L^2(\M,\rho\mu)}  \\
 \leq &\, \alpha( \lVert \vec{m}- \rho \rVert_\infty + \veps L_\rho )(1+ 2\alpha L_\rho \veps)\lVert f \rVert^2_{L^2(\M, \rho d \mu)}
\\&+ \frac{C \alpha (2+ \alpha L_\rho \veps)  (1+m\alpha L_p \veps )m 2^{m/2} \sigma_\eta^{1/2}}{\sqrt{\eta(1/2)\omega_m}}   
\veps    \lVert f \rVert_{L^2(\M, \rho d \mu)}  D(f)^{\frac{1}{2}}
\end{align*}
where the last inequality follows from Lemma \ref{lem:P*Palmostinverse} and from the boundedness of $\rho$. This proves the first assertion.

Regarding the second assertion we follow the proof of \cite[Lemma 4.3(ii)]{BIK} and obtain that
\[ \abs{Pf(x_j) -Pf(x_i)}^2 \leq n^2 \int_{U_i}\int_{U_j} \abs{f(y)-f(x)}^2 d\mu(y) d\mu(x).\]
Let  $\hat{h}:= (1+ \frac{27}{\rc^2}h^2)h$. Then, by Proposition \ref{prop:metricestimates}, Lemma \ref{lem:etaestimates} and by the monotonicity of $\eta$ we have
\begin{align*} b(Pf)  &\leq \frac{1}{\sigma_\eta h^{m+2}} \sum_i \sum_j \int_{U_i}   \int_{U_j} \eta\left( \frac{\abs{x_i-x_j}}{h}\right) \abs{f(y)-f(x)}^2 d\mu(y) d\mu(x) \\
&\leq \frac{1}{\sigma_\eta h^{m+2}} \sum_i\sum_j  \int_{U_i} \int_{U_j} \eta\left(  \frac{d(x_i,x_j)}{\hat{h}} \nc \right)   \abs{f(y)-f(x)}^2 d\mu(y) d\mu(x) \\
&\leq \frac{1}{\sigma_\eta h^{m+2}} \int_\M \int_\M \eta\left(  \frac{(d(x,y)-2\veps)_+}{\hat h} \nc \right)   \abs{f(y)-f(x)}^2 d\mu(y) d\mu(x) \\
&\leq \frac{1}{\sigma_\eta h^{m+2}} \int_\M \int_\M \left(\eta\left( \frac{d(x,y)}{  \hat h + 2 \veps \nc} \right) + 2L_\eta\frac{\veps}{\hat h}\mathds{1}_{B_\M(x, \hat{h}+ 2 \veps)}(y)  \nc \right)  \\
& \qquad \qquad \qquad \qquad \quad \abs{f(y)-f(x)}^2 d\mu(y) d\mu(x)\\
&= \frac{1}{\sigma_\eta h^{m+2}} \left(E_{  \hat{h} +2\veps \nc} (f) +  \frac{2L_\eta}{\eta(1/2)}\frac{\veps}{h} E_{2(\hat{h}+ 2 \veps)}(f)\right),
\end{align*}
where we refer to Remark \ref{rem:Erwithkernels} to justify the last step.
Due to Assumptions \ref{stdassumptions}, we obtain from Lemma \ref{lem:Ervsdf} that
\begin{align*}
\frac{1}{\sigma_\eta h^{m+2}} E_{\hat h + 2\veps}(f) & \leq  (1+ C \alpha L_p h ) (1+ Cm K h^2)  \left(1+\frac{27 h^2}{R^2}+2\frac{\veps}{h}\right)^{m+2} D(f)  
 \\ &  \leq  (1+ C \alpha L_p h ) (1+ Cm K h^2)   \left(1+Cm \frac{h^2}{R^2}+Cm \frac{\veps}{h}\right) D(f)   ,
 \end{align*}
where the last inequality is obtained from the fact that
\[ (1+ s)^m \leq 1+ Cs , \quad \forall  0 \leq s \leq \frac{3}{m},  \] 
 for some universal constant $C>0$. Likewise, we obtain 
\begin{align*}
 \frac{1}{\sigma_\eta h^{m+2}} \frac{2L_\eta}{\eta(1/2)}\frac{\veps}{h} E_{2(\hat{h}+ 2 \veps)}(f) & \leq \frac{2^{m+1}L_\eta}{\eta(1/2)} (1+ C \alpha L_p h ) (1+ Cm K h^2) 
 \\ & \quad\cdot \left(1+Cm \frac{h^2}{R^2}
 +Cm \frac{\veps}{h}\right)  \frac{\veps}{h} D(f).
 \end{align*}
\nc
The result follows directly from the previous estimates.
\end{proof}
We can now establish an upper bound for $\lambda_k(\Gamma)$ in terms of $\lambda_k(\M)$.
\begin{proof}[of upper bound of Theorem \ref{thm:conveigenvalues}]
	Fix $k\in\N$. By the minmax principle \eqref{eqn:Courant0} we have
	\[ \lambda_k(\Gamma) \leq \sup_{u\in L\setminus\{0\}} \frac{b(u)}{\norm{u}_{L^2(X,\vec{m}\mu_n)}^2}\]
	for every $k$-dimensional subspace $L\subseteq L^2(X,\vec{m}\mu_n)$.
	Following the proof of \cite[Prop 4.4]{BIK} we denote by
	$W\subset H^1(\M)$ the span of orthonormal (with respect to the $L^2(\M, \rho \mu)$ inner product) eigenfunctions of $\Delta$ corresponding to $\lambda_1(\M),\dots,\lambda_k(\M)$
	and we set $L\coloneqq P(W)$. 
	For every $f\in W$ we have $D(f)\leq \lambda_k(\M)\norm{f}_{L^2(\M, \rho \mu)}^2$. It thus follows from part (i) of Lemma \ref{lem:discretization} that 
	\begin{align}
	\label{eqn:normu-lowerbound}
	\norm{Pf}_{L^2(X,\vec{m}\mu_n)}^2\geq (1- \alpha( \lVert  \vec{m } -  \rho \rVert_\infty  + \veps L_\rho ) -  \tilde C' \sqrt{\lambda_k(\M)} \veps )\norm{f}^2_{L^2(\M,\rho\mu)}.
	\end{align}
	 Hence, provided that 
	 \[   \alpha( \lVert  \vec{m } +  \rho \rVert_\infty  + \veps L_\rho ) +  \tilde C' \sqrt{\lambda_k(\M)} \veps \leq \frac{1}{2} ,\] 
	 we can conclude that $P$ is injective on $W$ and therefore $\dim L = k$. \nc Moreover, in that case by applying part (ii) of Lemma \ref{lem:discretization} to $u=Pf \in L$ we obtain that
\begin{align*}
	&  \frac{b(u)}{\lVert  u \rVert^2_{L^2(\M, \vec{m}\mu_n)}}  \leq \frac{\left(1+C_1' h + C_2' \frac{\veps}{h} + C_3' h^2 \right)}{1 - \alpha (  \lVert \vec{m}- \rho \rVert_\infty + \veps L_\rho)- \tilde C' \sqrt{\lambda_k(\M) } \veps} \lambda_k(\M)
	\\& \leq  \left(1+C_1' h + C_2' \frac{\veps}{h} + C_3' h^2 + \alpha C (\lVert \vec{m} - \rho \rVert_\infty + \veps L_\rho) + \tilde C' \sqrt{\lambda_k (\M)} \veps \right) \lambda_k(\M).
	\end{align*} \nc
	  Since the previous inequality holds for every $u= P f$ with $f \in W$, the desired estimate now follows.
\end{proof}

\nc


\begin{lemma}[cf.\ {\cite[Lemma 6.2]{BIK}}]
\label{lem:interpolation}
 Suppose that $h$ satisfies Assumptions \ref{stdassumptions}. Then, \nc
	\begin{enumerate}[(i)]
	\item  For every $u \in L^2(X)$, 
\begin{align*}
\abs[\big]{\norm{Iu}^2_{L^2(\M,\rho\mu)}- & \norm{u}^2_{L^2(X,\vec{m}\mu_n)}} \leq  \tilde C'' h \lVert u \rVert_{L^2(X,\vec{m}\mu_n)}\cdot b(u)^{\frac{1}{2}} \\
& \quad + 2\alpha(1+\alpha L_\rho) \cdot (  \norm{\vec{m}- \rho}_\infty + L_\rho \veps ) \norm{u}_{L^2(X,\vec{m}\mu_n)}^2,
\end{align*}
where the constant $\tilde C''$ can be written as
\[  \tilde C''= C\alpha(1+ \alpha L_p)\cdot(1+\alpha L_\rho) \cdot \left( 1+ c'' \right), \quad c''= \frac{L_\eta 4^m \omega_m^2(1+\alpha L_p)^2 }{\eta(1/2)(m+2)}.\]
\nc
	\item   For every $u \in L^2(X)$, 
	\[ D(Iu) \leq (1+ C_1'' h + C_2'' \frac{\veps}{h} + C_3''h^2 )b(u),\]
	where the constants $C_1''$, $C_2''$, $C_3''$ have the form
	\[C_1''= \alpha L_p , \quad C_2''= C( m + C_2') , \quad C_3''= C(1+ 1/\sigma_\eta)m K.  \]
	\end{enumerate}
\end{lemma}
\begin{proof}
\nc
First, by Lemma \ref{lem:P*almostisometric}, 
\begin{align}
\begin{split}
 \abs[\big]{\norm{Iu&}^2_{L^2(\M,  \rho \mu)}- \norm{u}^2_{L^2(X,\vec{m} \mu_n)}} \\
  \leq &\, \abs[\big]{\norm{Iu}^2_{L^2(\M, \rho \mu)}- \norm{P^*u}^2_{L^2(\M,\rho \mu)}} +\abs[\big]{\norm{P^*u}^2_{L^2(\M,\rho \mu)}-\norm{u}^2_{L^2(X,\vec{m} \mu_n)}} 
 \\ \leq &\,  (  \norm{Iu}_{L^2(\M, \rho \mu)} +  \norm{P^*u}_{L^2(\M,\rho \mu)}   ) \norm{Iu-P^*u}_{L^2(\M,\rho \mu)} 
 \\ &+  \alpha (\lVert \vec{m}- \rho \rVert_{\infty} + \veps L_\rho ) \norm{P^*u}^2_{L^2(\M,\rho \mu)}.
 \end{split}
 \label{aux:Lem62}
\end{align}
Since $(m+2)\veps < h$ (by Assumption \ref{stdassumptions}), we conclude from Lemma \ref{lem:Lambdaf-f} that
\[ \norm{Iu - P^*u}^2_{L^2(\M,\rho\mu)} = \norm{\Lambda_{h-2\veps}P^*u - P^*u}^2 \leq  \frac{C\alpha^2 }{\sigma_\eta h^m} \nc E_{h-2\veps}(P^*u).\]
for some universal constant $C>0$. 

Let us now estimate $E_{h-2\veps}$ in terms of $b(u)$. First consider the kernel $\tilde{\eta}=\mathds{1}_{[0,1]}$. We use $\tilde{b}$ and $\tilde{E}$ to denote the discrete Dirichlet form and the energy $E$ when using the kernel $\tilde{\eta}$
and we write $b_h$ and $\tilde b_h$, respectively, to specify that the forms $b$ and $\tilde b$ are being constructed using the value $h$.
We claim that
\begin{equation}
\label{eqn:tildebhlower}
 \tilde{b}_h(u) \geq    \frac{m+2}{\omega_m h^{m+2}} \nc \tilde{E}_{h-2\veps} (P^*u).
\end{equation}
Indeed, let $T$ denote the transportation map 
introduced in Section \ref{sec:mainresults} satisfying $U_i = T^{-1}(x_i)$, then
\begin{align*}
&	\tilde{b}_h(u) = \frac{1}{\sigma_{\tilde{\eta}} h^{m+2}} \frac{1}{n^2} \sum_i \sum_j \tilde{\eta}\left(\frac{\abs{x_i-x_j}}{h}\right) \abs{u(x_i)-u(x_j)}^2 \\
	&= \frac{1}{\sigma_{\tilde{\eta}} h^{m+2}} \sum_{i,j}\int_{U_i} \int_{U_j} \tilde{\eta}\left(\frac{\abs{T(x) -T(y)}}{h}\right) \abs{(P^*u)(x) - (P^*u)(y)}^2 d\mu(y)d\mu(x)\\ 
	&\geq \frac{1}{\sigma_{\tilde{\eta}} h^{m+2}} \int_\M \int_\M \tilde{\eta}\left(\frac{d(T(x),T(y))}{h} \right)\abs{(P^*u)(x) - (P^*u)(y)}^2 d\mu(y)d\mu(x)\\
	&\geq  \frac{1}{\sigma_{\tilde{\eta}} h^{m+2}} \int_\M \int_\M  \tilde{\eta}\left(\frac{d(x,y)}{h-2\veps}\right)\abs{(P^*u)(x)-(P^*u)(y)}^2 d\mu(y)d\mu(x) \\
	&=  \frac{1}{\sigma_{\tilde{\eta}} h^{m+2}} \tilde{E}_{h-2\veps} (P^*u),
\end{align*}
where we note that the last inequality follows from the fact that  $d(T(x),T(y)) > h$ implies that $d(x,y) > h-2 \veps$;  we have used Remark \ref{rem:sigmaetanum} to rewrite $\sigma_{\tilde \eta}$ \nc.  We now consider general $\eta$. 
Since $\eta(t) \geq \eta(1/2) >0 $ for all $t \in [0,1/2]$, it follows that 
\begin{equation}
  \tilde{b}_{h/2}(u) \leq    \frac{\sigma_\eta \omega_m 2^{m+2}}{\eta(1/2) (m+2)}  \nc b_h(u).
  \label{lemma:aux4.8}
\end{equation}
On the other hand, by the monotonicity of $\eta$ and Lemma \ref{lem:etaestimates} we obtain
\begin{align*}
	b_h(u) &
	\geq \frac{1}{\sigma_\eta h^{m+2}} \! \int_\M \int_\M \! \eta\left(\frac{d(T(x),T(y))}{h} \right)\abs{(P^*u)(x) - \! (P^*u)(y)}^2 d\mu(y)d\mu(x)\\
	&\geq \frac{1}{\sigma_\eta h^{m+2}} \int_\M \int_\M \eta\left(\frac{d(x,y)+2\veps}{h}\right) \abs{(P^*u)(x) - (P^*u)(y)}^2 d\mu(y)d\mu(x)\\
	&\geq  \frac{1}{\sigma_\eta h^{m+2}} \int_\M \int_\M  \eta\left(\frac{d(x,y)}{h-2\veps}\right)\abs{(P^*u)(x)-(P^*u)(y)}^2 d\mu(y)d\mu(x)  \\
	& \quad\!-   \frac{L_\eta  }{\sigma_\eta} \nc \frac{ \veps}{h}\frac{1}{ h^{m+2}} \int_\M \int_\M   \mathds{1}_{\left\{ d(x,y)\leq h-2\veps \right\}}\abs{(P^*u)(x)-(P^*u)(y)}^2 d\mu(y)d\mu(x)\\
	&=  \frac{1}{\sigma_\eta h^{m+2}} E_{h-2\veps} (P^*u)  -  \frac{L_\eta  }{\sigma_\eta} \nc \frac{\veps}{h}\frac{1}{ h^{m+2}} \tilde{E}_{h-2\veps} (P^*u)\\
	& \geq \frac{1}{\sigma_\eta h^{m+2}} E_{h-2\veps} (P^*u)  -  \frac{C L_\eta 4^m(1+\alpha L_p)^2 }{\sigma_\eta}  \nc\frac{\veps}{h}\frac{1}{ h^{m+2}} \tilde{E}_{ \frac{h}{2}-2\veps \nc} (P^*u),
\end{align*}
where the last inequality follows after applying Lemma \ref{lemma:fourthradius} twice.
We conclude from \eqref{eqn:tildebhlower} that
\[  b_h(u) \geq \frac{1}{\sigma_\eta h^{m+2}} E_{h-2\veps} (P^*u) -   \frac{C L_\eta 2^m \omega_m(1+\alpha L_p)^2 }{(m+2)\sigma_\eta}  \nc  \frac{\veps}{h} \tilde{b}_{\frac{h}{2}}(u).\]
Combining this inequality with \eqref{lemma:aux4.8} we deduce that
\[ \left(1+   \frac{C L_\eta 4^m \omega_m^2 (1+\alpha L_p)^2}{\eta(1/2)(m+2)^2} \nc \frac{\veps}{h}\right)b_h(u)  \geq   \frac{1}{\sigma_\eta h^{m+2}} E_{h-2\veps} (P^*u)\]
which can be rewritten as
\begin{align}
\label{eqn:buVSEr}
E_{h-2\veps}(P^*u) \leq \left(1+   \frac{C L_\eta 4^m \omega_m^2 (1+\alpha L_p)^2}{\eta(1/2)(m+2)^2} \nc \frac{\veps}{h}\right)  \sigma_\eta h^{m+2} b(u).
\end{align}
Hence,
\begin{align}
\label{eqn:Iu-P*u}
 \norm{Iu-P^*u}^2 \leq  \frac{C\alpha^2 }{\sigma_\eta h^m}  \nc E_{h-2\veps}(P^*u) \leq  C \alpha^2 \!\left(\!1+  \frac{ L_\eta 4^m \omega_m^2 (1+\alpha L_p)^2}{\eta(1/2)(m+2)^2} \frac{\veps}{h}\right) \!  h^2 b(u).
\end{align}
 Finally, from Lemma \ref{lem:P*almostisometric} it follows that
\[ \norm{P^* u}_{L^2(\M, \rho \mu)}^2 \leq 2 (1+ \alpha L_\rho \veps) \norm{u}^2_{L^2(X, \vec{m}\mu)}  \]
and from Lemma \ref{lem:Lambdaf-f} 
\begin{align*} 
\norm{Iu}_{L^2(\M, \rho \mu )} & =  \norm{\Lambda_{h- 2\veps} P^* u  }_{L^2(\M, \rho \mu )} \\
& \leq C (1+ \alpha L_p h )^{1/2}\cdot (1+ \alpha  L_\rho h)^{1/2} \lVert P^* u \rVert_{L^2(\M, \rho \mu)}  
\\& \leq C (1+ \alpha L_p h )\cdot (1+ \alpha  L_\rho h) \lVert u \rVert_{L^2(X, \vec{m} \mu_n)}
\end{align*}
The assertion (i) follows by  inserting all these estimates back in \eqref{aux:Lem62}.

Regarding assertion (ii), we conclude from Lemma \ref{lem:normdLambdaf} that
\begin{align*}
& D(Iu) \leq    (1+\alpha L_p h)\cdot (1+ C(1 + \frac{1}{\sigma_\eta})mK h^2 ) \frac{1}{\sigma_\eta (h-2\veps)^{m+2}} E_{h-2\veps}(P^*u)  \nc\\
& \leq   (1+\alpha L_p h)\cdot \left(1+ C(1 + \frac{1}{\sigma_\eta})mK h^2 \right)\left(1+ Cm\frac{\veps}{h}\right) \frac{1}{\sigma_\eta h^{m+2}} E_{h-2\veps}(P^*u)  \nc\\
&\leq    \left(1+\alpha L_p h+ C(1 + \frac{1}{\sigma_\eta})mK h^2 + Cm \frac{\veps}{h}\right) \frac{1}{\sigma_\eta h^{m+2}} E_{h-2\veps}(P^*u).  \nc
\end{align*}
Combining with \eqref{eqn:buVSEr} we obtain the desired estimate.
\end{proof}

We can now establish a lower bound for $\lambda_k(\Gamma)$ in terms of $\lambda_k(\M)$. 
\begin{proof}[of lower bound of Theorem \ref{thm:conveigenvalues}] 
	Let $k\in\N$. It follows from \eqref{eqn:Courant} that for very $k$-dimensional subspace $L\subset H^1(\M)$ we have
	\[ \lambda_k(\M) \leq \sup_{f\in L\setminus\{0\}} \frac{D(f)}{\lVert f\rVert^2_{L^2(\M,\rho\mu)}}.\]
	As in the proof of \cite[Prop 6.3]{BIK}
	we denote by $W\subseteq L^2(X)$ the span of orthonormal eigenvectors of $\Delta_{\Gamma}$ corresponding 
	to $\lambda_1(\Gamma),\dots,\lambda_k(\Gamma)$
	and we set $L\coloneqq I(W)$. Then $b(u)\leq \lambda_k(\Gamma)\norm{u}^2_{L^2(X,\vec{m}\mu_n)}$ for all $u\in W$.  Using this, we conclude from Lemma \ref{lem:interpolation}
\begin{align} \label{eqn:Iulowerbound}
\begin{split}
	\norm{Iu}^2_{L^2(\M, \rho \mu)}\geq & \, \big(1 - 2\alpha(1+\alpha L_\rho)(\lVert  \vec{m} - \rho \rVert_\infty + L_\rho \veps  )   - \tilde C'' \sqrt{\lambda_k(\Gamma)} h   \big) \\
	& \: \norm{u}_{L^2(X, \vec{m}\mu_n)}^2  
\end{split}
\end{align}
for all $u \in W$. It follows that if
\[  2\alpha(1+\alpha L_\rho)(\lVert  \vec{m} - \rho \rVert_\infty + L_\rho \veps  )   + \tilde C'' \sqrt{\lambda_k(\Gamma)} h \leq \frac{1}{2}\]
	then the operator $I$ is injective on $W$ and thus $\dim L = k$; notice that this inequality is satisfied under condition \eqref{Prop2Condition} thanks to the upper bound for $\lambda_k(\Gamma)$ in terms of $\lambda_k(\M)$.  It follows from part (ii) of Lemma \ref{lem:interpolation} that for any $f= Iu $ with $u \in W$,
	\begin{align*}
	\begin{split}
	\label{eqn:DIuupperbound}
	 \frac{D(f)}{\lVert f \rVert_{L^2(\M, \rho \mu)}^2} & \leq \frac{1+ C_1'' h + C_2''\frac{\veps}{h} + C_3''h^2 }{1 - 2\alpha(1+\alpha L_\rho)(\lVert  \vec{m} - \rho \rVert_\infty + L_\rho \veps  )   - \tilde C'' \sqrt{\lambda_k(\Gamma)} h }\lambda_k (\Gamma)
	 \end{split}
	\end{align*}
The result now follows from the fact that the above inequality holds for arbitrary $u \in W$ and  the fact that $\lambda_k(\Gamma)$ can be bounded from above by a constant multiple of $\lambda_k(\M)$.
\end{proof}

\section{Approximation of eigenfunctions}
\label{sec:Eigenfunctions}
In this section we are concerned with the convergence of eigenvectors of $\Delta_\Gamma$.
We start by showing that the discretization and interpolation operators $P$ and $I$ are almost inverse of one another. 

\begin{lemma}[cf.\ {\cite[Lemma 6.4]{BIK}}]
\label{lem:almostinverse}
Under Assumptions \ref{stdassumptions}, there exists a constant $C'''$ only depending on $m,\alpha, \eta, L_p , L_\rho$ such that
	\begin{enumerate}[(i)]
	\item $\norm{IPf-f}_{L^2(\M,\rho\mu)} \leq  C''' \nc h D(f)^{\frac{1}{2}}$ for all $f \in H^1(\M)$.
\end{enumerate}
Moreover, if
\[\alpha \lVert \vec{m}- \rho \rVert_\infty + \veps L_\rho \leq \frac{1}{2}, \text{ then, } \]
	\begin{enumerate}[(ii)]
		\item $\norm{PIu-u}_{L^2(X,\vec{m}\mu_n)} \leq C'''h  b(u)^\frac{1}{2}$ for all $u\in L^2(X)$. 
	\end{enumerate}
\end{lemma}
\begin{proof}

	By definition of $I$ we have
	\[ \norm{ IPf - f} \leq \norm{\Lambda_{h-2\veps} (P^*P f -f)} + \norm{\Lambda_{h-2\veps} f-f}.\]
	From Lemmas \ref{lem:Lambdaf-f} and \ref{lem:P*Palmostinverse}, and from Assumptions \ref{stdassumptions}, we know that for a constant $ C''>0$, depending on  $\eta$, $m$, $L_p$, $L_\rho$ and $\alpha$,
	\begin{align*}
	\norm{\Lambda_{h-2\veps}(P^*Pf-f)}_{L^2(\M,\rho\mu)} \leq  C''\norm{P^*Pf-f}_{L^2(\M,\mu)}  
	\leq  C''\veps D(f)^{\frac{1}{2}}.
	\end{align*}
	Likewise, from Lemma \ref{lem:Lambdaf-f} and Lemma \ref{lem:Ervsdf}, 
	\[ \norm{\Lambda_{h-2\veps} f -f}_{L^2(\M,\rho\mu)}^2 \leq \frac{C''}{(h-2\veps)^m} E_{h-2\veps}(f) \leq C''h^2 D(f),\]
	 and from this we deduce assertion (i).

	Regarding assertion (ii),  if we assume that $\alpha \lVert \vec{m}- \rho \rVert_\infty + \veps L_\rho \leq \frac{1}{2}$, we obtain from Lemma \ref{lem:P*almostisometric} that
	\begin{align*}
	 \norm{PIu-u}_{L^2(\M,\rho \mu)} &\leq 4 \norm{P^*(PIu-u)}_{L^2(\M,\rho \mu)} 
	 \\&\leq 4 \norm{P^*PIu-Iu}_{L^2(\M,\rho \mu)} + 4\norm{Iu-P^* u}_{L^2(\M,\rho \mu)}.
	 \end{align*}
	From Lemmas \ref{lem:P*Palmostinverse} and \ref{lem:interpolation}, and from Assumptions \ref{stdassumptions}, we obtain that
	\[ \norm{P^*PIu-Iu}_{L^2(\M,\rho \mu)} \leq C'\veps D(Iu)^\frac{1}{2} \leq C'\veps b(u)^\frac{1}{2} \]
	for a constant $C'$ depending on  $\eta$, $m$, $L_p$, $L_\rho$ and $\alpha$. Moreover, by \eqref{eqn:Iu-P*u} we know there exists $C'''>0$ (depending on  $\eta$, $m$, $L_p$, $L_\rho$ and $\alpha$)  such that
	\[ \norm{Iu-P^*u}_{L^2(\M, \rho \mu)} \leq C''' h b(u)^\frac{1}{2}. \]\nc
	This implies assertion (ii).
\end{proof}

Now we adopt some additional notation from \cite[Section 7]{BIK}. For a value $\lambda\in \R$ we denote by $H_\lambda(\M)$ the linear span in $H^1(\M)$ 
of all eigenfunctions of $\Delta$ corresponding to eigenvalues in the interval $(-\infty,\lambda)$. 
Similarly we define $H_\lambda(X)$ as the linear span of eigenvectors of $\Delta_\Gamma$ corresponding to eigenvalues in $(-\infty,\lambda)$.
We write $\P_\lambda$ for both, the orthogonal projection onto $H_\lambda(\M)$ and $H_\lambda(X)$.

\begin{lemma}[cf.\ {\cite[Lemma 7.1]{BIK}}]
\label{lem:formcomparison}
Suppose that $h$ satisfies Assumptions \ref{stdassumptions} and that 
\[\alpha \lVert \vec{m}- \rho \rVert_\infty + \veps L_\rho \leq \frac{1}{2}.\]
Then, for every $\lambda >0$ we have
	\begin{enumerate}[(i)]
	\item $b(Pf)^\frac{1}{2} \geq  \left(1-  (\sqrt{\lambda}C''' +  C_1'') h - C_2'' \frac{\veps}{h} - C_3'' h^2 \right)\nc D(f)^\frac{1}{2}$.
	\item $D(Iu)^{\frac{1}{2}} \geq  \left(1-  (\sqrt{\lambda}C''' +  C_1') h - C_2' \frac{\veps}{h} - C_3' h^2 \right)\nc b(u)^\frac{1}{2}$
	\end{enumerate}
	for all $f\in H_\lambda(\M)$ and $u\in H_\lambda(X)$. The constants $C_1'', C_2'', C_3''$ are as in Lemma \ref{lem:interpolation}, $C_1', C_2', C_3'$ are as in Lemma \ref{lem:discretization} and  the constant $C''' $  is as in Lemma \ref{lem:almostinverse}.
\end{lemma}
\begin{proof}
	Fix some $\lambda>0$. First note that the projection $\P_\lambda$ does not increase the Dirichlet energy  (neither the graph one nor the continuum one) and hence we conclude  that
	\begin{align*}
	D(IPf)^\frac{1}{2} & \geq D( \P_\lambda IPf)^{1/2}\geq D(f)^{1/2} - D(\P_\lambda IPf-f)^{1/2}.
	\end{align*}
	From Lemma \ref{lem:almostinverse} (i) it follows that, \nc
	\[ D(\P_\lambda IPf-f)^\frac{1}{2} = D(\P_\lambda (IPf - f))^\frac{1}{2} \leq \sqrt{\lambda} \norm{IPf-f}_{L^2(\M,\rho\mu)} 
	\leq C'''\sqrt{\lambda} hD(f)^\frac{1}{2}\]
	for all $f\in H_\lambda(\M)$. Hence,
	\[ D(IPf)^\frac{1}{2} \geq (1-C'''h\sqrt{\lambda})D(f)^\frac{1}{2}.\]
	 Moreover, we know from Lemma \ref{lem:interpolation} (ii) that
	\[ D(IPf)^\frac{1}{2} \leq \left(1+C_1'' h + C_2'' \frac{\veps}{h} + C_3'' h^2 \right) b(Pf)^\frac{1}{2}\]
	 and thus 
	\begin{align*}
	 b(Pf)^\frac{1}{2}& \geq \frac{1-C'''h\sqrt{\lambda}}{1+C_1'' h + C_2'' \frac{\veps}{h} + C_3'' h^2 }D(f)^\frac{1}{2}
	 \\& \geq \left(1-  (\sqrt{\lambda}C''' +  C_1'') h - C_2'' \frac{\veps}{h} - C_3'' h^2 \right) D(f)^\frac{1}{2}
	\end{align*}
	for all $f\in H_\lambda(\M)$  as claimed in (i).  Regarding assertion (ii) we proceed similarly. First, we obtain that
\begin{align*}
 b(PIu)^\frac{1}{2}  \geq  b(u)^{1/2}  - b( \P_\lambda PIu - u  )^{1/2}
 \end{align*}
for $u\in H_\lambda(X)$. Since
\[ b(\mathds{P}_\lambda(PI u-u))^\frac{1}{2}  \leq \sqrt{\lambda} \norm{ PIu- u}_{L^2(X,\vec{m}\mu_n)}
\leq \sqrt{\lambda} C'''h b(u)^\frac{1}{2}\]
by part (ii) of Lemma \ref{lem:almostinverse}, we have 
\[ b(PIu)^\frac{1}{2}  \geq (1-C'''\sqrt{\lambda}h)b(u)^\frac{1}{2}.\]
Moreover, we know from part (ii) of Lemma \ref{lem:discretization} that
\[ b(PIu)^\frac{1}{2} \leq \left(1+C_1'h + C_2'\frac{\veps}{h} + C_3'h^2 \right) D(Iu)^\frac{1}{2}.\]
 Therefore,
 \begin{align*}
 \begin{split}
  D(Iu)^\frac{1}{2} & \geq \frac{1-C'''\sqrt{\lambda}h}{1+C_1'h + C_2'\frac{\veps}{h} + C_3'h^2} b(u)^\frac{1}{2}
 \\& \geq \left(1-(C'''\sqrt{\lambda} +C_1')h -C_2'\frac{\veps}{h} - C_3'h^2\right)b(u)^\frac{1}{2},
 \end{split}
 \end{align*}
which proves assertion (ii). 
\end{proof}

\begin{proof}[Theorem \ref{thm:Iuapproxf}]
	This theorem can now be proven word-for-word as \cite[Theorem 4]{BIK} together with the
	required Lemmas \cite[Lemma 7.2, 7.3, 7.4]{BIK} by replacing every application of Lemma 4.3, 6.2, 7.1 and Theorem 1 therein	with the previously proven Lemmas \ref{lem:discretization}, \ref{lem:interpolation}, \ref{lem:formcomparison} and Theorem \ref{thm:conveigenvalues}, respectively.
\end{proof}

We now focus on establishing Theorem \ref{thm:voronoiapprox}. To simplify our computations we set
\[ \theta :=  \left( \frac{\veps}{h} + (1 + \sqrt{\lambda_k(\M)}) h + \left(K + \frac{1}{R^2} \right)h^2  + \lVert \vec{m}- \rho \rVert_\infty\right).  \]
In the setting of Theorem \ref{thm:Iuapproxf} we have
\begin{align*}
\red \norm{P^* u - f} &\leq \norm{P^* u - Iu}+ \norm{Iu - f} \leq C'h b(u)^{1/2} + \frac{\tilde{C}}{g_{k,\rho\mu}} \theta \nc \\
&=  C'h\sqrt{\lambda_k(\Gamma)} + \red \frac{\tilde{C}}{g_{k,\rho\mu}} \theta \nc
\end{align*}
where the second inequality follows from \eqref{eqn:Iu-P*u}.  From Theorem \ref{thm:conveigenvalues}, for $h$ small enough we have 
\begin{align}
\label{eqn:P*approx}
\red \norm{P^* u - f}  \leq C' h \sqrt{\lambda_k (\M)}   + \frac{\tilde{C}}{g_{k,\rho\mu}} \theta,  
 \end{align}
Therefore, every extension of $u$ that approximates $P^* u$ in $L^2(\M,\rho\mu)$ (or equivalently in $L^2(\M,\mu)$) 
is also an approximation of the eigenfunction $f$. 

We recall the definition of sets $U_i \subset \M$ in \eqref{def:Ui},  Euclidean Voronoi cells $V_i$ in \eqref{eqn:voronoidef},  and of the extended vector $\bar u$ from  \eqref{eqn:barudef}.
Concerning the measure of such a Voronoi cell, we obtain the following bound.
\begin{lemma}
\label{lem:UivsVi}
 For every $\beta>1$ there exists a constant $C>0$ depending on $m$ and on $\ell$ from \eqref{winfest} such that
	\[ \frac{\mu(V_i)}{\mu(U_i)} \leq C\cdot \log^{mp_m} n \eqqcolon C(n) \]
	for all $i=1,\dots,n$ and all $n\in\N$ with probability at least $1-C_{K,\vol(\M),m,i_0} n^{-\beta}$.
\end{lemma}
\begin{proof}
	We first show that $V_i \subseteq \{ x\in \M  : \abs{x-x_i} \leq \veps\}$. To this end,
	suppose $x\in \M$ such that $\abs{ x-x_i } > \veps$. 
Then also $d(x,x_i) > \veps$.  Since the balls $B_\M(x_j,\veps)$ cover $\M$ by the choice of $\veps$,
there exists $x_j$ such that $d(x,x_j) < \veps$. 
Therefore,  $\abs{ x-x_j } < \veps < \abs{ x-x_i}$ and thus $x\not \in V_i$. This proves the claim.

Now we assume that the assertion of Theorem \ref{thm:transport} holds.
For $\veps\leq \frac{\rc}{2}$ it follows from Proposition \ref{prop:metricestimates} that
$V_i$ is contained in the ball $B_\M(x_i,3\veps)$.
Thus, we obtain from the bounds on the distortion of metric by the exponential map \eqref{eqn:EstimateVolumeBall2} that
\[ \frac{\mu(V_i)}{\mu(U_i)} \leq  \frac{\mu( B_\M(x_i,3\veps))}{\mu(U_i)} 
 \leq  \frac{\alpha \omega_m (3\veps)^m C}{1/n} = C \alpha \omega_m 3^m \ell^m \log^{p_m \cdot m}(n) \]
where $\ell$ defined in \eqref{winfest}, and $C>0$ is a universal constant. 
\end{proof}


\begin{proof}[Theorem \ref{thm:voronoiapprox}] 
%
Let $u\in L^2(X)$ be a normalized eigenvector of $\Delta_\Gamma$ corresponding to $\lambda_k(\Gamma)$ and let $f$ a normalized eigenfunction of $\Delta$ corresponding to $\lambda_k(\M)$ as in Theorem \ref{thm:Iuapproxf} (or as in \eqref{eqn:P*approx}). Let
\[ V\coloneqq \norm{ \bar u - f }^2_{L^2(\M,\mu)} = \sum_{i=1}^n \int_{V_i} \abs{ u(x_i) - f(y) }^2 d\mu(y)\]
and
\[ U\coloneqq \sum_{i=1}^n \frac{\mu(V_i)}{\mu(U_i)} \int_{U_i} \abs{ u(x_i)-f(x) }^2 d\mu(x).\]
Then, by Lemma \ref{lem:UivsVi} and 
\eqref{eqn:P*approx},
\[ \sqrt{U} \leq \sqrt{C(n)} \cdot \lVert P^* u - f\rVert_{L^2(\M,\mu)} \leq \sqrt{C(n)} \left( C' h \sqrt{\lambda_k (\M)}   + \frac{\tilde{C}}{g_{k,\rho\mu}} \theta \right). \] 
On the other hand,
\begin{align*}
&\abs{ V-U } 
\leq \sum_{i=1}^n \mu(V_i) \abs*{ \int_{V_i} \abs{u(x_i)-f(y)}^2 \frac{d\mu(y)}{\mu(V_i)}-\int_{U_i}\abs{ u(x_i) -f(x)}^2 \frac{d\mu(x)}{\mu(U_i)}}
\\& = \sum_{i=1}^n \mu(V_i) \abs*{ \int_{V_i} \int_{U_i} \bigl(\abs{ u(x_i) - f(y) }^2 - \abs{ u(x_i) - f(x) }^2 \bigr)\frac{d\mu(y)}{\mu(V_i)} \frac{d\mu(x)}{\mu(U_i)}}  
\\& =  \sum_{i=1}^n \mu(V_i) \Bigg| \int_{V_i} \int_{U_i} \bigl(  2(u(x_i) - f(x))(f(x) - f(y)) + (f(x)-f(y))^2 \bigr)\\
& \hspace*{60pt} \frac{d\mu(y)}{\mu(V_i)} \frac{d\mu(x)}{\mu(U_i)} \Bigg|\\
&\leq  8 \lVert \nabla f\rVert_\infty \veps  \left( \sum_{i=1}^n \mu(V_i)  \int_{V_i} \int_{U_i}  \abs{  u(x_i)  -f(x)    }   \frac{d\mu(y)}{\mu(V_i)} \frac{d\mu(x)}{\mu(U_i)}  \right) +   16 \lVert \nabla f\rVert_\infty^2  \veps^2
\\& =  8 \lVert \nabla f\rVert_\infty \veps  \left( \sum_{i=1}^n \mu(V_i) \int_{U_i}  \abs{  u(x_i)  -f(x)     }   \frac{d\mu(x)}{\mu(U_i)}  \right) +   16 \lVert \nabla f\rVert_\infty^2  \veps^2
\\&\leq 8 \lVert \nabla f\rVert_\infty \veps  \sqrt{U} +   16 \lVert \nabla f\rVert_\infty^2  \veps^2
\end{align*}
where in the second equality we have used the fact that for all $y \in V_i$ and all $x \in U_i$,
$d(x,y) \leq d(x, x_i)+ d(x_i,y) \leq 3\veps + \veps$; the last inequality follows from Jensen's inequality.
Thus,
\begin{align*}
V \leq \abs{ V-U} + U & \leq 16(  \veps \lVert \nabla f \rVert_\infty  + \sqrt{U} )^2,
\end{align*}
and from this it follows that
\[\norm{ \bar u - f }_{L^2(\M,\mu)} = \sqrt{V} \leq 4\veps \lVert \nabla f \rVert_\infty + 4 \sqrt{U}.  \]
Using \cite{ShiXu10} we know that 
\[ \lVert \nabla f\rVert_\infty \leq  C_\M\lambda_k(\M)^{\frac{m+1}{4}} \lVert f\rVert_{L^2(\M,\mu)} = C_\M \lambda_k(\M)^{\frac{m+1}{4}},  \]
for a constant $ C_\M>0$ that depends on the manifold $\M$. Putting everything together we deduce that 
\[ \norm{ \bar u - f }_{L^2(\M,\mu)} \leq C_\M \lambda_k(\M)^{\frac{m+1}{4}} \veps+ \tilde{C} \sqrt{C(n)}\left( \sqrt{\lambda_k(\M)} h + \frac{\theta}{g_{k,\rho\mu}}  \right) ,\]
which is the desired estimate.
\end{proof}

\begin{acknowledgements}
The authors thank Yaroslav Kurylev for stimulating discussions. DS is grateful for support of the  National Science Foundation under the  grant DMS 1516677. MH is grateful for the support by the ERC Grant NOLEPRO. The authors are also grateful to the Center for Nonlinear Analysis (CNA) for support.

\end{acknowledgements}

\appendix

\section{Kernel density estimates via transportation} \label{sec:kdeA}

Here we use the estimates on infinity transportation distance established in Section \ref{sec:transport} to show the kernel density estimates we need. 
While the estimates we prove are not optimal, they do not affect the rate of convergence of eigenvalues and eigenfunctions in our main theorems. We chose to present the proof below as it highlights how the optimal transportation estimates can be used to provide general kernel density estimates in a simple and direct way. 

\begin{lemma} \label{lem:kde}
\red
Consider $\eta: \R \to \R$, nonincreasing, supported on $[0,1]$, and normalized: $ \int_{\R^m} \eta(\abs{x}) dx =1$. \nc Consider $h>0$ satisfying Assumption \ref{stdassumptions}.  Then  \eqref{LinfWeights} holds. That is 
there exists a universal constant $C>0$ such that 
\begin{equation}
 \max_{i=1, \dots, n}  \lvert m_i - p(\x_i) \rvert  \leq  CL_p h  +  C \alpha \eta(0) m \omega_m \frac{\veps}{h} + C \alpha m \left(K + \frac{1}{R^2} \right)h^2 ,  
 \end{equation}
where  $\veps$ is the $\infty$-OT distance between $\mu_n$ and $\mu$ (see Section \ref{sec:transport}). 
\end{lemma}

The weights $\vec{m}$ are defined by
\[  m_i = \frac{1}{n h^m } \sum_{j=1}^n \eta \left( \frac{\lvert  x_i - x_j \rvert}{h} \right), \quad i =1 , \dots ,n ,   \]
 $p$ is   the density of $\mu$ with respect to $\M$'s volume form.
 \red We remark that  we do not require $\eta$ to be Lipschitz on $[0,1]$.
\nc

\begin{proof}
First, notice that for every $i,j$ with $\lvert x_i - x_j \rvert \leq h $ we have $\lvert x_i - x_j \rvert \leq \frac{R}{2} $ and hence Proposition \ref{prop:metricestimates} implies that
\[ d(x_i, x_j)  \leq \lvert x_i - x_j \rvert + \frac{8}{R^2} \lvert x_i - x_j \rvert^3  \leq \left(1+ \frac{8 h^2}{R^2} \right) \lvert x_i- x_j \rvert. \]
Therefore, for every $i,j$ and every $y \in U_j$,
\[ \eta \left( \frac{\lvert x_i - x_j  \rvert}{ h} \right) \leq \eta \left( \frac{d (x_i , x_j) }{ \hat h} \right)  \leq  \eta \left( \frac{(d( x_i, y) - \veps )_+ }{ \hat{h}} \right), \]
where we recall that $\veps$ is the $\infty$-OT distance between $\mu_n$ and $\mu$ and where $\hat{h}:= h+ \frac{27 h^3}{R^2}$. From this it follows that 
\begin{align}
\begin{split}
\label{AuxAppendix0}
 m_i = \frac{1}{n h^m} \sum_{j=1}^{n} \eta \left( \frac{\lvert \x_i - \x_j \rvert}{h} \right) & \leq \frac{1}{ h^m} \int_{\M} \eta \left( \frac{ (d( x_i, y) -\veps )_+}{ \hat h} \right) p(y) d Vol(y)
 \\& \leq  (p(x_i) + 10L_p h) \frac{1}{ h^m} \int_{\M} \eta \left( \frac{(d (x_i , y ) -\veps )_+}{\hat h} \right)  d Vol(y), 
\end{split}
\end{align}
where the last inequality follows using the Lipschitz continuity of $p$, the fact that $\veps< h $ and the fact that $h < \frac{R}{2}$ (so that in particular $\hat{h} + \veps < 10 h $). Now,
\begin{align}
\begin{split}
\label{AuxAppendix1}
\frac{1}{ h^m} \int_{\M} \eta \left( \frac{ ( d(x_i , y) -\veps )_+}{\hat h} \right)  d Vol(y) & =  \frac{1}{ h^m} \int_{B( \hat{h} + \veps)} \eta \left( \frac{(\lvert z \rvert -\veps )_+}{\hat{h}} \right)  J_{x_i}(z) d z 
\\&  \leq  \frac{1}{ h^m} \int_{B( \hat{h} + \veps)} \eta \left( \frac{(\lvert z \rvert -\veps )_+}{\hat h} \right)  J_{x_i}(z) d z 
\\& \leq  (1+ C mK h^2)  \frac{1}{ h^m} \int_{B( \hat{h} + \veps)} \eta \left( \frac{(\lvert z \rvert -\veps )_+}{ \hat h} \right)  d z, 
\end{split}
\end{align}
where $C$ is a universal constant. The last integral above can be estimated as follows
\begin{align}
\begin{split}
\label{AuxAppendix2}
\frac{1}{h^m}\int_{\R^m} \eta \left( \frac{(\lvert z \rvert - \veps)_+}{ \hat h} \right) dz & =    \eta(0)  \omega_m  
\frac{\veps^m}{h^m}  +  \frac{1}{h^m}\int_{B( \hat{h}+\veps ) \setminus B( \veps)} \eta \left( \frac{\lvert z \rvert-\veps}{ \hat h} \right) dz
\\& = \eta(0)  \omega_m  
\frac{\veps^m}{h^m}  +  \frac{\hat{h}^m}{h^m}\int_{0}^{1} m \omega_m \left(r+ \frac{\veps}{\hat h} \right)^{m-1}  \eta \left( r \right) dr
\\& \leq \eta(0)  \omega_m  
\frac{\veps^m}{h^m}  +  \left(1 + \frac{16 mh^2 }{R^2}\right)  \int_{0}^{1} m \omega_m \left(r+ \frac{\veps}{ h} \right)^{m-1}  \eta \left( r \right) dr
\end{split}
\end{align}
Using the binomial theorem we obtain
\begin{align*}
\begin{split}
m \omega_m\! \int_{0}^1 \! \left(r+ \frac{\veps}{h}\right)^{m-1} \!\eta(r)dr & \leq m \omega_m \int_{0}^1 \! r^{m-1} \eta(r) dr  +  m \omega_m \eta(0) \sum_{k=1}^{m-1}  \binom{m-1}{k} \left( \frac{\veps}{h} \right)^k \!\frac{1}{m-k}  
\\ &= 1 +  \omega_m \eta(0) \sum_{k=1}^{m-1}  \binom{m}{k} \left( \frac{\veps}{h} \right)^k 
\\& = 1+ \omega_m \eta(0) \left( \left( 1+ \frac{\veps}{h}  \right)^{m}  - 1 - \frac{\veps^m}{h^m} \right)
\\&  \leq  1+  2 m\eta(0)\omega_m\frac{\veps}{h}  -   \eta(0)  \omega_m\frac{\veps^m}{h^m}
\end{split}
\end{align*}
where in the first equality we have used the fact that $\eta$ was assumed to be normalized and in the last inequality we have used  
\begin{align*} \label{2simest1}
 (1+s)^m \leq 1 + 2ms \quad   \te{ whenever } 0 \leq s < \frac{1}{m}.
\end{align*}
Combining \eqref{AuxAppendix0}, \eqref{AuxAppendix1} and \eqref{AuxAppendix2} we conclude that 
\[ m_i - p(x_i) \leq  p(x_i) +   CL_p h + C \alpha \eta(0) m \omega_m \frac{\veps}{h}    +C\alpha m \left( K + \frac{1}{R^2} \right) h^2,\]
for a universal constant $C>0$.

In a similar fashion we can find an upper bound for $p(x_i) -m_i$. Indeed, observe that for every $i,j$ and $ y\in U_i$ we have
\[ \eta\left( \frac{\lvert x_i - x_j \rvert}{h} \right) \geq  \eta\left( \frac{d(x_i, x_j) }{h} \right) \geq  \eta\left( \frac{d(x_i, y) + \veps}{h} \right) \]
and so
\begin{align}
\begin{split}
m_i & \geq  \frac{1}{h^m} \int_{\M} \eta \left( \frac{d(x_i, y) + \veps}{h}   \right) p(y) d Vol(y) 
\\ & \geq   \frac{1}{h^m} \int_{\M} \eta \left( \frac{d(x_i, y) + \veps}{h}   \right)(p(x_i)-  L_p d(x_i, y)) d Vol(y)
\\&\geq   (p(x_i)- L_p  h )  \frac{1}{h^m} \int_{\M} \eta \left( \frac{d(x_i, y) + \veps}{h}   \right) d Vol(y).
\end{split}
\end{align}
The above integral can be estimated from below by
\begin{align}
\begin{split}
\frac{1}{h^m}\int_{\M} \eta \left( \frac{d(x_i, y) + \veps}{h}   \right) d Vol(y) &  = \frac{1}{h^m} \int_{B( h- \veps)} \eta \left( \frac{\lvert  z \rvert + \veps}{h}   \right) J_{x_i}(z) d z
\\& \geq (1- Cm K h^2) \frac{1}{h^m} \int_{B( h- \veps)} \eta \left( \frac{\lvert  z \rvert + \veps}{h}   \right)  d z
\\& = (1- Cm K h^2) \int_{\veps/h}^1 m \omega_m \eta(r)(r - \frac{\veps}{h})^{m-1}dr
\end{split}
\end{align}
where the second equality follows using polar coordinates and a change of variables; the last inequality follows from the fact that $\eta$ is assumed to be normalized. In turn, 
\begin{align*}
\begin{split}
\int_{\veps/h}^1 m \omega_m \eta(r)\left(r - \frac{\veps}{h}\right)^{m-1}dr  &  \geq   \int_{\veps/h}^1 m \omega_m \eta(r)r^{m-1} dr-  m \omega_m  \frac{\veps}{h}\int_{\veps/h}^1 (m-1) \eta(r)r^{m-2} dr
\\& \geq 1- 2\eta(0)m \omega_m \frac{\veps}{h},
\end{split}
\end{align*}
where we have used the fact that $\eta$ was assumed to be normalized. Combining the above inequalities we deduce that 
\[   p(x_i) - m_i \leq   L_p h + C \alpha m \omega_m K h^2  + C \alpha  m \omega_m  \eta(0) \frac{\veps}{h}.  \]
\end{proof}



\bibliographystyle{siam}
\bibliography{bib_gghs}

\begin{thebibliography}{10}

\bibitem{AreEls12}
{\sc W.~Arendt and A.~F.~M. ter Elst}, {\em Sectorial forms and degenerate
  differential operators}, J. Operator Theory, 67 (2012), pp.~33--72.

\bibitem{Belkin02laplacianeigenmaps}
{\sc M.~Belkin and P.~Niyogi}, {\em Laplacian eigenmaps for dimensionality
  reduction and data representation}, Neural Computation, 15 (2002),
  pp.~1373--1396.

\bibitem{belkin2007convergence}
{\sc M.~Belkin and P.~Niyogi}, {\em Convergence of {L}aplacian eigenmaps},
  Advances in Neural Information Processing Systems (NIPS), 19 (2007), p.~129.

\bibitem{bel_niy_LB}
\leavevmode\vrule height 2pt depth -1.6pt width 23pt, {\em Towards a
  theoretical foundation for {L}aplacian-based manifold methods}, J. Comput.
  System Sci., 74 (2008), pp.~1289--1308.

\bibitem{Bess78}
{\sc A.~L. Besse}, {\em Manifolds all of whose geodesics are closed}, vol.~93
  of Ergebnisse der Mathematik und ihrer Grenzgebiete [Results in Mathematics
  and Related Areas], Springer-Verlag, Berlin-New York, 1978.
\newblock With appendices by D. B. A. Epstein, J.-P. Bourguignon, L.
  B{\'e}rard-Bergery, M. Berger and J. L. Kazdan.

\bibitem{BIK}
{\sc D.~Burago, S.~Ivanov, and Y.~Kurylev}, {\em A graph discretization of the
  {L}aplace-{B}eltrami operator}, J. Spectr. Theory, 4 (2014), pp.~675--714.

\bibitem{Cha1984}
{\sc I.~Chavel}, {\em Eigenvalues in Riemannian geometry}, Academic Press, New
  York, 1984.

\bibitem{Coifman1}
{\sc R.~R. Coifman and S.~Lafon}, {\em Diffusion maps}, Appl. Comput. Harmon.
  Anal., 21 (2006), pp.~5--30.

\bibitem{doCa92}
{\sc M.~P. do~Carmo}, {\em Riemannian geometry}, Mathematics: Theory \&
  Applications, Birkh\"auser Boston, Inc., Boston, MA, 1992.
\newblock Translated from the second Portuguese edition by Francis Flaherty.

\bibitem{Fuj1995}
{\sc K.~Fujiwara}, {\em Eigenvalues of laplacians on a closed riemannian
  manifold and its nets}, Proc. Amer. Math. Soc., 123 (1995), pp.~2585--2594.

\bibitem{GTS15a}
{\sc N.~Garc{\'{\i}}a~Trillos and D.~Slep{\v{c}}ev}, {\em On the rate of
  convergence of empirical measures in {$\infty$}-transportation distance},
  Canad. J. Math., 67 (2015), pp.~1358--1383.

\bibitem{GTSspectral}
{\sc N.~Garc\'i�a~Trillos and D.~Slep\v{c}�ev}, {\em A variational approach
  to the consistency of spectral clustering}, Appl. Comput. Harmon. Anal.,
  (2016), pp.~--.

\bibitem{GK}
{\sc E.~Gin{\'e} and V.~Koltchinskii}, {\em Empirical graph {L}aplacian
  approximation of {L}aplace-{B}eltrami operators: large sample results}, in
  High dimensional probability, vol.~51 of IMS Lecture Notes Monogr. Ser.,
  Inst. Math. Statist., Beachwood, OH, 2006, pp.~238--259.

\bibitem{Hei2006}
{\sc M.~Hein}, {\em Uniform convergence of adaptive graph-based
  regularization}, in Proc. of the 19th Annual Conference on Learning Theory
  (COLT), G.~Lugosi and H.~U. Simon, eds., Springer, 2006, pp.~50--64.

\bibitem{HeAuvL07}
{\sc M.~Hein, J.-Y. Audibert, and U.~v. Luxburg}, {\em Graph laplacians and
  their convergence on random neighborhood graphs}, Journal of Machine Learning
  Research, 8 (2007), pp.~1325--1368.

\bibitem{LeightonShor}
{\sc T.~Leighton and P.~Shor}, {\em Tight bounds for minimax grid matching with
  applications to the average case analysis of algorithms}, Combinatorica, 9
  (1989), pp.~161--187.

\bibitem{Moh1991}
{\sc B.~Mohar}, {\em Some applications of laplace eigenvalues of graphs}, in
  Graph Theory, Combinatoris and Applications, Y.~Alavi, G.~Chartrand, O.~R.
  Oellermann, and A.~J. Schwenk, eds., Wiley, 1991, pp.~871--898.

\bibitem{MugNit12}
{\sc D.~Mugnolo and R.~Nittka}, {\em Convergence of operator semigroups
  associated with generalised elliptic forms}, J. Evol. Equ., 12 (2012),
  pp.~593--619.

\bibitem{NiSmWe08}
{\sc P.~Niyogi, S.~Smale, and S.~Weinberger}, {\em Finding the homology of
  submanifolds with high confidence from random samples}, Discrete Comput.
  Geom., 39 (2008), pp.~419--441.

\bibitem{PenroseBook}
{\sc M.~Penrose}, {\em Random geometric graphs}, vol.~5 of Oxford Studies in
  Probability, Oxford University Press, Oxford, 2003.

\bibitem{RosBelVit2010}
{\sc L.~Rosasco, M.~Belkin, and E.~D. Vito}, {\em On learning with integral
  operators}, Journal of Machine Learning Research, 11 (2010), pp.~905--934.

\bibitem{ShiXu10}
{\sc Y.~Shi and B.~Xu}, {\em Gradient estimate of an eigenfunction on a compact
  {R}iemannian manifold without boundary}, Ann. Global Anal. Geom., 38 (2010),
  pp.~21--26.

\bibitem{Shi2015}
{\sc Z.~Shi}, {\em Convergence of laplacian spectra from random samples}.
\newblock arXiv preprint arXiv:1507.00151, 2015.

\bibitem{ShorYukich}
{\sc P.~W. Shor and J.~E. Yukich}, {\em Minimax grid matching and empirical
  measures}, Ann. Probab., 19 (1991), pp.~1338--1348.

\bibitem{SinWu13}
{\sc A.~Singer and H.-T. Wu}, {\em Spectral convergence of the connection
  laplacian from random samples}, Information and Inference: A Journal of the
  IMA, 6 (2017), pp.~58--123.

\bibitem{TalagrandGenericChain}
{\sc M.~Talagrand}, {\em The generic chaining}, Springer Monographs in
  Mathematics, Springer-Verlag, Berlin, 2005.
\newblock Upper and lower bounds of stochastic processes.

\bibitem{THJ}
{\sc D.~Ting, L.~Huang, and M.~I. Jordan}, {\em An analysis of the convergence
  of graph {L}aplacians}, in Proc. of the 27th Int. Conference on Machine
  Learning (ICML), 2010.

\bibitem{vonLux_tutorial}
{\sc U.~von Luxburg}, {\em A tutorial on spectral clustering}, Statistics and
  computing, 17 (2007), pp.~395--416.

\bibitem{vLBeBo08}
{\sc U.~von Luxburg, M.~Belkin, and O.~Bousquet}, {\em Consistency of spectral
  clustering}, Ann. Statist., 36 (2008), pp.~555--586.

\end{thebibliography}

\end{document}